\documentclass{article}
\usepackage{arxiv}

\usepackage[utf8]{inputenc} % allow utf-8 input
\usepackage[T1]{fontenc}    % use 8-bit T1 fonts
\usepackage{hyperref}       % hyperlinks
\usepackage{url}            % simple URL typesetting
\usepackage{booktabs}       % professional-quality tables
\usepackage{amsfonts}       % blackboard math symbols
\usepackage{nicefrac}       % compact symbols for 1/2, etc.
\usepackage{microtype}      % microtypography
\usepackage{paralist,multirow,graphicx,amsmath,amsfonts,amssymb,epstopdf,color,epsfig,ragged2e,array,standalone,booktabs,subfig,float}
\usepackage{stfloats}
\graphicspath{ {./images/} }

\usepackage{array}
\usepackage{tabularx}
\usepackage{amsmath}
\usepackage{amsthm}
\usepackage{mathtools}
\usepackage{algorithm}
\usepackage{algorithmic}
\usepackage{enumitem}
\usepackage{centernot}
\DeclarePairedDelimiter\ceil{\lceil}{\rceil}
\DeclarePairedDelimiter\floor{\lfloor}{\rfloor}
\DeclareMathOperator\supp{supp}

\newcommand{\reals}{\ensuremath{\mathbb{R}}}
\newcommand{\naturals}{\ensuremath{\mathbb{N}}}
\newcommand{\ints}{\ensuremath{\mathbb{Z}}}
\newcommand{\defeq}{\vcentcolon=}
\newcommand{\eqdef}{=\vcentcolon}

\newcommand{\prob}{\ensuremath{\mathbb{P}}}

\newcommand{\cN}{\ensuremath{\mathcal{N}}}

\newcommand{\cC}{\ensuremath{\mathcal{C}}}

\newcommand{\cX}{\ensuremath{\chi}}

\newcommand{\cU}{\ensuremath{\mathcal{U}}}
\newcommand{\cO}{\ensuremath{\mathcal{O}}}

\newcommand{\cS}{\ensuremath{\mathcal{S}}}

\newcommand{\cP}{\ensuremath{\mathcal{P}}}
\newcommand{\cE}{\ensuremath{\mathcal{E}}}
\newcommand{\cQ}{\ensuremath{\mathcal{Q}}}

\newcommand{\cT}{\ensuremath{\mathcal{T}}}

\newcommand{\va}{\mathbf a}
\newcommand{\vz}{\mathbf z}
\newcommand{\vx}{\mathbf x}
\newcommand{\vy}{\mathbf y}
\newcommand{\vb}{\mathbf b}
\newcommand{\vh}{\mathbf h}
\newcommand{\vr}{\mathbf r}
\newcommand{\vw}{\mathbf w}
\newcommand{\vt}{\mathbf t}
\newcommand{\vq}{\mathbf q}

\newcommand{\vu}{\mathbf u}
\newcommand{\vv}{\mathbf v}
\newcommand{\mW}{\ensuremath{\mathbf{W}}}

\newcommand{\mX}{\ensuremath{\mathbf{X}}}

\newcommand{\mY}{\ensuremath{\mathbf{Y}}}
\newcommand{\mB}{\ensuremath{\mathbf{B}}}
\newcommand{\mG}{\ensuremath{\mathbf{G}}}

\newcommand{\mA}{\ensuremath{\mathbf{A}}}
\newcommand{\mR}{\ensuremath{\mathbf{R}}}
\newcommand{\mV}{\ensuremath{\mathbf{V}}}
\newcommand{\mP}{\ensuremath{\mathbf{P}}}

\newcommand{\mU}{\ensuremath{\mathbf{U}}}

\newtheorem{theorem}{Theorem}[section]
\newtheorem{corollary}{Corollary}[theorem]
\newtheorem{lemma}[theorem]{Lemma}
\newtheorem{definition}{Definition}[section]

\title{Encoder Blind Combinatorial Compressed Sensing}

\author{
 Michael Murray, Jared Tanner\\
  Mathematical Institute, University of Oxford, UK\\
  \& The Alan Turing Institute, London, UK\\
  \texttt{[murray,tanner]@maths.ox.ac.uk}\\
  }

\begin{document}
\maketitle

\begin{abstract}
In its most elementary form, compressed sensing studies the design of decoding algorithms to recover a sufficiently sparse vector or code $\vx \in \reals^n$, from a linear measurement vector $\vy = \mA \vx \in \reals^m$, normally in the context that $m \ll n$. Typically it is assumed that the decoder has access to the encoder matrix $\mA \in \reals^{m \times n}$, which in the combinatorial case is sparse and binary. In this paper we consider the problem of designing a decoder to recover a set of sparse codes from their linear measurements alone, that is without access to $\mA$. To this end we study the matrix factorisation task of recovering $\mA$ and $\mX$ from $\textbf{Y} = \textbf{A}\textbf{X}$, where $\mY \in \reals^{m \times N}$ is a set of $N$ measurement vectors, $\textbf{A}$ is an $m \times n$ sparse, binary matrix which we are free to design, and $\textbf{X}$ is an $n \times N$ real matrix whose columns are $k$-sparse. The contribution of this paper is a computationally efficient decoding algorithm, Decoder-Expander Based Factorisation (D-EBF), with strong performance guarantees. In particular, under mild assumptions on the sparse coding matrix and by deploying a novel random encoder matrix, we prove that D-EBF recovers both the encoder and sparse coding matrix at the optimal measurement rate with high probability in $n$, from a near optimal number $N =\Omega \left(\frac{n}{k}\log^2(n)\right)$ of measurement vectors. In addition, our experiments demonstrate the efficacy and computational efficiency of D-EBF in practice. Beyond compressed sensing our results may be of interest for researchers working in areas such as linear sketching, coding theory and matrix compression.
\end{abstract}

\section{Introduction}\label{sec:problem}

Compressed sensing (CS) \cite{1580791,4016283,c2005decoding,quant_uncertainty,  stable_sig_recovery,Donoho2004Compressed}, in its simplest setting, studies the design of sensing or encoder matrices $\mA \in \reals^{m \times n}$ and decoder algorithms in order to recover a sparse vector $\vx \in \reals^{n}$ from $m\ll n$ linear measurements, $\vy = \mA \vx \in \reals^m$. A primary motivation for compressed sensing is the recovery and reconstruction of sparse signals using a number of measurements well below that suggested by the Nyquist-Shannon sampling theorem. In particular, instead of sampling a signal at a high rate and then performing compression, which can be wasteful, one can directly sample or sense the data in a compressed form, thereby reducing the number of measurements required for reconstruction. We refer the reader to \cite{elad_book,intro_CS} for an in-depth overview and survey of the topic. Typically when analysing and designing compressed sensing algorithms it is assumed that the decoder has access to both the measurement vector and encoder matrix. Here we instead study the design of an encoder matrix and decoder algorithm to perform compressed sensing under the assumption that the decoder does not have access to the encoder matrix. In this context we refer to the decoder algorithm as being `blind' to the specific encoder matrix used to generate the linear measurements available to it \footnote{In order to avoid confusion we emphasise the distinction here with the problem studied in \cite{blind_cs_yonina}, in which the decoder has knowledge of the encoder matrix, but does not have access to the sparsifying basis of the data in question.}.

Our motivation for studying this problem, and intuition as to how it may be solved, arises from observations on combinatorial compressed sensing \cite{4797639}. In particular, consider an encoder matrix which is sparse, binary, has a fixed number $d \ll m$ ones per column and is the adjacency matrix of $(k,\epsilon,d)$-expander graph: such graphs satisfy the unique neighbour property, defined and discussed in detail in Section \ref{subsec:ccs_expanders}, which bounds the overlap in support between any subset of $k$ columns of $\mA$. With $\vy = \mA \vx$, then this property implies that there exists a threshold on the number of times a sum of two or more entries of $\vx$ can appear in $\vy$. As a result, and if sums over different sets of nonzeros in $\vx$ result in different values, then nonzero entries in $\vx$ can trivially be recovered by identifying the entries in $\vy$ whose frequency of appearance is above this threshold. Furthermore, the location of a nonzero of $\vx$ can then be recovered by finding the column of $\mA$ whose support maximally overlaps the set row locations of $\vy$ in which the nonzero appears. Once a nonzero value of $\vx$ and its location have been recovered, then its contribution can be removed from the measurement vector, thereby allowing for the identification of further nonzeros and their locations. As long as the overlap in support in each subset of $k$ columns of $\mA$ is sufficiently small, then this iterative, peeling approach will fully decode $\vy$, we refer the reader to \cite{ER_paper} for further details.

Consider now the same setting but assume that the decoder is not granted access to $\mA$: can the same principles and techniques be applied? Observe that although it may be possible to recover certain nonzeros of $\vx$ in the same manner as before, without access to $\mA$ it is not possible to recover the locations of these nonzeros, or remove their contribution from $\vy$. Without the ability to peel away the contributions of the recovered nonzeros from $\vy$ then it may not be possible to recover all the nonzeros in $\vx$. To apply the same iterative, peeling technique as before it therefore seems necessary to also recover the encoder matrix. Ignoring for now the recovery of the locations, then identifying a nonzero value of $\vx$ in $\vy$, which does not require access to $\mA$, facilitates the recovery of at least part of the support of a column of $\mA$. If we have access to multiple measurements vectors, and can identify nonzeros in each corresponding to the same column of $\mA$, then by unioning the locations in which they appear it seems possible that a full column of $\mA$ could be recovered. We therefore study the matrix factorisation problem of recovering both the encoder $\mA \in \reals^{m \times n}$ and the $N$ $k$-sparse column codes $\mX \in \reals^{n \times N}$ from $\mY = \mA \mX$. We remark that to recover all of the columns of $\mA$ it is necessary that $N \geq n/k$. This problem is daunting, indeed, before even considering computational feasibility and efficiency, the conditions ensuring uniqueness of such a factorisation are far from clear. In particular, for any permutation matrix $\mP \in \cP^{n \times n}$ \footnote{Note here that $\cP^{n \times n}$ is used to denote the set of $n \times n$ permutation matrices.}, as $\mY = \mA \mX = \mX \mP^T \mP \mA$ then, regardless of the choice of $\mA$, without any information in addition to $\mY$ the decoder algorithm can at best only hope to recover the sparse codes up to row permutation, meaning the location information is lost. We therefore allow the decoder and encoder to agree on a unique ordering of the columns of $\mA$, discussed in detail in Section \ref{subsec:summary_contributions}, in advance, as under this assumption recovery up to permutation suffices to achieve exact recovery.

The purpose of the rest of this section is to chart a course towards designing an encoder matrix and a decoding algorithm which, under certain assumptions on the sparse coding matrix, can solve the aforementioned matrix factorisation problem with high probability. We start by providing a background on combinatorial compressed sensing and expander graphs in Section \ref{subsec:ccs_expanders}. Building on prior work \cite{DBLP:journals/corr/abs-1804-09212}, in Section \ref{subsec:psb_model} we first introduce a random, binary matrix with the following properties: there are a fixed number $d$ ones per column, there are at most $\ceil{n/ \floor{d/m}}$ ones per row and with high probability in $n$ this random matrix is the adjacency matrix of an expander graph. We describe how to sample an encoder matrix from the distribution of this random matrix in Algorithm \ref{alg:generate_A}. We then discuss assumptions on the sparse coding matrix which we place in order to aid us in deriving uniqueness guarantees and a computationally efficient decoding algorithm: in particular, we assume that the supports of each column have cardinality $k$, are mutually independent and chosen uniformly at random across all $\binom{n}{k}$ possible supports, and that the nonzero entries in each column are dissociated. This last condition, borrowed from the field of additive combinatorics, is fairly mild and, for example, holds almost surely for mutually independent entries sampled from any continuous distribution. The distributions placed on the encoder and sparse coding matrices in turn place a distribution on the measurement matrix, we refer to this distribution as the Permuted Striped Block (PSB) model, defined in Definition \ref{def_data_model}. In Section \ref{subsec:summary_contributions} we present and discuss the key theoretical result of this paper, Theorem \ref{theorem:main}. This theorem, which we prove in Section \ref{sec:theory}, provides performance guarantees for the Decoder-Expander Based Factorisation (D-EBF) algorithm, defined and described in detail in Section \ref{sec:alg}. The key takeaway of this theorem is that D-EBF will successfully recover both the encoder and sparse coding matrices from a measurement matrix drawn from the PSB model with high probability in $n$ as long as $N=\Omega\left(\frac{n}{k}\log^2(n) \right)$. Note that this sample complexity is optimal up to additional logarithmic factors, which arise due to a variant of the coupon collector problem, encountered as a result of the distribution placed on the support of the sparse coding matrix. In addition to our theoretical analysis, in Section \ref{sec:EBF_numerics} we demonstrate the computational efficiency and efficacy of D-EBF in practice on synthetic, mid-sized problems. Finally, we remark that the problem studied in this paper has similarities with those studied in dictionary learning \cite{elad_book, 1710377, Lee07efficientsparse, 10.5555/2981780.2981909,Olshausen97sparsecoding} and in particular dictionary recovery \cite{2013arXiv1309.1952A, DBLP:journals/corr/AroraBGM14, DBLP:journals/corr/AroraGM13, DBLP:journals/corr/BarakKS14, DBLP:journals/corr/abs-1206-5882, DBLP:journals/corr/SunQW15b, DBLP:journals/corr/SunQW15c}. We discuss these connections in detail in Section \ref{subsec:related_work}.

\subsection{Compressed sensing and expander graphs} \label{subsec:ccs_expanders}
In compressed sensing many of the theorems guaranteeing the recovery of the sparse code rely on the encoder matrix satisfying one or more properties, in particular the nullspace property \cite{best_k_approx,959265}, a condition on the mutual coherence \cite{959265,Donoho2197} or the restricted isometry property (RIP) \cite{simple_proof_RIP, sharp_RIP, c2005decoding}. Although constructing matrices which satisfy a condition on the mutual coherence can be accomplished in a computationally efficient manner, it is NP-hard to compute both the nullspace constant (NSC) and restricted isometry constant (RIC) of a matrix \cite{6658871}. On the other hand, it is recognised that stronger conditions can be proved via the nullspace property and RIP \cite{intro_CS}. A perhaps surprising result is that a number of random matrix constructions, notably Gaussian, Bernoulli and partial Fourier, satisfy RIP with probability approaching one exponentially fast in the problem size \cite{intro_CS}. Furthermore, not only are these random constructions popular from the perspective of deriving improved recovery guarantees, but also play a key role in applications: for instance, Gaussian matrices achieve the optimal measurement rate $m= \cO(k \log (n/k))$ with high probability while partial Fourier matrices are the natural sensing mechanism in Tomography \cite{cs_mri}. 

Combinatorial compressed sensing (CCS) \cite{4797639} analyses the compressed sensing problem in the setting where the encoder matrix is sparse and binary. Not only are these types of encoder matrix the relevant sensing mechanism in certain applications, e.g., the single pixel camera \cite{single_pixel}, but they also have benefits in terms of reduced computation and memory overheads compared with dense alternatives. Encoder matrices of this form are often interpreted as the adjacency matrices of an unbalanced bipartite graph. Furthermore, in \cite{4797639} it was shown that when the encoder matrix is the adjacency matrix of a $(k,\epsilon,d)$-expander graph, then the optimal measurement rate can be achieved.
\begin{definition} [\textbf{(k, $\epsilon$, d)-expander  graph}] \label{def_exp}
	Consider a left d-regular bipartite graph $G=([n], [m], E)$, for any $\cS \subseteq [n]$ let $\cN(\cS) \defeq \{j \in [m]: \exists l \in \cS \; s.t. \; (l,j) \in E \}$ be the subset of nodes in $[m]$ connected to a node in $\cS$. With $\epsilon \in (0,1)$, then $G$ is a $(k, \epsilon, d)$-expander graph iff
	\begin{equation} \label{equation_exp}
	| \cN(\cS) | > (1-\epsilon)d |\cS| \quad \forall \quad \cS \in [n]^{(\leq k)}.
	\end{equation}
\end{definition}
Here $[n]^{(\leq k)}$ denotes the set of subsets of $[n]$ with cardinality at most $k$ and $\epsilon$ is the expansion parameter, which plays a role analogous to that of the RIC of a dense matrix in proving recovery guarantees \cite{4797639}. In this paper we define the \textit{adjacency matrix  of a} $(k, \epsilon, d)$\textit{-expander graph} $G = ([n], [m], E)$ as an $m \times n$ binary matrix $\mA$, where $a_{j,l} = 1$ iff there is an edge between node $j \in [m]$ and $l \in [n]$ and is $0$ otherwise\footnote{\noindent We note that the adjacency matrices of graphs are often defined to describe the connectivity between all nodes in the graph, however, as we only consider bipartite graphs, in the definition adopted here the edges between nodes in the same group are not set to zero but rather are omitted entirely.}. In all that follows we will use $\cE^{m \times n}_{k, \epsilon, d} \subset \{0,1\}^{m \times n}$ to denote the set of $(k, \epsilon, d)$-expander graph adjacency matrices of dimension $m \times n$. The existence of optimal expanders, in the sense of achieving the optimal measurement rate  $m= \cO(k \log (n/k))$, was proved in \cite{zbMATH03513685,capalbo2002randomness}. With regard to constructing such encoder matrices however, the current best deterministic constructions only achieve a measurement rate of $m= \cO(k^{1+\gamma})$ for some constant $\gamma>0$ \cite{deterministic_construction_expanders}. As a result it is common to resort to random constructions, sampling $A$ from a distribution so that with high probability $A \in \cE^{m \times n}_{k, \epsilon, d}$. A popular and natural construction in this regard is to sample $A$ so that its columns are mutually independent and identically distributed, with the support of each column being chosen uniform at random from all possible supports of size $d$. Current state of the art bounds on the probability that this particular random matrix is a $(k, \epsilon, d)$-expander satisfying the optimal measurement rate can be found in \cite{DBLP:journals/corr/abs-1804-09212}. In particular, we highlight the following result.

\begin{lemma}[\textbf{\cite[Lemma 3.1]{DBLP:journals/corr/abs-1804-09212}}]
Let $A$ be an $m \times n$ random, binary matrix with mutually independent, identically distributed columns, whose supports are drawn uniform at random across all possible supports of cardinality $d$. For any $\epsilon \in (0,1/2)$, suppose as $(k,m,n) \rightarrow \infty$ that $k/m \rightarrow \rho$ and $m/n \rightarrow \alpha_2$, where $\rho, \alpha_2 \in (0,1)$ are constants. As long as $\rho < (1-\gamma) \rho_{BT}(\alpha_2,\epsilon, d)$ for some constant $\gamma \in (0,1)$, then the probability that $A \in \cE^{m \times n}_{k, \epsilon, d} $ approaches one exponentially fast in $n$.
\end{lemma} \label{lemma:whp_expander}

Here $\rho_{BT}(\alpha_2,\epsilon, d)$ denotes the phase transition curve, which is computed numerically, see \cite[Equation 29]{DBLP:journals/corr/abs-1804-09212}. Following \cite{4797639}, a series of iterative, greedy CCS algorithms were proposed: Sparse Matching Pursuit (SMP) \cite{inproceedingsSMP}, Sequential Sparse Matching Pursuit (SSMP) \cite{inproceedingsSSMP}, Left Degree Dependent Signal Recovery (LDDSR) \cite{inproceedingsLDDSR}, Expander Recovery (ER) \cite{ER_paper} and $\ell_0$-decode \cite[Algorithms 1 \& 2]{mendoza-smith2017a}. These algorithms are specifically designed for the CCS setting and utilise certain key properties of expander graphs, namely the unique neighbour property.
\begin{theorem}[\textbf{Unique neighbour property} \textbf{\cite[Lemma 1.1]{capalbo2002randomness}}]\label{theorem_unp}
	Suppose that $G=([n], [m], E)$ is an unbalanced, left d-regular bipartite graph. Let $\cS \in [n]^{(\leq k)}$ and define
	\[
	\cN_1(\cS) = \{ j\in \cN(\cS): |\cN(j) \cap \cS | = 1\},
	\]
	where $\cN(j)$ is the subset of nodes in $[n]$ connected to the node $j\in [m]$. If $G$ is a $(k, \epsilon, d)$-expander then
	\begin{equation} \label{equation_unp}
	| \cN_1(\cS)| > (1-2\epsilon)d|\cS| \quad \forall \quad \cS \in [n]^{(\leq k)}.
	\end{equation}
\end{theorem}
\noindent A proof of Theorem \ref{theorem_unp} in the notation used here is available in \cite[Appendix A]{mendoza-smith2017a}. The unique neighbour property is used to prove that expander encoders satisfy RIP in the $\ell_1$ norm \cite{4797639}. Moreover, and critical to our purpose, if $\mA \in \cE^{m \times n}_{k, \epsilon, d} $ then the unique neighbour property ensures that certain entries in $\vx$ appear repeatedly in the measurement vector $\vy = \mA \vx$. In Section \ref{subsec:singletons_partial_supports} we study sufficient conditions for identifying such entries, which, being the sum of a single nonzero entry in $\vx$, we refer to as singleton values. The unique neighbour property is exploited by CCS algorithms such as LDDSR, ER and $\ell_0$-decode to guarantee that at each iteration a contraction in $||\vy - \mA \hat{\vx} ||_0$ occurs, where here $\hat{\vx}$ denotes the reconstruction of $\vx$. We likewise leverage this fact in the design of D-EBF: in particular, the locations in which an entry of $\vx$ appears in $\vy$ also provides access to part of the support of a column of $\mA$. The key idea behind the D-EBF algorithm is to combine and cluster partial supports, extracted from different columns of $\mY$, so as to recover the columns of $\mA$. We define and discuss in detail the D-EBF algorithm in Section \ref{sec:alg}.

\subsection{The Permuted Striped Block (PSB) model} \label{subsec:psb_model}
In this section we introduce and motivate a particular distribution over a set of measurement matrices, which we call the Permuted Striped Block (PSB) model. This distribution arises from conditions placed on both the encoder matrix, which we have freedom free to design, as well as modelling assumptions placed on the sparse coding matrix.
 
Based on the discussion at the end of Section \ref{subsec:ccs_expanders}, a strong candidate encoder is one sampled from the distribution of a random binary matrix whose column supports are mutually independent and have a fixed cardinality of size $d$. With $k,m$ and $n$ such that Lemma \ref{lemma:whp_expander} is satisfied, then for sufficiently large problems such a matrix will with high probability be the adjacency matrix of a $(k, \epsilon, d)$-expander graph and satisfy the unique neighbour property. However, the downside of this construction is that it does not rule out the possibility of there existing a very dense, in terms of the number of ones, row. The approach we adopt in the design of D-EBF relies on eventually observing each one in a column of $\mA$ through the extraction of partial supports. As a result, if a row is sufficiently dense then the one entries in said row are unlikely to ever be identified. In fact, the proof of Lemma \ref{lemma:recon_from_L}, see Appendix \ref{appendix:supporting_lemmas}, which plays a key role in providing the guarantees for D-EBF summarised in Theorem \ref{theorem:main}, relies on the number of ones per row of $\mA$ being bounded. To this end we consider a variant on the random construction already discussed which maintains the same key properties while in addition bounding the number of ones per row. How to sample from the distribution of this random matrix is described by Algorithm \ref{alg:generate_A}.

\begin{algorithm}
	\caption{GENERATE-ENCODER$(m,n,d)$} \label{alg:generate_A}
	\begin{algorithmic}[1]
	\STATE $A \leftarrow \mathrm{zeros}(m,n)$
	\STATE $\alpha \leftarrow \floor{m/d}$
	\STATE $\beta \leftarrow \ceil{n/\alpha}$
	\STATE $l \leftarrow 1$
	\FOR{$i=1:\beta$}
	\STATE $Q \sim U(\pi([m]))$ 
	\STATE $\omega \leftarrow \min \{n-l+1, \alpha \}$
	\FOR{$j=1:\omega$}
	\STATE $c \leftarrow (j-1)d$
	\FOR{$r = 1:d$}
	\STATE $A_{Q_{c+r},l} \leftarrow 1$
	\ENDFOR
	\STATE $l \leftarrow l+1$
	\ENDFOR
	\ENDFOR
	\STATE \textbf{Return} $A$
	\end{algorithmic}
\end{algorithm}

Algorithm \ref{alg:generate_A} takes as inputs the dimensions of the encoder $m$ and $n$, the column sparsity $d$, and returns a binary matrix. The variable $\alpha$ stores the number of columns whose supports can be assigned using a single permutation of $[m]$, $\beta$ is the number of permutations that are needed to be drawn in order to assign a support of cardinality $d$ to each column of $A$, $\omega$ is the number of columns of $A$ to assign a support to using the current permutation of $[m]$, $c$ denotes the entry of the permutation vector from which to start assigning a support, and finally $l$ is the index of the column next in line to be assigned a support. On line 6 a permutation of the set $[n]$ is drawn uniformly at random, note here that $Q \in [n]^n$ is a random vector holding this permutation and $Q_i$ is the $i$th element of this vector. Note also that draws of different permutations are considered to be mutually independent. The for loop on lines 8-14 then uses the permutation $Q$ to assign a support to columns $l$ through to $l+ \omega -1$ of $A$. This is then repeated by drawing another permutation until all columns of $A$ have been assigned a support. As a result, encoder matrices generated using Algorithm \ref{alg:generate_A} have a fixed number $d$ ones per column and a maximum of $\beta$ ones per row. Again we emphasise here that this upper bound on the number of ones per row of $A$ will be crucial for proving our recovery results, in particular Lemma \ref{lemma:recon_from_L}. Note also that the support of each column of $A$ is dependant on the supports of at most $\alpha -1$ other columns, and that the intersection in support of these sets of dependent columns is empty by construction. We claim here that Lemma \ref{lemma:whp_expander} applies without modification to encoder matrices generated using Algorithm \ref{alg:generate_A}. This result follows in exactly the same manner as proved in \cite{DBLP:journals/corr/abs-1804-09212} by observing that any pair of columns of $A$ generated using Algorithm \ref{alg:generate_A} either have independently drawn supports, as considered in the original case in \cite{DBLP:journals/corr/abs-1804-09212}, or, if they are dependent, then they are disjoint by construction. For brevity we do not replicate the original proof in detail here, instead referring the reader to \cite{DBLP:journals/corr/abs-1804-09212} for further details.

Turning our attention now to modelling assumptions on the sparse coding matrix, we remark that multimeasurement vector compressed sensing \cite{4014378, 1453780,4553693} also studies the compressed sensing problem in the context of a set of $N>1$ measurement vectors. Typically in this context $\mX$ is structured such that sparse recovery of one column aids the recovery of others. For example, in the joint sparse setting the columns of $\mX$ share a common support. As our goal here is to use the different measurement vectors of $\mY$ to recover different parts of the encoder matrix we adopt a very different model: in particular we assume that the supports of each column are chosen uniformly at random over all possible supports of size $k$, and that the columns are mutually independent of one another. This distribution feels neither adversarial or overly sympathetic to our objective, and is adopted in order to capture a sense of how well an algorithm might typically perform at the desired factorisation task. In addition to the distribution placed on the support, we also assume the the nonzero coefficients in each column of the sparse coding matrix are dissociated.
\begin{definition}[\textbf{Dissociated vector, see Definition 4.32 of \cite{tao_vu_2006}}]\label{def_dissociated}
	A vector $\vx \in \reals^N$ is said to be dissociated, which we will denote as $\vx \in \cX^n_k$, iff for any pair of subsets $\cT_1, \cT_2 \subseteq supp(\vx)$ it holds that $\sum_{j \in \cT_1} x_j \neq \sum_{i \in \cT_2} x_i$. For $N>1$ then $\mX \in \cX^{n \times N}_k$ iff $\vx_i \in \cX^n_k$ for all $i \in [N]$.
\end{definition}\label{def:dissociated}
Although at first glance this condition may appear restrictive, it is fulfilled almost surely for isotropic vectors and more generally for any random vector whose nonzeros are drawn independently from a continuous distribution. This property plays a key role in the design of D-EBF as discussed in Section \ref{subsec:singletons_partial_supports}, and is also adopted in \cite{mendoza-smith2017a} where it plays a key role in enabling sparse recovery in the context of massive problems, i.e., $n=2^{26}$, $m/n = 10^{-3}$ and $k/n = 0.3$, in a matter of only seconds using standard computing infrastructure. We are now ready to introduce the PSB model.

% \begin{definition}[\textbf{PSB model}] \label{def_data_model}
% 	Let $d,k,m,n, N\in \naturals$ with $k = \alpha_1 n$ and $m = \theta k \ln(n/k)$, where $\alpha_1 \in (0,1)$ and $\theta \in \reals_{>0}$ are constants satisfying $- \frac{1}{ \log(\alpha_1)}<\theta < - \frac{1}{\alpha_1 \log(\alpha_1)}$ such that, with $\rho \defeq - \frac{1}{\theta \ln(\alpha_1)}$ and $\alpha_2 \defeq - \theta \alpha_1 \log(\alpha_1)$, there exists a constant $\gamma \in (0,1)$ for which $\rho < (1 - \gamma)\rho_{BT}(\alpha_2, 1/6, d)$. Consider now the following random matrices.
% 	\begin{itemize}
% 		\item $A \defeq [ A_1 \; A_2... \; A_n]$ is a random binary matrix of size $m\times n$, sampled using Algorithm \ref{alg:generate_A}. The columns of $A$ have exactly $d$ ones while the rows have at most $\ceil{n/ \floor{d/m}}$ ones by construction.
% 		\item $X \defeq [X_1 \; X_2... \; X_N]$ is a random real matrix of size $n\times N$, whose columns are mutually independent, have exactly $k$ nonzero entries, are dissociated and whose supports are chosen uniformly at random across all possible supports with cardinality $k$.
% 	\end{itemize}
% 	A random, real $m \times N$ matrix $Y$ is sampled from the PSB model, which we will denote as $Y \sim PSB(d,k,m,n,N)$, iff $Y \defeq AX$.
% \end{definition}

\begin{definition}[\textbf{PSB model}] \label{def_data_model}
	Let $d,k,m,n, N\in \naturals$ such that $k<m<n$, $k/n= \alpha_1$ and $m/n = \alpha_2 \defeq \theta \alpha_1 \log(1/\alpha_1)$ where $\alpha_1, \alpha_2 \in (0,1)$ and $\theta \in \reals_{>0}$ are constants, and suppose there exists a constant $\gamma \in (0,1)$ such that $k/m < (1 - \gamma)\rho_{BT}(\alpha_2, 1/6, d)$. Consider now the following random matrices.
	\begin{itemize}
		\item $A \defeq [ A_1 \; A_2... \; A_n]$ is a random, binary matrix of size $m\times n$ whose distribution can be sampled from using Algorithm \ref{alg:generate_A}. The columns of $A$ each have exactly $d$ ones, while the rows have at most $\ceil{n/ \floor{d/m}}$ ones by construction.
		\item $X \defeq [X_1 \; X_2... \; X_N]$ is a random, real matrix of size $n\times N$, whose columns are mutually independent, have exactly $k$ nonzero entries, are dissociated and whose supports are chosen uniformly at random across all possible supports with cardinality $k$.
	\end{itemize}
	A random, real $m \times N$ matrix $Y$ is sampled from the PSB model, which we will denote as $Y \sim PSB(d,k,m,n,N)$, iff $Y \defeq AX$.
\end{definition}

With regard to the choice of parameters, observe that, under the assumptions listed, if an algorithm recovers the sparse codes from a measurement matrix drawn from the PSB model, then it does so at the optimal measurement rate $m = \cO(k \log(n/k))$. Second, as $k/n = \alpha _1$ then $m/n = -\theta \alpha_1 \log(\alpha_1) \eqdef \alpha_2$ and $k/m = -1/\theta\log(\alpha_1)$ are constants. Therefore, by construction, it follows that the conditions stated in Lemma \ref{lemma:whp_expander} are satisfied, and as a result the probability that $A \in \cE^{m \times n}_{k,\epsilon,d}$ with $\epsilon \leq 1/6$ approaches one exponentially fast as $n \rightarrow \infty$.

\begin{figure}[!htb]
	\centering
	\includegraphics[width=0.9\textwidth]{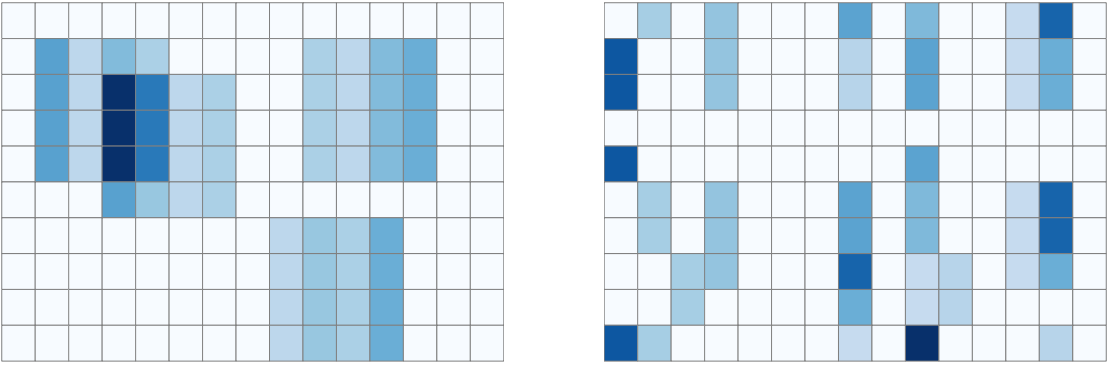}
	\caption{Visualisation of two examples of PSB matrices. In both cases the matrix consists of a sum of four rank 1 matrices, each with a support cardinality of $16$.  The left hand plot corresponds to the case where the support of each rank one matrix is arranged into a block. The right hand plot takes the same rank one matrices as in the left hand plot, but before summing them applies a random row and column permutation to each. The white squares indicate a zero entry and nonzero entries in the same column have the same coefficient value. As a result, aside from where there are overlaps, each column of a block is a single stripe of colour.}
	\label{fig:PSB_matrix}
\end{figure}

In order to explain the origin of its name, observe that matrices drawn from the PSB model can be expressed as the sum of $n$ rank one matrices. The supports of these matrices, under appropriate row and column permutations, can be organised into a single dense block. These permuted blocks are also striped in the sense that the entries in any given column of a block have the same value. We provide a visualisation in Figure \ref{fig:PSB_matrix}.

\subsection{Summary of contributions} \label{subsec:summary_contributions}
The contributions of this paper are twofold: first, we provide a novel algorithm, Decoder-Expander Based Factorisation (D-EBF), detailed in Algorithm \ref{alg:DEBF_detail} in Section \ref{subsec:EBF_detail}, which is designed to recover $\mA \in \cE_{k, \epsilon, d}^{m \times n}$ and $\mX \in \cX_k^{n \times N}$ up to permutation from the measurement matrix $\mY = \mA \mX$. Second, we analyse the performance of D-EBF in the context of measurement matrices sampled from the PSB model, proving that it recovers the matrix factors up to permutation with high probability. The theoretical guarantees for D-EBF are summarised in Theorem \ref{theorem:main}.

\begin{theorem} \label{theorem:main}
	Let $Y \sim PSB(d,k,m,n,N)$ as per Definition \ref{def_data_model}. Under the assumption that $ A \in \cE_{k, \epsilon, d}^{m \times n}$ with $\epsilon \leq 1/6$, consider the reconstructions of the matrix factors of $Y$ returned by D-EBF,
	\[
	\hat{A},\hat{X} \leftarrow  \mathrm{D}\text{-}\mathrm{EBF}(Y,m,n,N,\epsilon,d).
	\]
    Then the following statements are true.
	\begin{enumerate}
		\item \textbf{The reconstructions are accurate up to permutation:} there exists a random permutation $P \in \cP^{n \times n}$ such that $\supp(\hat{A} P^T) \subseteq \supp(A)$, $\supp(P \hat{X}) \subseteq \supp(X)$ and $\hat{x}_{j,i} = x_{j,i}$ for all $(j,i) \in \supp(\hat{X})$.
		\item \textbf{On the uniqueness of the factorisation:} if $Y = \hat{A} \hat{X}$ then this factorisation is unique up to permutation, i.e., there exists a random permutation $P \in \cP^{n \times n}$ such that $\hat{A}P^T = A$ and $P \hat{X} = X$.
		\item \textbf{D-EBF is successful with high probability in $n$:} Suppose in addition to the assumptions of the PSB model that $N \geq \nu \frac{8d}{3 \tau(n)} \frac{n}{k} (\ln^2(n) + \ln(n))$, where $\tau(n)= \cO(1)$ \footnote{ A definition of $\tau(n)$ is provided in Lemma \ref{lemma:recon_from_L}} and $\nu>1$ is a constant. Then the probability that there exists a permutation $P\in \cP^{n\times n}$ such that $\hat{A}P^T= A$ and $P\hat{X} = X$ is greater than $1 - \cO\left( n^{-\sqrt{\nu}+1} \log^2(n)\right)$
	\end{enumerate}
\end{theorem}
\noindent A few remarks regarding the D-EBF algorithm and Theorem \ref{theorem:main} are in order.
\begin{itemize}
    \item \textbf{Probability that $A$ is an expander:} recall from the discussion in Section \ref{subsec:psb_model} that, with $A$ defined as in the PSB model, then the probability that $A \in \cE_{k, \epsilon, d}^{m \times n}$ with $\epsilon \leq 1/6$ goes to one exponentially as $n \rightarrow \infty$. Furthermore, the probability of this event can be lower bounded using the bounds derived in \cite{DBLP:journals/corr/abs-1804-09212}. As a result, from the definition of conditional probability, upper bounds on the probability of all three statements of Theorem \ref{theorem:main} can easily be derived. It therefore also follows that the three statements of Theorem \ref{theorem:main} still hold with high probability even without conditioning on the event $A \in \cE_{k, \epsilon, d}^{m \times n}$ with $\epsilon \leq 1/6$. We have opted not to present Theorem \ref{theorem:main} in this form for two reasons: first in order to more clearly highlight the contributions of this paper and second because the lower bound on the probability that $A \in \cE_{k, \epsilon, d}^{m \times n}$ with $\epsilon \leq 1/6$ derived in \cite{DBLP:journals/corr/abs-1804-09212} is loose and overly pessimistic.
    \item \textbf{Sample complexity of D-EBF:} perhaps the most remarkable aspect of Theorem \ref{theorem:main} is the required sample complexity. Indeed, observing that $\nu$ and $d$ are constants and that $\tau(n)=\cO(1)$, $n/k = \cO(1)$, then as long as $N = \Omega(\frac{n}{k}\log^2(n))$ D-EBF succeeds in recovering $A$ and $X$ up to permutation with high probability. This is equal to the lower bound $N \geq \frac{n}{k}$ up to logarithmic factors. For measurement matrices drawn from the PSB model we hypothesise that this is likely to be optimal as one $\log(n)$ factor arises from a coupon collector argument, inherent to the way $X$ is sampled, and the second $\log(n)$ factor is needed to achieve the stated rate of convergence in probability.
    \item \textbf{Computational complexity of D-EBF:} in Section \ref{subsec:EBF_detail} we will ascertain that the per while loop iteration complexity cost of D-EBF is $\cO(k^2(n +N)N)$. A trivial asymptotic upper bound of $\cO(nk^2(n +N)N)$ follows from our method of proof, however this is highly pessimistic and analysing and bounding the number of iterations of the while loop of D-EBF in a meaningful manner is beyond the scope of Theorem \ref{theorem:main}. Our preliminary experimental investigations on moderately sized problems indicate that the number of iterations required in practice is not onerous and, as expected, reduces as $N$ increases. We leave a proper analysis to future work.
    \item \textbf{Input arguments required by D-EBF:} observe that D-EBF requires not only the product matrix $Y$ and its dimensions, but also certain parameters of the encoder matrix: namely the number of columns $n$, the expansion parameter $\epsilon$ and the column sparsity $d$. As will be discussed in Section \ref{sec:EBF_numerics}, in practice $n$ is not required. Indeed, ND-EBF can be run in an online fashion, processing input vectors as and when they are received and dynamically updating a list of column vectors representing $\hat{\mA}$. Although the the column sparsity $d$ is known in advance by the encoder, and can therefore be reasonably agreed upon in advance with the decoder, the same cannot be said of the expansion parameter $\epsilon$, as even with access to $\mA$ computing $\epsilon$ is an NP-complete problem \cite{alon_expanders}. However, as discussed and demonstrated empirically in section \ref{sec:EBF_numerics}, in practice ND-EBF can still function effectively given access only to an upper bound on $\epsilon$ instead. In our experiments we use $1/6$ for simplicity but tighter upper bounds on the expansion parameter can be computed with relative ease, see e.g., \cite{expansion_bipartite}.
\end{itemize} 

As shall be discussed in more detail in Section \ref{sec:EBF_numerics}, as long as $\mA \in \cE_{k, \epsilon, d}^{m \times n}$ with $\epsilon \leq 1/6$ and $N$ is sufficiently large, then in practice D-EBF requires only knowledge of the column sparsity $d$ in advance to compute the factors $\mA$ and $\mX$ of $\mY$ up to permutation. In terms of applications, then from the permuted sparse code one is still able to recover and compute key statistics concerning the original sparse codes: for example the average value of the nonzero entries, the maximum value or a histogram of the non-zero entries. As will also be discussed in Section \ref{sec:EBF_numerics}, D-EBF can also be adapted to operate in an online fashion. In this setting one could consider the columns of $\mY$ as linear measurement vectors of a system's state vector, which evolves in time and is sparse. Examples of such systems might include a network of sensors or the actions of users in a social network. Assuming that this system is measured using a suitable encoder matrix, which is fixed across time, then Theorem \ref{theorem:main} implies that given sufficient samples then the activation statistics of all samples to date, as well as all future ones, could be computed without ever having explicit access to the encoder matrix. This would be necessary in situations where the encoder matrix is unknown to or private from the would be observer of the system, and may be advantageous from an efficiency perspective if the encoder matrix is costly or difficult to transmit. 

In the context of compressed sensing, the fact that the D-EBF algorithm is only able to recover the sparse codes up to permutation may be problematic. However, if the encoder and decoder agree in advance on a protocol with which to identify the original encoder matrix then this issue can be overcome. One such protocol to this end, based a particular ordering of the columns, is as follows.
\begin{enumerate}
    \item The encoder generates an encoder matrix $A$ as per the PSB model and then, interpreting each column as a rational number written in binary, reorders the columns by the size of their corresponding rational number, for example, from largest to smallest. The specific ordering used is agreed in advance with the decoder.
    \item The encoder then generates $Y$ by multiplying $X$ by the ordered encoder matrix and then sends $Y$ and $d$ to the decoder.
    \item The decoder applies the D-EBF algorithm to recover the encoder matrix and sparse codes up to permutation. The decoder then interprets each of the columns of the reconstructed encoder matrix as a rational number written in binary, and reorders the columns of $\hat{A}$ and the corresponding rows of $\hat{X}$ by the size of their corresponding rational number.
\end{enumerate}
If $A \in \cE_{k, \epsilon, d}^{m \times n}$ with $\epsilon \leq 1/6$ then each column of $A$ maps to a unique rational number and therefore the columns of $A$ have a unique ordering. As a result, if the decoder knows the column ordering protocol, e.g., smallest to largest or largest to smallest, and sparsity $d$ in advance, then recovery up to permutation implies full recovery of the location information also.

\subsection{Related work} \label{subsec:related_work}
In regard to connections with other matrix factorisation problems and methods, the PSB model and associated factorisation task are most closely related to those in which sparsity is also a prominent feature, such as in dictionary learning and subspace clustering. Dictionary learning \cite{elad_book, 1710377, Lee07efficientsparse, 10.5555/2981780.2981909,Olshausen97sparsecoding} is a prominent matrix factorisation technique in data science, in which the columns of the observed data matrix are assumed to lie, at least approximately, on a union of low dimensional subspaces. This structure can be expressed as the product of an overcomplete matrix, known as a dictionary, and a sparse matrix or code. To be clear, given a data matrix $\mY$ dictionary learning methods seek to compute a dictionary $\mA$ and a sparse coding matrix $\mX$ such that $\mY \approx \mA \mX$. Compared with dictionary learning, subspace clustering \cite{5206547, Parsons2004SubspaceCF, Vidal_atutorial} adopts the additional structural assumption that the data lies on a union of low dimensional subspaces which are independent of one another.

More specifically, the literature most relevant to this work is that on dictionary recovery, a field aiming to provide recovery guarantees for dictionary learning. Compared with dictionary learning, dictionary recovery \cite{2013arXiv1309.1952A, DBLP:journals/corr/AroraBGM14, DBLP:journals/corr/AroraGM13, DBLP:journals/corr/BarakKS14, DBLP:journals/corr/abs-1206-5882, DBLP:journals/corr/SunQW15b, DBLP:journals/corr/SunQW15c} presupposes that $\mY \defeq \mA \mX$ with the goal being to recover $\mA$ and $\mX$ up to permutation. It is common in dictionary recovery to place further structural assumptions on the factor matrices so as to facilitate the development of strong guarantees, albeit at the expense of model expressiveness. Two popular structural assumptions are that the dictionary is complete \cite{DBLP:journals/corr/abs-1206-5882, DBLP:journals/corr/SunQW15b, DBLP:journals/corr/SunQW15c} or incoherent \cite{2013arXiv1309.1952A, DBLP:journals/corr/AroraGM13}. Similar to the PSB model, the sparse coding matrix is often assumed to have been drawn from a sparse distribution, e.g., Bernoulli-Gaussian \cite{2013arXiv1309.1952A, DBLP:journals/corr/AroraGM13} or Bernoulli-uniform \cite{DBLP:journals/corr/abs-1206-5882, DBLP:journals/corr/SunQW15b, DBLP:journals/corr/SunQW15c}. This work diverges from the majority of the literature on dictionary recovery in regard to the construction of the dictionary: in particular, while prior work uses incoherence or completeness to design algorithms and derive guarantees, we instead operate under the assumption that the encoder matrix is the adjacency matrix of a $(k, \epsilon,d)$-expander graph. We therefore use a very different set of ideas and tools in our algorithmic approach and proof technique.

In the context of matrix factorisation and dictionary recovery, we are aware of only one other work in which $A$ is constructed so as to leverage the properties of expander graphs \cite{DBLP:journals/corr/AroraBGM13}. In this work the authors, guided by the notion of reversible neural networks, considered the learning of a deep network as a layerwise nonlinear dictionary learning problem. This work differs substantially from \cite{DBLP:journals/corr/AroraBGM13} in a number of respects. First, in terms of the problem setup, the primary difference is that in \cite{DBLP:journals/corr/AroraBGM13} the problem of recovering $A$ and $X$ from $Y=\sigma \left( A X \right)$ is studied, where $\sigma$ denotes the elementwise unit step function and both $X$ and $Y$ are binary. Second, in regard to algorithmic approach, in this prior work $A$ is recovered up to permutation using a non-iterative approach, involving first the computation of all row wise correlations of $Y$ and then applying a clustering technique adapted from the graph square root problem. By contrast, and as described in detail in Section \ref{subsec:EBF_detail}, our method iteratively recovers parts of $A$ and $X$ directly from the columns of the residual by simultaneously leveraging both the unique neighbour property of $A$ and the fact that the columns of $X$ are dissociated.

\section{Decoder-Expander Based Factorisation (D-EBF)} \label{sec:alg}
In this section we present the Decoder-Expander Based Factorisation (D-EBF) algorithm as a method for factorising $\mY = \mA \mX$ where $\mA \in \cE_{k, \epsilon, d}^{m \times n}$ and $\mX \in \cX^{n \times N}_{k}$. In particular, in Section \ref{subsec:EBF_overview} we provide a high level overview of the D-EBF algorithm, before describing in detail the key steps in Sections \ref{subsec:singletons_partial_supports} and \ref{subsec:l0_decode}. Finally, we provide a full definition of D-EBF in Section \ref{subsec:EBF_detail}.

% ~~~~~~~~~~~~~~~~~~~~~~~~~~~~~~~~~~~~~~~~~~~~~~~~~~~~~~

\subsection{Overview of the D-EBF algorithm}\label{subsec:EBF_overview}
The D-EBF algorithm iteratively computes a reconstruction of $\mA$ and $\mX$ by recovering parts of them from the observed matrix $\mY$, removing the contributions of these parts to form a residual, and then repeating these steps on said residual. The D-EBF algorithm starts each iteration by searching, independently, the columns of the residual for singleton values - nonzeros which are the sum of one nonzero in the corresponding column of $\mX$. In Section \ref{subsec:singletons_partial_supports} we prove, under certain assumptions, that singleton values can be identified by the number of times they appear in a given column of the residual. In addition, as the columns of $\mX$ are dissociated, then the locations in which a singleton value appears indicate part of the support of a column of $\mA$. Furthermore, if a singleton value appears in sufficiently many locations then the associated partial support is unique to one column of $\mA$. After identifying these frequently occurring singleton values from the columns of the residual, the D-EBF algorithm clusters the associated partial supports by their unique column of $\mA$. The union of the partial supports in each cluster is then used to reconstruct, at least partially, a column of $\mA$. The completely recovered columns, i.e., those whose supports have cardinality $d$, are then used to recover further entries of $\mX$ using a sparse decoding algorithm from the combinatorial compressed sensing literature \cite{inproceedingsSMP, inproceedingsSSMP,mendoza-smith2017a, inproceedingsLDDSR, ER_paper}. Finally, the contributions from the recovered parts of $\mA$ and $\mX$ are removed from $\mY$ to form a new residual. This process is then repeated on the subsequent residuals until either no additional entries from $\mA$ or $\mX$ are recovered, or $\mY$ is fully decoded. We summarise this high level approach in Algorithm \ref{alg:high_level_l0-EBF}. 

\begin{algorithm}
	\caption{High level overview of the D-EBF algorithm} \label{alg:high_level_l0-EBF}
	\begin{algorithmic}[1]
	    \STATE Initialise residual as observed data $\mY$
		\WHILE{\text{not converged}}
		\STATE Extract singleton values and associated partial supports from the residual
		\STATE Cluster the partial supports by their column of $\mA$
		\STATE Update the reconstructions of $\mA$ and $\mX$ using the clustered partial supports and the singleton values respectively
		\STATE Run a decode algorithm using the fully recovered columns of $\mA$ to recover more entries of $\mX$.
		\STATE Update residual by removing contributions from recovered parts of $\mA$ and $\mX$
		\ENDWHILE
	\end{algorithmic}
\end{algorithm}

In the following sections we proceed to define each of the high level steps of Algorithm \ref{alg:high_level_l0-EBF} in detail. In Section \ref{subsec:singletons_partial_supports} we study steps 3, 4 and 5, which concern the extraction and clustering of singleton values and partial supports. In Section \ref{subsec:l0_decode} we consider step 6, the decoding step, and review sparse decoding algorithms from the combinatorial compressed sensing literature. Finally, in Section \ref{subsec:EBF_detail} we present a detailed version of Algorithm \ref{alg:high_level_l0-EBF}.
% ~~~~~~~~~~~~~~~~~~~~~~~~~~~~~~~~~~~~~~~~~~~~~~~~~~~~~~
\subsection{Singleton values and partial supports} \label{subsec:singletons_partial_supports}
As referenced to in Section \ref{subsec:ccs_expanders}, expander graphs play major role both in the design of D-EBF and the proof of Theorem \ref{theorem:main}. In Lemma \ref{lemma:adjmat} we summarise key properties of the adjacency matrices of $(k, \epsilon, d)$-expander graphs.

\begin{lemma}[\textbf{Properties of the adjacency matrix of a $(k, \epsilon, d)$-expander graph}] \label{lemma:adjmat}
	If $\mA \in \cE^{m \times n}_{k, \epsilon, d}$ then any submatrix $\mA_{\cS}$ of $\mA$, where $\cS \in [n]^{(\leq k)}$, satisfies the following.
	\begin{enumerate}
		\item There are more than $(1-\epsilon)d |\cS|$ rows in $\mA_{\cS}$ that have \textbf{at least one} non-zero.
		\item There are more than $(1-2\epsilon)d |\cS|$ rows in $\mA_{\cS}$ that have \textbf{only one} non-zero.
		\item The overlap in support of the columns of $\mA_{\cS}$ is upper bounded as follows,
		\[
		|\bigcap_{l \in \cS} \supp(\va_l)| < 2 \epsilon d.
		\]
	\end{enumerate}
\end{lemma}

A proof of Lemma \ref{lemma:adjmat} is provided in Appendix \ref{appendix:singletons_partials}. The inspiration for studying singleton values arises from property 2 of Lemma \ref{lemma:adjmat}. Indeed, by conditioning on $\mA$ being the adjacency matrix of a $(k, \epsilon, d)$-expander graph, then the submatrix of $\mA$ associated with the support of a column of $\mX$ has potentially many rows with a single nonzero entry. This implies that there are potentially many entries in each column of $\mY$ which correspond directly to an entry in $\mX$. Denoting the $j$th row of $\mA$ as $\tilde{\va}_j$, we are now ready to provide the definition of a singleton value.

\begin{definition}[\textbf{Singleton value}] \label{def_singletons}
	Consider a vector $\vr = \mA \vz$ where $\mA \in \{0,1\}^{m \times n}$ and $\vz \in \reals^n$. A singleton value of $\vr$  is an entry $r_j\neq 0$ such that $|\supp(\tilde{\va}_j) \cap \supp(\vz)| = 1$, hence $r_j = x_l$ for some $l \in \supp(\vz)$.
\end{definition}
We call a binary vector whose nonzeros coincide with some subset of the nonzeros of a column of $\mA$ a partial support of that column.
\begin{definition}[\textbf{Partial support}]\label{def_ps}
	A partial support $\vw \in \{0,1\}^{m}$ of a column $\va_l \in \{0,1\}^{m}$ of $\mA \in \{0,1\}^{m \times n}$ is a binary vector satisfying $\supp(\vw) \subseteq \supp(\va_l)$. Furthermore, $\vw$ is said to originate from $\va_l$ iff  $\supp(\vw) \subseteq \supp(\va_l)$ and  $\supp(\vw) \nsubseteq \supp(\va_h)$ for all $h \in [n]\backslash\{ l\}$.
\end{definition}
We now introduce some useful notation for what follows.
\begin{itemize}
    \item Let the function $f: \reals \times \reals^m \rightarrow [m]$ count the number of times a real number $\alpha \in \reals$ appears in some vector $\vr \in \reals^m$, i.e., $f(\alpha, \vr) \defeq | \{j \in [m]: r_j = \alpha \}|$.
    \item Let the function $g: \reals \times \reals^m \rightarrow \{0,1 \}^m$ return a binary vector whose nonzeros correspond to the locations in which $\alpha \in \reals$ appears in a vector $\vr \in \reals^m$, i.e., with $\vw = g(\alpha, \vr)$ then $w_j = 1$ iff $r_j = \alpha$ and is $0$ otherwise.
\end{itemize}

Observe that if $\vz$ is dissociated then the locations in which a singleton value appears in $\vr$ defines a partial support of a column of $\mA$. Furthermore, under the assumption that $\mA$ is the adjacency matrix of a $(k, \epsilon, d)$-expander graph and that $\vz$ is dissociated, it is possible to derive the following sufficient condition with which to identify singleton values.

\begin{corollary}[\textbf{Sufficient condition for identifying singleton values (i)}] \label{corollary:suff_cond_singletons}
    Consider a vector $\vr = \mA \vz$ where $\mA \in \cE^{m \times n}_{k, \epsilon, d}$ and $\vz\in \cX^n_k$. With $\alpha \in \reals \backslash \{0\}$ if $f(\alpha, \vr)\geq 2 \epsilon d$ then there exists an $l \in \supp(\vx)$ such that $\alpha = x_l$.
\end{corollary}

A proof of Corollary \ref{corollary:suff_cond_singletons} is provided in Appendix \ref{appendix:singletons_partials}. Under the assumptions stated this corollary implies that any value which appears in $\vr$ at least $2\epsilon d$ times must be a singleton value. However, this corollary does not provide insight into whether or not there exist singletons values in $\vr$ appearing in at least $2 \epsilon d$ locations. To resolve this matter, in Lemma \ref{lemma_existence_singletons}, adapted from \cite[Theorem 4.6]{mendoza-smith2017a}, we show, under certain assumptions, that there always exist a positive number of singleton values which appear more than $(1-2 \epsilon)d$ times in $\vr$. This result is needed for the proof of Theorem \ref{theorem:main}.

\begin{lemma}[\textbf{Existence of frequently occurring singleton values, adapted from \cite[Theorem 4.6]{mendoza-smith2017a}}] \label{lemma_existence_singletons}
Consider a vector $\vr = \mA \vz$ where $\mA \in \cE^{m \times n}_{k, \epsilon, d}$ and $\vz \in \cX^n_k$. For $l \in \supp(\vz)$, let $\Omega_l \defeq \{j \in [m]: r_j= z_l\}$ be the set of row indices $j \in [m]$ of $\vr$ such that $r_j= z_l$. Defining $\cT \defeq \{l \in \supp(\vz): |\Omega_l| > ( 1-2 \epsilon) d \}$ as the set of singleton values which each appear more than $(1-2\epsilon)d$ times in $\vr$, then
	\[
	| \cT | \geq \frac{|\supp(\vz)|}{(1+2\epsilon)d}.
	\]
\end{lemma}

A proof of Lemma \ref{lemma_existence_singletons} is provided in Appendix \ref{appendix:singletons_partials}. If $\epsilon \leq 1/4$ then $(1-2\epsilon)d \geq 2 \epsilon d$, combining this with Corollary \ref{corollary:suff_cond_singletons} and Lemma \ref{lemma_existence_singletons} guarantees, under the necessary assumptions, that it is always possible to identify at least one singleton value which appears more than $2 \epsilon d$ times in $\vr$. The results presented so far concerning the extraction of singleton values and partial supports assume that the residual analysed is of the form $\vr = \mA \vz$, where $\mA \in \cE^{m \times n}_{k, \epsilon, d}$ and $\vz \in \cX^n_k$. Due to the iterative nature of D-EBF, typically the residual under consideration is instead of the form $\vr = \vy  - \hat{\mA} \hat{\vx}$, where $\hat{\mA}$ and $\hat{\vx}$ are the estimates or reconstructions of $\mA$ and $\vx$ respectively. Under certain assumptions, the following result maintains that the sufficient condition given in Corollary \ref{corollary:suff_cond_singletons} also holds for residuals of this form.

\begin{lemma}[\textbf{Sufficient condition for identifying singleton values (ii)}]\label{lemma:freq_non_singletons_residual}
Consider a vector $\vy = \mA \vx$ where $\mA \in \cE^{m \times n}_{k, \epsilon, d}$ and $\vx \in \cX^n_k$. Let $\hat{\mA} \in \{0,1\}^{m \times n}$ and $\hat{\vx}\in\cX^n_k$ be such that a column $\hat{\va}_l$ of $\hat{\mA}$ is nonzero iff $\hat{x}_l\neq 0$, and there exists a permutation matrix $\mP \in \cP^{n \times n}$ such that $\supp(\hat{\mA}\mP^T) \subseteq \supp(\mA)$, $\supp(\mP\hat{\vx}) \subseteq \supp(\vx)$ and $\hat{x}_{P(l)} = x_{P(l)}$ for all $l \in supp(\hat{\vx})$, where $P:[n] \rightarrow [n]$ denotes the row permutation caused by pre-multiplication with $\mP$. Consider the residual $\vr = \vy - \hat{\mA} \hat{\vx}$, if for some $\alpha \in \reals \backslash \{ 0\} $ it holds that $f(\alpha, \vr) \geq 2 \epsilon d$ is satisfied, then there exists an $l \in [n]$ such that $\alpha = x_l$. Furthermore, with $\vw = g(\alpha, \vr)$ then $\supp(\vw) \subseteq \supp(\va_l)$ is a partial support of $\va_l$.
\end{lemma}

A proof of Lemma \ref{lemma_existence_singletons} is provided in Appendix \ref{appendix:singletons_partials}. To perform step 4 of Algorithm \ref{alg:high_level_l0-EBF}, it is necessary that the partial supports extracted across all columns of $\mY$, or the residual of $\mY$, can be clustered accurately and efficiently. To this end we present Lemma \ref{corollary_sort_ps}, which provides a necessary and sufficient condition with which to cluster sufficiently large partial supports.

\begin{lemma}[\textbf{Clustering partial supports}] \label{corollary_sort_ps}
 Consider a pair of partial supports $\vw_1$ and $\vw_2$, extracted from $\vr_1 = \mA \vz_1$ and $\vr_2 = \mA \vz_2$ respectively, where $\mA \in \cE^{m \times n}_{k, \epsilon, d}$ and $\vz_1, \vz_2 \in \cX^n_k$. If $\epsilon \leq 1/6$, $|\supp(\vw_1)| >(1 - 2 \epsilon)d$ and $|\supp(\vw_2)| >(1 - 2 \epsilon)d$, then $\vw_1$ and $\vw_2$ originate from the same column of $\mA$ iff $\vw_1^T \vw_2 \geq 2 \epsilon d$.
\end{lemma}

A proof of Lemma \ref{corollary_sort_ps} is provided in Appendix \ref{appendix:singletons_partials}. As highlighted in the proof, $\vw_1^T\vw_2 \geq 2 \epsilon d$ is sufficient to conclude that any two partial supports originate from the same column. However, without adding the additional assumptions on $\epsilon$ and on the size of the partial supports, we cannot conclude that $\vw_1^T \vw_2 < 2 \epsilon d$ implies $\vw_1$ and $\vw_2$ do not originate from the same column. This limits us, for now, to clustering only fairly large partial supports, which have at least $2/3$rds of the total nonzero entries of their respective columns of $\mA$. Fortunately, under the same conditions as Lemma \ref{corollary_sort_ps}, Lemma \ref{lemma_existence_singletons} guarantees the existence of partial supports of this size. Note however, that if we have access to a complete column $\va_l$ of $\mA$, then $\vw_1^T\va_l\geq 2 \epsilon d$ is sufficient to conclude that $\vw_1$ originates from $\va_l$. This implies that it is possible to assign partial supports consisting of only around $1/3$rd of the total nonzeros to a column of $\mA$. For this reason we will later introduce the decode step, discussed in section \ref{subsec:l0_decode}, to try and match smaller partial supports to the columns of $\hat{\mA}$ which have $d$ nonzeros.

We are now in a position to define in detail the subroutines corresponding to steps 3, 4 and 5 of Algorithm \ref{alg:high_level_l0-EBF}. We emphasise before proceeding that the subroutines we will present are designed around the assumption that $\mA \in \cE^{m \times n}_{k, \epsilon, d}$ with $\epsilon \leq 1/6$. Consider the column vector $\vy = \mA \vx$ and suppose that $\hat{\mA}$ and $\hat{\vx}$ are the current reconstructions of $\mA$ and $\vx$. Algorithm \ref{alg:extract} processes the residual $\vr = \vy - \hat{\mA}\hat{\vx}$, taking as input arguments $\vr$, the dimension $m$ of $\vr$, the expansion parameter $\epsilon$, the column sparsity $d$, as well as the current reconstructions $\hat{\mA}$ and $\hat{\vx}$. The algorithm identifies singleton values and partial supports, either matching them with a nonzero column of $\hat{\mA}$ or identifying them as belonging to a column of $\mA$ not yet observed. To this end Algorithm \ref{alg:extract} returns the unmatched singleton values and partial supports, stored as entries of the vector $\vq$ and columns of the matrix $\mW$ respectively, the number of unmatched partial supports $p$, the updated reconstructions $\hat{\mA}$ and $\hat{\vx}$ and a Boolean variable UPDATED, which indicates whether or not the reconstructions have in fact been updated. The outer for loop and subsequent if statement, lines 5 and 6 of Algorithm \ref{alg:extract} respectively, iterate through each nonzero entry of the input vector $\vr$ to check for singleton values: this is performed by the inner for loop, starting on line 10, which identifies the locations in $\vr$ where the current entry being checked appears. The variable $c$ counts the number of appearances of the current entry; if $c>(1-2\epsilon)d$ then it is presumed that the current entry is a singleton value. Subsequently, on lines 18 and 19 the column $\hat{\va}_{\kappa}$ of $\hat{\mA}$ whose support overlaps most with the associated candidate partial support is identified. Considering line 20, then if the overlap is sufficiently large, i.e., at least $2 \epsilon d$ as per Lemma \ref{corollary_sort_ps}, then the partial support is assumed to originate from the identified column $\hat{\va}_{\kappa}$ and the reconstructions are updated accordingly as per lines 21 and 22. Note here that $\sigma: \reals \rightarrow \reals$ denotes the unit threshold function, meaning $\sigma(x) = 1$ for all $x\geq1$ and is $0$ otherwise. Note also, guided by Lemma \ref{corollary_sort_ps} and given that $\hat{\mA}$ may contain incomplete columns \footnote{Observe we can only guarantee that a nonzero column of $\hat{\mA}$ has more than $(1-2 \epsilon)d$ nonzeros.}, that Algorithm \ref{alg:extract} only attempts to match partial supports with cardinality larger than $(1-2\epsilon)d$ to a column of $\hat{\mA}$. If the overlap is not sufficient, then, again guided by Lemma \ref{corollary_sort_ps}, we may conclude that the partial support under consideration does not match any existing nonzero column of the reconstruction. As a result, on lines 25 and 26 the current entry of the residual and its associated candidate partial support are added to $\vq$ and $\mW$ respectively. On line 27 the variable $p$, used to count the number of unmatched singleton values extracted from $\vr$, is updated accordingly. If $c\leq (1-2\epsilon)d$ then the for loop starting on line 30 checks if the entry under consideration corresponds to a previously identified singleton value, now stored in $\hat{\vx}$. If so, then on line 32 the partial support associated with the current entry is used to update the appropriate column of $\hat{\mA}$, regardless of its size.

\begin{algorithm}
	\caption{EXTRACT\&MATCH$(\vr,m,\epsilon,d,\hat{\mA}, \hat{\vx})$} \label{alg:extract}
	\begin{algorithmic}[1]
    \STATE $\vq \leftarrow [\;]$
    \STATE $\mW \leftarrow [\;]$
    \STATE $p \leftarrow 0$
    \STATE UPDATED$\leftarrow$FALSE
	\FOR{$i=1:m$}
	\IF{$r_i \neq 0$}
	\STATE $c \leftarrow 1$
	\STATE $\vw \leftarrow \text{zeros}(m)$
	\STATE $w_i \leftarrow 1$
	\FOR{$j=(i+1):m$}
	\IF{$r_i=r_j$}
	\STATE $w_j \leftarrow 1$
	\STATE $r_j \leftarrow 0$
	\STATE $c \leftarrow c+1$ 
	\ENDIF
	\ENDFOR
	\IF{$c >(1-2\epsilon)d$}
    \STATE $\vt \leftarrow \hat{\mA}^T \vw$
    \STATE $\kappa \leftarrow \text{argmax}_{k \in [n]}\{ t_k \}$
    \IF{$t_{\kappa} \geq 2\epsilon d$}
    \STATE $\hat{x}_{\kappa} \leftarrow r_i$
    \STATE $\hat{\va}_{\kappa} \leftarrow \sigma(\hat{\va}_{\kappa} + \vw)$
    \STATE UPDATED$\leftarrow$TRUE
    \ELSE
    \STATE $\vq \leftarrow [\vq^T ; r_i]$
    \STATE $\mW \leftarrow [\mW ; \vw]$
    \STATE $p \leftarrow p+1$
    \ENDIF
	\ELSE
	\FOR{$l \in \supp(\hat{\vx})$}
	\IF{$r_i = \hat{x}_l$}
	\STATE $\hat{\va}_l \leftarrow \sigma(\hat{\va}_l + \vw)$
	\STATE UPDATED$\leftarrow$TRUE
	\ENDIF
	\ENDFOR
	\ENDIF
	\ENDIF
	\ENDFOR
	\STATE Return $\vq, \mW,p, \hat{\mA}, \hat{\vx}$, UPDATED
	\end{algorithmic}
\end{algorithm}

In regard to the computational complexity of Algorithm \ref{alg:extract}, if the sparsity of $\vr$, $\hat{\mA}$ and $\hat{\vx}$ is not taken advantage of then the for loops starting on lines 5 and 10 iterate through $\cO(m)$ entries, and the for loop starting on line 30 through $\cO(n)$ entries. As the latter two for loops are nested inside the first, and $m =\cO(n)$, then this corresponds to $\cO(mn)$ for loop iterations. Based on the matrix vector product on line 18, each iteration has a cost of $\cO(mn)$, therefore the total cost of Algorithm \ref{alg:extract} is $\cO(m^2n^2)$. If the sparsity of $\vr$, $\hat{\mA}$ and $\hat{\vx}$ is taken advantage of, then as $\vr$ has at most $kd$ nonzero entries and $d$ is constant, the for loops starting on lines 5, 10 and 30 each iterate through $\cO(k)$ nonzero entries. Again, as the latter two for loops are nested inside the first this corresponds to $\cO(k^2)$ for loop iterations. Based on the sparse matrix vector product on line 18, each iteration has a cost of $\cO(n)$, therefore when sparsity is leveraged the total cost of Algorithm \ref{alg:extract} is $\cO(k^2n)$.

We now turn our attention to clustering the unmatched partial supports and singleton values extracted across all columns of the residual $\mR = \mY - \hat{\mA}\hat{\mX}$. This step of D-EBF is performed by Algorithm \ref{alg:cluster}, which takes as inputs the unmatched singleton values, stored in the entries of the vector $\vq \in \reals^p$, the unmatched partial supports stored in the columns of the matrix $\mW \in \{0,1\}^{m \times p}$, ordered so that $\vw_i$ is the partial support associated with the singleton $q_i$, the number of unmatched singleton values $p$, the smallest zero column index $\eta \defeq \min_{l \in [n]}\{\hat{\va}_l = \textbf{0}_m\}$, the vector $\vh \in [N]^p$ which stores the indices of the columns of $\mR$ from which each singular value was extracted, i.e., $h_i$ is the column index of $\mR$ from which $q_i$ was extracted, the expansion parameter $\epsilon$ and column sparsity $d$ of $\mA$, and the current reconstructions $\hat{\mA}$ and $\hat{\mX}$. The $p$ dimensional binary vector $\vb$, initialised as the zero vector on line 1 of Algorithm \ref{alg:cluster}, is used to indicate whether or not a partial support has been assigned to a column of $\hat{\mA}$, one indicating true and zero false. The entries of the gram matrix $\mG$, calculated on line 2, are used to cluster the unmatched partial supports. The outer for loop, starting on line 3, first checks on line 4 if the current partial support under consideration has been used to update a column of the reconstruction already. If it hasn't, then on line 5 it is used as the basis for a new nonzero column $\hat{\va}_{\eta}$ of $\hat{\mA}$. On line 6 the corresponding singleton value is then used to update the appropriate entry in $\hat{\mX}$. The inner for loop, starting on line 8, checks the inner product of the partial support in question with all other partial supports that have not already been processed, as per the if statement on line 9. On line 10, any partial support that overlaps the new nonzero column $\hat{\va}_{\eta}$ sufficiently, as per Lemma \ref{corollary_sort_ps}, and which has not yet been assigned to a column, is then used to update $\hat{\va}_{\eta}$. On line 11 the corresponding singleton values are used to update the appropriate entries in $\hat{\mX}$. Once this process is complete. then on line 15 $\eta$ is updated accordingly, with new nonzero columns being introduced to $\hat{\mA}$ from left to right.

\begin{algorithm}
	\caption{CLUSTER\&ADD$(\vq, \mW, p, \eta, \vh, \epsilon, d,\hat{\mA}, \hat{\mX})$} \label{alg:cluster}
	\begin{algorithmic}[1]
	\STATE $\vb \leftarrow \text{zeros}(p)$
	\STATE $\mG \leftarrow \mW^T\mW$
	\FOR{$i=1:p$}
	\IF{$b_i = 0$}
	\STATE $\hat{\va}_{\eta} \leftarrow \vw_i $
	\STATE $\hat{x}_{\eta, h_{i}} \leftarrow q_i$
	\STATE $b_i \leftarrow 1$
	\FOR{$j=(i+1):p$}
	\IF{$g_{i,j} \geq 2\epsilon d$ \textbf{and} $b_j = 0$}
	\STATE $\hat{\va}_\eta \leftarrow \sigma(\hat{\va}_{\eta} + \vw_j)$
	\STATE $\hat{x}_{\eta, h_{j}} \leftarrow q_j$
	\STATE $b_j \leftarrow 1$
	\ENDIF
	\ENDFOR
	\STATE $\eta \leftarrow \eta +1$
	\ENDIF
	\ENDFOR
	\STATE Return $\hat{\mA},\hat{\mX},\eta$
	\end{algorithmic}
\end{algorithm}

The main expense of Algorithm \ref{alg:cluster} is the computation of the inner products between different partial supports, equivalent to the matrix product on line 2. As $p = \cO(kN)$ then if the sparsity of $\mW$ is not leveraged this matrix product costs $\cO(k^2 mN^2)$. If $\mW$ is stored using a sparse format, as would be natural, then as each column of $\mW$ is $d$ sparse with $d$ fixed then the cost can be reduced to $\cO(k^2N^2)$.

% ~~~~~~~~~~~~~~~~~~~~~~~~~~~~~~~~~~~~~~~~~~~~~~~~~~~~~~
\subsection{Using a decoder algorithm to improve recovery} \label{subsec:l0_decode}

\begin{algorithm}
	\caption{DECODE$(\vy,m, \epsilon, d, \hat{\mA},\hat{\vx})$} \label{alg:decode}
	\begin{algorithmic}[1]
	\STATE $\cS \leftarrow \{l \in [n]: |\supp(\va_l)| = d \}$ 
	\STATE UPDATED$\leftarrow$FALSE
	\STATE RUN$\leftarrow$TRUE
	\WHILE{RUN = TRUE}
	\STATE RUN$\leftarrow$FALSE
	\STATE $\vr \leftarrow \vy - \hat{\mA} \hat{\vx}$
	\IF{$\vr \neq \textbf{0}_m$}
	\FOR{$i=1:m$}
	\IF{$r_{i} \neq 0$}
	\STATE $c \leftarrow 1$
	\STATE $\vw \leftarrow \text{zeros}(m)$
	\STATE $w_i \leftarrow 1$
	\FOR{$j=(i+1):m$}
	\IF{$r_{i}=r_{j}$}
	\STATE $w_j \leftarrow 1$
	\STATE $r_j \leftarrow 0$
	\STATE $c \leftarrow c+1$ 
	\ENDIF
	\ENDFOR
	\IF{$c\geq 2 \epsilon d$}
	\STATE $\vt \leftarrow \hat{\mA}_{\cS}^T \vw$
    \STATE $\kappa \leftarrow \mathrm{argmax}_{k \in \cS}\{ t_k \}$
    \IF{$t_{\kappa} \geq 2\epsilon d$}
    \STATE $\hat{x}_{\kappa} \leftarrow r_{i}$
    \STATE UPDATED$\leftarrow$TRUE
    \STATE RUN$\leftarrow$TRUE
	\ENDIF
	\ENDIF
	\ENDIF
	\ENDFOR
	\ENDIF
	\ENDWHILE
	\STATE Return $\hat{\vx}$, UPDATED
	\end{algorithmic}
\end{algorithm}
\noindent While the nonzeros of $\mA$ can be recovered from potentially many of the columns of $\mR$, each nonzero of $\mX$ can be recovered only from one column. This makes the recovery up to row permutation of $\mX$ more challenging than than the recovery up to column permutation of $\mA$. As highlighted in the discussion of Algorithm \ref{alg:extract} in Section \ref{subsec:singletons_partial_supports}, due to Lemma \ref{corollary_sort_ps} partial supports which are at most $(1-2 \epsilon)d$ are not clustered or matched with a column of $\hat{\mA}$. This is because the updates of prior iterations only guarantee that each nonzero column of the reconstruction has more than $(1-2 \epsilon)d$ nonzeros. However, if we know that $\hat{\va}_l$ is complete, i.e. there exists a $h \in [n]$ such that $\hat{\va}_l = \va_h$, then for any partial support $\vw$ it suffices only that $\hat{\va}_l^T \vw \geq 2 \epsilon d$ to conclude that $\vw$ originates from $\hat{\va}_l$. Therefore it is possible to match smaller partial supports, i.e., those of size around $d/3$ rather than $2d/3$, to columns of the reconstruction which are complete. To this end, while Algorithm \ref{alg:extract} focuses on clustering partial supports with cardinality larger than $(1-2\epsilon)d$ and matching them to a nonzero column of $\hat{\mA}$, the decode step, detailed in Algorithm \ref{alg:decode}, focuses on matching any partial support whose cardinality is at least $2 \epsilon d$ to a complete column of $\hat{\mA}$. More generally, the decode step should be considered a placeholder for one of the many algorithms from the combinatorial compressed sensing literature, aiming to at least partially decode $\hat{\mX}$ given $\mR$ and the complete columns of the encoder reconstruction $\hat{\mA}_{\cS}$. In particular, \cite[Algorithms 1 \& 2]{mendoza-smith2017a} are particularly appropriate for deployment in this context as they are based on a model in which the columns of $\mX$ also satisfy the dissociated property, and, as can be observed in \cite{mendoza-smith2017a}, are able to recover $\mX$ from $\mY$ given $\mA$ even when $k$ is large, i.e., $k\approx m/3$. For clarity, and to allow for the reuse of certain theoretical tools, in our analysis we will consider the decoder subroutine detailed in Algorithm \ref{alg:decode}, which is inspired by Expander Recovery (ER) \cite{ER_paper}.

Algorithm \ref{alg:decode} attempts to decode a column vector $\vy = \mA \vx$, taking as inputs $\vy$ and its dimension $m$, the expansion parameter $\epsilon$ and column sparsity $d$ of $\mA$, and the corresponding reconstructions $\hat{\mA}$ and $\hat{\vx}$ of $\mA$ and $\vx$. The algorithm returns an updated reconstruction of the sparse code $\hat{\vx}$ and a Boolean variable UPDATED, which indicates whether or not an update to $\hat{\vx}$ has actually occurred. Algorithm \ref{alg:decode} operates in a manner similar to Algorithm \ref{alg:extract}, the key difference being that the partial supports are compared with and matched to only the complete columns of $\hat{\mA}$, which are identified on line 1. As in the case of Algorithm \ref{alg:extract}, the role of the nested for loops starting on lines 8 and 13 is to identify and extract singleton values and partial supports. However, as the columns of $\hat{\mA}_{\cS}$ are complete, then the required overlap in support as per line 20 is only $2 \epsilon d$: this follows from the fact that any partial support $\vw$, irrespective of its size, originates from a column $\va_l$ of $\mA$ if $\vw^T \va_l \geq 2 \epsilon d$. The Boolean variable UPDATED is used to track whether or not at any point an update to $\hat{\vx}$ has occurred. The Boolean variable RUN is used to determine whether the while loop running from lines 4 to 32 should iterate: this occurs as long as a new nonzero is added to $\hat{\vx}$ on line 24. As a result, at each iterate prior to the termination iterate, then on line 6 the contribution of at least one column of $\hat{\mA}_{\cS}$ is removed from $\vy$. This update to the residual potentially reveals new singleton values, which can be matched with another column of $\hat{\mA}_{\cS}$, allowing for further updates to $\hat{\vx}$. If $\hat{\vx}$ is not updated during the previous iterate, then the residual at the current and previous iterates will be the same. Therefore no progress can be made and so the algorithm terminates.

Without leveraging sparsity, the outer for loop starting on line 8 iterates through $\cO(m)$ elements. Each iteration involves an inner for loop, starting on line 13, which also iterates through $\cO(m)$ elements, and a matrix vector product on line 21. The cost of each of these is $\cO(m)$ and $\cO(mn)$ respectively, therefore overall the cost per iteration of the while loop is $\cO(m^2n)$. As $\hat{\vx}$ has $\cO(k)$ nonzeros then the while loop starting on line 4 iterates $\cO(k)$ times. Therefore, without taking advantage of sparsity, the total computational cost of Algorithm \ref{alg:decode} is $\cO(km^2n)$. However, if $\vr$, $\hat{\mA}_{\cS}$ and $\hat{\vx}$ are stored using a sparse format, then the per iteration cost of the while loop is $\cO(kn)$, resulting in a reduced total cost of $\cO(k^2n)$. 

% % ~~~~~~~~~~~~~~~~~~~~~~~~~~~~~~~~~~~~~~~~~~~~~~~~~~~~~~
\subsection{Detailed summary of the D-EBF algorithm}\label{subsec:EBF_detail}
Algorithm \ref{alg:DEBF_detail} provides a detailed version of Algorithm \ref{alg:high_level_l0-EBF}. The algorithm takes as inputs the product matrix $\mY \in \reals^{m \times N}$, the dimensions $m,n$ and $N$ of $\mA$ and $\mX$, and the expansion parameter $\epsilon$ and column sparsity $d$ of $\mA$.  The algorithm returns the estimates $\hat{\mA}$ and $\hat{\mX}$ of $\mA$ and $\mX$ respectively, aiming to recover them up to permutation. We emphasise again that the D-EBF algorithm implicitly assumes that $\mA \in \cE^{m \times n}_{k, \epsilon, d}$ with $\epsilon \leq 1/6$.

The while loop starting on line 6 runs until either $\hat{\mA}\hat{\mX} = \mY$, or no changes to the residual $\mR$ of $\mY$ occur. The for loop starting on line 11 extracts partial supports and singleton values from each column of $\mR$ as per Algorithm \ref{alg:extract}, storing the unmatched singleton values and partial supports in $\vq$ and $\mW$ respectively. On line 15 the vector $\vh$ is used to keep track of each singleton value's column of origin, while on line 16 $p$ is used to count the number of singleton values extracted across all columns of $\mR$ at the current iteration. So long as at least one unmatched singleton value is found, then on line 19 the associated unmatched partial supports are clustered and used to construct new nonzero columns in $\hat{\mA}$. Note that the nonzero columns of $\hat{\mA}$ are added from left to right, with the variable $\eta$ tracking the zero column of $\hat{\mA}$ with the smallest index. To further grow the support of $\hat{\mX}$, the for loop starting on line 22 applies Algorithm \ref{alg:decode} to each of the columns of $\mY$. Alternatively, and as discussed in Section \ref{subsec:l0_decode}, some other decoding algorithm from the combinatorial compressed sensing literature could be used instead. Finally, given the updates to $\hat{\mA}$ and $\hat{\mX}$, the residual is then recomputed on line 25 in preparation for singleton value and partial support extraction and clustering at the next iteration.

\begin{algorithm}
	\caption{D-EBF$(\mY,m,n,N,\epsilon, d)$} \label{alg:DEBF_detail}
	\begin{algorithmic}[1]
	\STATE $\hat{\mA}\leftarrow \text{zeros}(m,n)$
	\STATE $\hat{\mX}\leftarrow \text{zeros}(n,N)$
	\STATE $\mR \leftarrow \mY$
	\STATE UPDATED $\leftarrow$ TRUE
	\STATE $\eta \leftarrow 1$
	\WHILE{$\mR \neq \textbf{0}_{m \times N}$ \textbf{and} UPDATED$=$TRUE}
	\STATE $\vq \leftarrow [\;]$
	\STATE $\mW \leftarrow [\;]$
	\STATE $\vh \leftarrow [\;]$
	\STATE $p \leftarrow 0$
	\FOR{$i=1:N$}
	\STATE $\vu, \mV, c,\hat{\mA}, \hat{\vx}_i,\mathrm{UPDATED} \leftarrow$ EXTRACT\&MATCH$(\vr_i,m,\epsilon,d,\hat{\mA},\hat{\vx}_i)$
	\STATE $\vq \leftarrow [\vq^T; \vu^T]$
	\STATE $\mW \leftarrow [\mW; \mV]$
	\STATE $\vh \leftarrow [\vh, i\times \text{zeros}(c)]$
	\STATE $p \leftarrow p+c$
	\ENDFOR
	\IF{$p>0$}
	\STATE $\hat{\mA},\hat{\mX},\eta \leftarrow$ CLUSTER\&ADD$(\vq, \mW, p, \eta, \vh, \epsilon, d,\hat{\mA}, \hat{\mX})$
	\STATE UPDATED$\leftarrow$TRUE
	\ENDIF
	\FOR{i=1:N}
	\STATE $\hat{\vx}_i, \mathrm{UPDATED} \leftarrow \text{DECODE}(\vy_i,m,N,\epsilon, d,\hat{\mA},\hat{\vx}_i)$
	\ENDFOR
	\STATE $\mR \leftarrow \mY - \hat{\mA} \hat{\mX}$
	\ENDWHILE
	\STATE Return $\hat{\mA},\hat{\mX}$
	\end{algorithmic}
\end{algorithm}

In terms of computational cost the key contributors are lines 12, 19, 23 and 25. We now present the computational cost first without and then with taking advantage of sparsity. As per the discussion in Section \ref{subsec:singletons_partial_supports}, the for loop running from line 11 to 17 has a cost of $\cO(m^2n^2N)$ or $\cO(k^2nN)$ respectively. Observe also that by transferring the matching process, lines 17-24 of Algorithm \ref{alg:extract}, to Algorithm \ref{alg:cluster}, then singleton values and partial supports can be extracted from each column of $\mR$ in parallel, using $N$ processors with a computational cost of $\cO(k^2n)$ per processor. Indeed, this step is trivially parallelizable across the columns of $\mR$. Likewise, from the discussion in Section \ref{subsec:singletons_partial_supports}, line 19 has a cost of $\cO(k^2mN^2)$ or $\cO(k^2N^2)$ respectively. In comparison this step cannot be parallelized. Turning our attention to the decoding step of the algorithm, then, as per the discussion in Section \ref{subsec:l0_decode}, the for loop running from line 22 to 24 has a cost of $\cO(km^2nN)$ or $\cO(k^2nN)$ respectively. This step is again trivially parallelizable across the columns of $\mY$. Finally, the matrix product on line 25 has a cost of $\cO(mnN)$ or $\cO(knN)$ respectively. Therefore, the while loop running from lines 6 to 26 of Algorithm \ref{alg:DEBF_detail} has a per iteration cost of $\cO(m(mn^2+k^2N)N)$ if sparsity is not leveraged, or $\cO(k^2(n+N)N)$ if it is. Under certain assumptions we are also able to provide accuracy guarantees for Algorithm \ref{alg:DEBF_detail}, as detailed in Lemma \ref{lemma:accuracy_DEBF}.

\begin{lemma}[\textbf{Accuracy of D-EBF}] \label{lemma:accuracy_DEBF}
Let $\mY = \mA \mX$, where $\mA \in \cE^{m \times n}_{k, \epsilon, d}$ with $\epsilon \leq 1/6$ and $\mX \in \cX^{n\times N}_k$. Consider Algorithm \ref{alg:DEBF_detail}, D-EBF: in regard to the reconstructions $\hat{\mA}$ and $\hat{\mX}$ computed at any point during the run-time of D-EBF, we will say $\cQ$ is true iff the following hold.
\begin{enumerate}[label=(\alph*)]
    \item A column $\hat{\va}_l$ of $\hat{\mA}$ is nonzero iff the $l$th row of $\hat{\mX}$ is nonzero.
    \item For all $l \in [n]$ if $\hat{\va}_l$ is nonzero then $\supp(\hat{\va}_l) > (1-2 \epsilon)d$.
    \item There exists a permutation matrix $\mP \in \cP^{n \times n}$ such that $\supp(\hat{\mA}\mP^T) \subseteq \supp(\mA)$, $\supp(\mP\hat{\mX}) \subseteq \supp(\mX)$ and $\hat{x}_{P(l),h} = x_{P(l), h}$ for all $h \in [N]$ and $l \in supp(\hat{\vx}_l)$, where $P:[n] \rightarrow [n]$ denotes the row permutation caused by pre-multiplication with $\mP$.
\end{enumerate}
Suppose that D-EBF exits after the completion of iteration $t_f\in \naturals$ of the while loop, then the following statements are true.
\begin{enumerate}
    \item At any point during the run time of D-EBF $\cQ$ is true. Therefore the reconstructions $\hat{\mA}$ and $\hat{\mX}$ returned by D-EBF also satisfy the conditions for $\cQ$ to be true.
    \item For any iteration of the while loop $t <t_f$, let $\hat{\mA}^{(1)}$ and $\hat{\mX}^{(1)}$ denote the reconstructions at the start of the while loop, line 6, and $\hat{\mA}^{(2)}$ and $\hat{\mX}^{(2)}$ denote the reconstructions at the end of the while loop, line 26. Then either $\supp(\hat{\mA}^{(1)})\subset \supp(\hat{\mA}^{(2)})$ and $\supp(\hat{\mX}^{(1)})\subseteq \supp(\hat{\mX}^{(2)})$, or $\supp(\hat{\mA}^{(1)})\subseteq \supp(\hat{\mA}^{(2)})$ and $\supp(\hat{\mX}^{(1)})\subset \supp(\hat{\mX}^{(2)})$.
    \item Consider the reconstructions $\hat{\mA}$ and $\hat{\mX}$ returned by D-EBF and the associated residual $\mR = \mY - \hat{\mA} \hat{\mX}:$ there exists a permutation matrix $\mP \in \cP^{n \times n}$ such that $\hat{\mA}\mP^T = \mA$ and $\mP \hat{\mX} = \mX$ iff $\mR =\textbf{0}_{m \times N}$. As a result, $\mR = \textbf{0}_{m \times N}$ is a sufficient condition to ensure that $\mY$ has a unique, up to permutation, factorisation of the form $\mY = \mA \mX$, where $\mA\in \cE^{m \times n}_{k, \epsilon, d}$ with $\epsilon \leq 1/6$ and $\mX \in \cX^{n\times N}_k$.
\end{enumerate}
\end{lemma}

A proof of Lemma \ref{lemma:accuracy_DEBF} is provided in Appendix \ref{appendix:accuracy}. In short, the implication of Lemma \ref{lemma:accuracy_DEBF} is that as long as $\mA \in \cE^{m \times n}_{k, \epsilon, d}$ with $\epsilon \leq 1/6$, then the reconstructions returned by D-EBF are accurate up to permutation, and if they are in addition complete up to permutation then the factorisation of this form is unique. Without further assumptions on $\mX$ it is not possible to derive guarantees in respect to fully recovering the matrix factors up to permutation. For instance, if a row of $\mX$ is a zero row vector then the corresponding column of $\mA$ contributes to no columns of $\mY$ and therefore cannot be hoped to be recovered. In Section \ref{sec:theory} we will derive such guarantees by assuming that $Y$ is drawn from the PSB model, defined in Definition \ref{def_data_model}.

\section{Theoretical guarantees for D-EBF under the PSB model}\label{sec:theory}
In this section we derive Theorem \ref{theorem:main}. Analysing D-EBF directly is challenging, therefore our approach is instead to study a simpler, surrogate algorithm, which we use to lower bound the performance of D-EBF. To this end we introduce the Naive Decoder-Expander Based Factorisation Algorithm (ND-EBF). Although highly suboptimal from a computational perspective, and therefore clearly not recommended for deployment in practice, this algorithm allows us to analyse the parameter regime in which D-EBF is successful with high probability. The structure of this section is as follows: in Section \ref{subsec:NDEBF} we present and describe the ND-EBF algorithm, provide accuracy guarantees analogous to Lemma \ref{lemma:accuracy_DEBF} and connect ND-EBF with D-EBF. Then in Section \ref{subsec:proof} we use these results to prove Theorem \ref{theorem:main}.

% % ~~~~~~~~~~~~~~~~~~~~~~~~~~~~~~~~~~~~~~~~~~~~~~~~~~~~~~
\subsection{Naive Decoder-Expander Based Factorisation (ND-EBF)} \label{subsec:NDEBF}
In order to define the ND-EBF algorithm we must introduce the subroutine detailed in Algorithm \ref{alg:MAXCLUSTER}. This subroutine plays a role analogous to that of Algorithm \ref{alg:cluster} in D-EBF. The key difference between these two subroutines is that while Algorithm \ref{alg:cluster} attempts to cluster all partial supports in $\mW$, and then use each of these clusters to reconstruct a column of $\mA$, Algorithm \ref{alg:MAXCLUSTER} attempts to identify only the largest cluster of partial supports in $\mW$, and then use this one cluster to reconstruct a single column of $\mA$. Algorithm \ref{alg:MAXCLUSTER} takes as inputs the following: the unmatched singleton values, stored in the entries of the vector $\vq \in \reals^p$, the unmatched partial supports stored in the columns of the matrix $\mW \in \{0,1\}^{m \times p}$ and ordered so that $\vw_i$ is the partial support associated with the singleton $q_i$, the number of unmatched singleton values $p$, the index $\eta \defeq \min \{l \in [n]: \hat{\va}_l = \textbf{0}_m\}$, the vector $\vh \in [N]^p$ which stores the indices of the columns of $\mR$ from which each singleton value was extracted, i.e., $h_i$ is the column index of $\mR$ from which $q_i$ was extracted, the expansion parameter $\epsilon$ and column sparsity $d$ of $\mA$, and the current reconstructions $\hat{\mA}$ and $\hat{\mX}$. The algorithm returns updated reconstructions $\hat{\mA}$ and $\hat{\mX}$: as can be observed at most one new nonzero column of $\hat{\mA}$ and row of $\hat{\mX}$ are updated.

\begin{algorithm}
	\caption{MAXCLUSTER\&ADD$(\vq, \mW, p, \eta, \vh, \epsilon, d, \hat{\mA}, \hat{\mX})$}
	\label{alg:MAXCLUSTER} 
	\begin{algorithmic}[1]
	\STATE $\vb \leftarrow \mathrm{zeros}(p)$
	\STATE $\mG \leftarrow \mW^T\mW$
	\STATE $\cC^* \leftarrow \phi$
	\FOR{$i=1:p$}
	\IF{$b_i = 0$}
	\STATE $\cC \leftarrow \{i\}$
	\FOR{$j=(i+1):p$}
	\IF{$g_{i,j} \geq 2\epsilon d$ \textbf{and} $b_j = 0$}
    \STATE $\cC \leftarrow \cC \cup \{i\}$
	\ENDIF
	\ENDFOR
	\IF{$|\cC| > |\cC^*|$}
	\STATE $\cC^* \leftarrow \cC$
	\ENDIF
	\ENDIF
	\ENDFOR
	\FOR{$i \in \cC^*$}
	\STATE $\hat{\va}_{\eta} \leftarrow \sigma(\hat{\va}_{\eta} + \vw_i)$
	\STATE $\hat{x}_{\eta,h_i} \leftarrow q_i$
	\ENDFOR
	\STATE Return $\hat{\mA},\hat{\mX}$
	\end{algorithmic}
\end{algorithm}

\noindent The $p$ dimensional binary vector $\vb$, initialised as the zero vector on line 1, is used to indicate whether or not a partial support has been assigned to a column of $\hat{\mA}$, one indicating true and zero false. The entries of the gram matrix $\mG$, calculated on line 2, are used to cluster the unmatched partial supports. On line 3 the index set $\cC^*\subseteq [p]$ is initialised as empty and is used to denote the cluster of partial supports in $\mW$ with the largest cardinality identified so far. The outer for loop, starting on line 4, first checks, as per the if statement on line 5, if the partial support $\vw_i$ currently under consideration has been used to update a column of the reconstruction already. If this is not the case, then on line 6 the index set $\cC \subseteq [p]$, which is used to denote the set of partial supports in $\mW$ which originate from the same column of $\mA$ as $\vw_i$, is initialised. The inner for loop, lines 7-11, then identifies and finds the indices of any such partial supports. On lines 12 and 13, if $|\cC|> |\cC^*|$ then a larger cluster than $\cC^*$ has been identified and so $\cC^*$ is updated accordingly. Finally, after the completion of the outer for loop then, on lines 18 and 19, the partial supports and singleton values belonging to the largest cluster $\cC^*$ are used to update just the $\eta$th column of $\hat{\mA}$ and row of $\hat{\mX}$ respectively.

\begin{algorithm}
	\caption{ND-EBF$(\mY,m,n,N,d,\epsilon)$} \label{alg:NDEBF}
	\begin{algorithmic}[1]
	\STATE $\hat{\mA}\leftarrow \text{zeros}(m,n)$
	\STATE $\hat{\mX}\leftarrow \text{zeros}(n,N)$
	\STATE $\mR \leftarrow \mY$
	\STATE $\eta \leftarrow 1$
	\STATE RUN$\leftarrow$TRUE
	\WHILE{$\eta \leq n$ \textbf{and} RUN$=$TRUE}
	\STATE $\vq \leftarrow []$
	\STATE $\mW \leftarrow []$
	\STATE $\vh \leftarrow []$
	\STATE $p \leftarrow 0$
	\FOR{$i=1:N$}
	\STATE $\vu, \mV, c,\_,\hat{\vx}_i,\_ \leftarrow$ EXTRACT\&MATCH$(\vr_i,m,\epsilon,d,\hat{\mA},\hat{\vx}_i)$
	\STATE $\vq \leftarrow [\vq^T; \vu^T]$
	\STATE $\mW \leftarrow [\mW; \mV]$
	\STATE $\vh \leftarrow [\vh, i\times \text{zeros}(c)]$
	\STATE $p \leftarrow p+c$
	\ENDFOR
	\IF{$p>0$}
	\STATE $\hat{\mA},\hat{\mX} \leftarrow \mathrm{MAXCLUSTER\&ADD}(\vq, \mW, p, \eta, \vh,  \epsilon, d, \hat{\mA}, \hat{\mX})$
	\IF{$|\supp(\hat{\va}_{\eta})|=d$}
	\FOR{$i=1:N$}
	\STATE $\hat{\vx}_i,\_ \leftarrow \text{DECODE}(\vy_i,m,N,\epsilon, d,\hat{\mA}_{\cS}, \hat{\vx_i})$
	\ENDFOR
	\STATE $\mR \leftarrow \mY - \hat{\mA} \hat{\mX}$
	\STATE $\eta \leftarrow \eta +1$
	\ELSE
	\STATE RUN$\leftarrow$FALSE
	\ENDIF
	\ELSE
	\STATE RUN$\leftarrow$FALSE
	\ENDIF
	\ENDWHILE
	\STATE Return $\hat{\mA}$, $\hat{\mX}$
	\end{algorithmic}
\end{algorithm}

The ND-EBF algorithm is presented and defined in Algorithm \ref{alg:DEBF_detail}. As referenced to earlier, the way this algorithm operates is very similar to that of D-EBF, the key difference being that while D-EBF maximally utilises all partial supports and singleton values extracted at each iteration, ND-EBF seeks only to utilise those associated with the largest cluster. The extraction of singleton values and partial supports, lines 11-17, is done in exactly the same fashion as in D-EBF. If at least one singleton value and partial support are extracted, i.e., $p>0$, then on line 19  Algorithm \ref{alg:MAXCLUSTER} is deployed in order to recover a single column of $\mA$ from the largest cluster of partial supports in $\mW$. If this reconstruction is not complete, meaning the condition on line 20 is not satisfied, then the algorithm terminates. In order to recover further entries of $\mX$ and grow the support of $\hat{\mX}$, then on line 22 Algorithm \ref{alg:decode}, or some other combinatorial decoding algorithm, is deployed. Given the updates to both $\hat{\mA}$ and $\hat{\mX}$, the residual is then recomputed on line 24 in preparation for singleton value and partial support extraction and clustering at the next iteration. Analogous to Lemma \ref{lemma:accuracy_DEBF} we provide the following accuracy guarantees for ND-EBF.

\begin{lemma}[\textbf{Accuracy of ND-EBF}] \label{lemma:accuracy_ND-EBF}
Let $\mY = \mA \mX$, where $\mA \in \cE^{m \times n}_{k, \epsilon, d}$ with $\epsilon \leq 1/6$ and $\mX \in \cX^{n\times N}_k$. Consider Algorithm \ref{alg:NDEBF}, ND-EBF: in regard to the reconstructions $\hat{\mA}$ and $\hat{\mX}$ computed at any point during the run-time of ND-EBF, we will say, for $\eta \in [n]$, that $\cQ(\eta)$ is true iff the following hold.
\begin{enumerate}[label=(\alph*)]
    \item The column vector $\hat{\va}_l$ and row vector $\tilde{\hat{\vx}}_l$ are nonzero iff $l < \eta$.
    \item If $l < \eta$ then $|\supp(\hat{\va}_l)| = d$.
    \item There exists a permutation matrix $\mP \in \cP^{n \times n}$ such that $\supp(\hat{\mA}\mP^T) \subseteq \supp(\mA)$, $\supp(\mP\hat{\mX}) \subseteq \supp(\mX)$ and $\hat{x}_{P(l),i} = x_{P(l), i}$ for all $i \in [N]$ and $l \in supp(\hat{\vx}_l)$, where $P:[n] \rightarrow [n]$ denotes the row permutation caused by pre-multiplication with $\mP$.
\end{enumerate}
 Suppose that ND-EBF exits after the completion of iteration $\eta_f \in [n]$ of the while loop, then the following statements are true.
\begin{enumerate}
    \item For all $\eta \in [\eta_f]$, then at the start of the $\eta$th iteration of the while loop $\cQ(\eta)$ is true.
    \item Consider the reconstructions $\hat{\mA}$ and $\hat{\mX}$ returned by ND-EBF: $\hat{\mA}$ and $\hat{\mX}$ always satisfy (c) and there exists a permutation matrix $\mP \in \cP^{n \times n}$ such that $\hat{\mA}\mP^T = \mA$ and $\mP \hat{\mX} = \mX$ iff $\eta_f = n$ and $\cQ(n+1)$ is true. 
    \item Consider the reconstructions $\hat{\mA}$ and $\hat{\mX}$ returned by ND-EBF and the associated residual $\mR = \mY - \hat{\mA} \hat{\mX}:$ there exists a permutation matrix $\mP \in \cP^{n \times n}$ such that $\hat{\mA}\mP^T = \mA$ and $\mP \hat{\mX} = \mX$ iff $\mR =\textbf{0}_{m \times N}$. As a result, $\mR = \textbf{0}_{m \times N}$ is a sufficient condition to ensure that $\mY$ has a unique, up to permutation, factorisation of the form $\mY = \mA \mX$, where $\mA\in \cE^{m \times n}_{k, \epsilon, d}$ with $\epsilon \leq 1/6$ and $\mX \in \cX^{n\times N}_k$.
\end{enumerate}
\end{lemma}

A proof of Lemma \ref{lemma:accuracy_ND-EBF} is provided in Appendix \ref{appendix:accuracy}. In short, the implication of Lemma \ref{lemma:accuracy_ND-EBF} is that as long as $\mA \in \cE^{m \times n}_{k, \epsilon, d}$ with $\epsilon \leq 1/6$, then the reconstructions returned by ND-EBF are accurate up to permutation, and, if they are in addition complete up to permutation, then the factorisation of this form is unique. The following corollary will allow us to focus our analysis on the recovery of $\mA$.

\begin{corollary}[\textbf{Recovery of $\mA$ up to permutation is sufficient to recover both factors up to permutation}]\label{corollary:recovery_A_sufficient}
Let $\mY = \mA \mX$, where $\mA \in \cE^{m \times n}_{k, \epsilon, d}$ with $\epsilon \leq 1/6$ and $\mX \in \cX^{n\times N}_k$. Suppose that $\hat{\mA}$ and $\hat{\mX}$ are the reconstructions of $\mA$ and $\mX$ returned by either ND-EBF or D-EBF. If there there exists a $\mP \in \cP^{n \times n}$ such that $\hat{\mA}\mP^T = \mA$ then $\mP \hat{\mX} = \mX$.
\end{corollary}

A proof of Corollary \ref{corollary:recovery_A_sufficient} is provided in Appendix \ref{appendix:supporting_lemmas}.
To complete this section we provide Lemma \ref{lemma:NDEBF_lower_bounds_DEBF}: the key takeaway of this lemma is that if ND-EBF successfully recovers the matrix factors of $\mY$ up to permutation, then so to will D-EBF.

\begin{lemma}[\textbf{If ND-EBF successfully computes the matrix factors up to permutation then so will D-EBF}] \label{lemma:NDEBF_lower_bounds_DEBF}
Let $\mY = \mA \mX$ where $\mA \in \cE^{m \times n}_{k, \epsilon, d}$ with $\epsilon \leq 1/6$ and $\mX \in \cX^{n\times N}_k$. Let $\hat{\mA}^{(1)}$ and $\hat{\mX}^{(1)}$ denote the reconstructions returned by ND-EBF and $\hat{\mA}^{(2)}$ and $\hat{\mX}^{(2)}$ the reconstructions returned by D-EBF. If there exists a $\mP_1 \in \cP^{n \times n}$ such that $\hat{\mA}^{(1)}\mP_1^T = \mA$ and $\mP_1 \hat{\mX}^{(1)} = \mX$, then there exists a $\mP_2 \in \cP^{n \times n}$ such that $\hat{\mA}^{(2)}\mP_2^T = \mA$ and $\mP_2 \hat{\mX}^{(2)} = \mX$.
\end{lemma}

A proof of Lemma \ref{lemma:NDEBF_lower_bounds_DEBF} is provided in Appendix \ref{appendix:supporting_lemmas}.

% % ~~~~~~~~~~~~~~~~~~~~~~~~~~~~~~~~~~~~~~~~~~~~~~~~~~~~~~
\subsection{Proof of Theorem \ref{theorem:main}}\label{subsec:proof}
Together, Corollary \ref{corollary:recovery_A_sufficient} and Lemma \ref{lemma:NDEBF_lower_bounds_DEBF} allow us to lower bound the performance D-EBF by analysing the conditions under which ND-EBF recovers a column of $\mA$ at each iteration of the while loop, lines 6-31 of Algorithm \ref{alg:NDEBF}. To this end, we first lower bound $N$ in order to lower bound with high probability the number of nonzeros per row of $X$.

\begin{lemma}[\textbf{Nonzeros per row in $X$}]\label{lemma:N_lb}
For some $\beta \in \naturals$ and $\mu>1$, if $N \geq \beta \left(\mu  \frac{n}{k} \ln(n)+1 \right)$ then the probability that the random matrix $X$, as defined in the PSB model Definition \ref{def_data_model}, has at least $\beta$ nonzeros per row is more than $\left( 1 - n^{-(\mu -1)} \right)^{\beta}$.
\end{lemma}

A proof of Lemma \ref{lemma:N_lb} is given in Appendix \ref{appendix:supporting_lemmas}. Assuming a lower bound on the number of nonzeros per row of $X$ and that $Y$ is drawn from the PSB model, Lemma \ref{lemma:recon_from_L} provides a lower bound on the probability that a column of $A$ is recovered at iteration $\eta\in[n]$ of the while loop, lines 6-31, of ND-EBF. 

\begin{lemma}[\textbf{Probability that ND-EBF recovers a column at iteration $\eta$}]\label{lemma:recon_from_L}
Let $Y = AX$ as per the PSB model, detailed in Definition \ref{def_data_model}. In addition, assume that $A \in \cE_{k, \epsilon, d}^{m \times n}$ with $\epsilon \leq 1/6$, and that each row of $X$ has at least $\beta(n) = (1+2\epsilon d) L(n)$ nonzeros, where $L:\naturals \rightarrow \naturals$. Suppose that ND-EBF, Algorithm \ref{alg:NDEBF}, is deployed to try and compute the factors of $Y$ and that the algorithm reaches iteration $\eta \in [n]$ of the while loop. Then there is a unique column $A_{\ell(\eta)}$ of $A$, satisfying $A_{\ell(\eta)} \neq \hat{A}_{l}$ for all $l \in [\eta-1]$, such that
\[
\begin{aligned}
\prob(\hat{A}_{\eta} = A_{\ell(\eta)}) > 1 - d e^{- \tau(n) L(n)},
\end{aligned}
\]
where $\tau(n) \defeq - \ln \left( 1 - \left(1 - \frac{b}{\alpha_1 n} \right)^{\alpha_1 n}\right)$ is  $\cO(1)$.
\end{lemma}

A proof of Lemma \ref{lemma:recon_from_L} is given in Appendix \ref{appendix:supporting_lemmas}. With Lemmas \ref{lemma:N_lb} and \ref{lemma:recon_from_L} in place we are ready to prove Theorem \ref{theorem:main}. Statements 1 and 2 of Theorem \ref{theorem:main} are immediate consequences of Lemma $\ref{lemma:accuracy_DEBF}$, therefore all that is left to prove is Statement 3. To provide a brief recap, our objective is to recover up to permutation the random factor matrices $A$ and $X$, as defined in the PSB model in Definition \ref{def_data_model}, from the random matrix $Y = AX$. Our  strategy at a high level is as follows: using Lemma \ref{lemma:NDEBF_lower_bounds_DEBF} and Corollary \ref{corollary:recovery_A_sufficient} we lower bound the probability that D-EBF successfully factorises $Y$ up to permutation by lower bounding the probability that ND-EBF recovers $A$ up to permutation. ND-EBF recovers $A$ up to permutation iff at each iteration of the while loop, lines 6-31 of Algorithm \ref{alg:NDEBF}, a new column of $A$ is recovered. We lower bound the probability of this event in turn using Lemma \ref{lemma:recon_from_L}. For the proof Theorem \ref{theorem:main} we adopt the following notation.
\begin{itemize}
	\item $\Lambda_{\mathrm{D-EBF}}^*$ and $\Lambda_{\mathrm{ND-EBF}}^*$ are the events that D-EBF and ND-EBF recover $A$ and $X$ up to permutation respectively, i.e., there exists a random permutation $P \in \cP^{n \times n}$ such that $\hat{A}P^T = A$ and $P\hat{X} = X$.
	\item $\Lambda_0 \defeq \{A \in \cE_{k, \epsilon, d}^{m\times n} \}\cap \{\epsilon \leq 1/6\}$ is the event that the random matrix $A$ is the adjacency matrix of a $(k, \epsilon, d)$-expander graph with expansion parameter $\epsilon \leq 1/6$.
	\item For $\eta \in [n]$ let $\Lambda_{\eta}$ denote the event that at the end of $\eta$th iteration of the while loop of ND-EBF, there exists a unique column $A_{\ell(\eta)}$ of $A$ satisfying $A_{\ell(\eta)}=\hat{A}_{\eta}$ and $A_{\ell(\eta)}\neq \hat{A}_{l}$ for all $l \in [n]\backslash \{\eta\}$.
	\item $\Lambda_{n+1}$ denotes the event that each row of $X$ has at least $\beta(n) \defeq (1+2\epsilon)d L(n)$
	nonzeros per row where $L(n) \defeq \ceil{\frac{\gamma \ln(n)}{\tau(n)}}$ and $\gamma >1$ is a constant.
\end{itemize}	
\begin{proof}
	Suppose that ND-EBF is deployed instead of D-EBF to recover the matrix factors of $Y$. From Corollary \ref{corollary:recovery_A_sufficient}, for ND-EBF to recover both factors up to permutation it suffices to recover $A$ up to permutation. Clearly then
	\[
	\begin{aligned}
	\prob \left(\Lambda_{ND-EBF}^* \; | \; \Lambda_0 \right)  &=  \prob \left( \bigcap_{h=1}^{n} \Lambda_h \; | \; \Lambda_0 \right).
	\end{aligned}
	\]
	We now apply Bayes' Theorem and condition on $\Lambda_{n+1}$,
	 \[
	 \begin{aligned}
	 \prob \left(\Lambda_{ND-EBF}^*\; | \; \Lambda_0 \right)  &=  \prob \left( \bigcap_{\eta=1}^{n} \Lambda_{\eta} \; | \; \Lambda_0 \right)\\
	 &= \frac{ \prob \left( \bigcap_{\eta=1}^{n} \Lambda_{\eta} , \Lambda_0 \right)}{\prob \left( \Lambda_0 \right)}\\
	 & \geq \frac{\prob \left( \bigcap_{\eta=1}^{n} \Lambda_{\eta} , \Lambda_0  \; | \; \Lambda_{n+1} \right) \prob \left( \Lambda_{n+1} \right)}{\prob \left( \Lambda_0 \right)}\\
	 & = \frac{\prob \left( \bigcap_{\eta=1}^{n} \Lambda_{\eta}  \; | \; \Lambda_0, \Lambda_{n+1} \right)\prob \left( \Lambda_0 \; | \;\Lambda_{n+1} \right) \prob \left( \Lambda_{n+1} \right)}{\prob \left( \Lambda_0 \right)}\\
	 & = \prob \left( \bigcap_{\eta=1}^{n} \Lambda_{\eta} \; | \;  \Lambda_0, \Lambda_{n+1} \right) \prob \left( \Lambda_{n+1} \right)\\
	 & = \prob \left( \Lambda_{n+1} \right) \prod_{\eta=1}^n \prob \left( \Lambda_{\eta} \; | \; \bigcap_{l=0}^{\eta-1} \Lambda_l, \Lambda_{n+1} \right).
	 \end{aligned}
	 \]
In the above, line 2 follows as a result of Bayes Theorem, line 3 is an application of the law of total probability and line 4 is derived using the probability chain rule. The equality on line 5 follows as $A$ and $X$ are drawn independently of one another, therefore, given that $\Lambda_0$ is a property of $A$ and $\Lambda_{n+1}$ a property of $X$, $\prob(\Lambda_0 \; | \; \Lambda_{n+1}) = \prob \left( \Lambda_0\right)$. Finally, line 6 is once again an application of the probability chain rule. Assume for now that $N \geq \beta(n)\left(\mu  \frac{n}{k} \ln(n)+1 \right)$, where $\mu>1$ is some constant. Then as a consequence of Lemma \ref{lemma:N_lb} it follows that
	\[
	\prob \left( \Lambda_{n+1} \right)  > \left(1 - n^{-(\mu-1)} \right)^{\beta(n)}.
	\]
Applying Lemma \ref{lemma:recon_from_L} it also holds for any $\eta \in [n]$ that
	\[
	\begin{aligned}
	\prob \left( \Lambda_{\eta} \; | \; \bigcap_{l=0}^{\eta-1} \Lambda_l, \Lambda_{n+1} \right) 
	& >1 - de^{- \tau(n) L(n)}.
	\end{aligned}
	\]
where $\tau(n)$ is $\cO(1)$. As a result
	\[
	\begin{aligned}
	\prob \left( \Lambda_{ND-EBF}^* \; | \; \Lambda_0 \right) & > \left(1 - n^{-(\mu-1)} \right)^{\beta(n)} \left( 1 -  d e ^{- \tau(n) L(n)  }\right)^n.
	\end{aligned}
	\]
Recalling that $L(n) \defeq \ceil{\frac{\gamma \ln(n)}{\tau(n)}}$, where $\gamma >1$ is an arbitrary constant, then analysing the right-hand factor it follows that
	\[
	\begin{aligned}
	\left( 1 -  d e ^{- \tau(n) L(n)  }\right)^n & \geq \left( 1 -  d e ^{\ln(1/n^{\gamma}) }\right)^n \\
	& = \left( 1 - \frac{d}{n^{\gamma}} \right)^n\\
	& = \sum_{i =0}^{\infty} \binom{n}{i} (-1)^i \left( \frac{d}{n^{\gamma}} \right)^i\\
	& = 1 - \frac{d}{n^{\gamma -1}} + \sum_{i =2}^{\infty} \binom{n}{i} (-1)^i \left( \frac{d}{n^{\gamma}} \right)^i \\
	& = 1 - \cO\left(n^{ -\gamma+1} \right).
	\end{aligned} 
	\]
Above, the inequality on the first line follows from the fact that $\ceil{\frac{\gamma \ln(n)}{\tau(n)}} \geq \frac{\gamma \ln(n)}{\tau(n)}$. The equality on the third line follows by applying the binomial series expansion. Analysing now the left-hand factor derived from Lemma \ref{lemma:N_lb}, note that as $\beta(n) = (1+2\epsilon)d \ceil{\frac{\gamma \ln(n)}{\tau(n)}}$ and $\tau(n) = \cO(1)$ then $\beta(n) = \cO(\log(n))$. As a result
	\[
	\begin{aligned}
	\left(1 - n^{-(\mu-1)} \right)^{\beta(n) } & = \sum_{i =0}^{\infty} \binom{\beta(n)}{i} (-1)^i \left( n^{-\mu +1} \right)^i\\
	& = 1 - \frac{\beta(n)}{n^{\mu -1}} + \sum_{i =2}^{\infty} \binom{\beta(n)}{i} (-1)^i \left( n^{-\mu +1} \right)^i\\
	& = 1 - \cO\left(n^{ -\mu+1} \log(n) \right).
	\end{aligned} 
	\]
Therefore we arrive at the asymptotic lower bound
	\[
	\prob \left( \Lambda_{ND-EBF}^* \; | \; \Lambda_0 \right) >  \left(1 - \cO\left(n^{ -\mu+1} \log(n) \right) \right) \left( 1 - \cO\left(n^{ -\gamma+1} \right)\right).
	\]
Given that $\epsilon \leq 1/6$ implies $(1+2 \epsilon)\leq \frac{4}{3}$, $\ceil{\frac{\gamma \ln(n)}{\tau(n)}} \leq \frac{2\gamma \ln(n)}{\tau(n)}$, and $\mu, \frac{n}{k}>1$,
	\[
	\begin{aligned}
	\beta(n) \left(\mu  \frac{n}{k} \ln(n)+1 \right) & =	(1+2\epsilon)d \ceil{\frac{\gamma
	\ln(n)}{\tau(n)}}\left(\mu  \frac{n}{k} \ln(n)+1 \right)\\
	& \leq \gamma\mu \frac{8d}{3 \tau(n)} \frac{n}{k} (\ln^2(n) + \ln(n))\\
	& \leq N
	\end{aligned}
	\]
by construction. Given that $\gamma, \mu>1$ are arbitrary constants, let $\gamma = \mu$ and define $\nu \defeq \gamma^2$, clearly it must also hold that $\nu>1$. It follows from Lemma \ref{lemma:NDEBF_lower_bounds_DEBF} that if $N \geq \nu \frac{8d}{3 \tau(n)} \frac{n}{k} (\ln^2(n) + \ln(n))$ then
\[
\begin{aligned}
	\prob \left( \Lambda_{D-EBF}^*\; | \; \Lambda_0 \right) & \geq \prob \left(\Lambda_{ND-EBF}^{*} \; | \; \Lambda_0 \right)\\
	& >   1 - \cO(n^{-\sqrt{\nu} + 1} \log(n) )
\end{aligned}
\]
as claimed. This concludes the proof of Theorem \ref{theorem:main}.
\end{proof}

\section{Experiments} \label{sec:EBF_numerics}
A key takeaway of Theorem \ref{theorem:main} is that there are parameter regimes in which, at least asymptotically, D-EBF will successfully factorise a matrix $Y$ drawn from the PSB model. In this section we investigate this matter empirically. First, in Section \ref{subsec:deploying_DEBF}, we highlight certain adjustments to D-EBF which enable it to be deployed in practice: in particular, we discuss how to circumvent the issue of not having access to the expansion parameter $\epsilon$, and also how D-EBF can be adapted to be used in an online setting. Second, in Section \ref{subsec:DEBF_numerics}, we conduct experiments demonstrating the efficacy of D-EBF in factorising even mid-sized problems, i.e., $n = 10^3$.
% ~~~~~~~~~~~~~~~~~~~~~~~~~~~~~~~~~~~~~~~~~~~~~~~~~~~~~~
\subsection{Deploying D-EBF in practice} \label{subsec:deploying_DEBF}
As highlighted in Section \ref{subsec:summary_contributions}, in practice the decoder may not have access to the expansion parameter of the encoder matrix. One reason for this, already highlighted in Section \ref{subsec:summary_contributions}, is that computing the expansion parameter is an NP-complete problem. D-EBF is designed for factorisation problems in which $\epsilon \leq 1/6$, we therefore restrict our attention to this setting and consider how the decoder might still be able to succeed in factorising $\mY$ without knowledge of $\epsilon$ in advance. To this end, suppose that $1/6$ is used by D-EBF instead of $\epsilon$: as $2 \epsilon d \leq d/3$ then using this value places a stricter requirement on the frequency with which an entry of the residual must occur in order to be accepted as a candidate singleton value. In addition, as $(1-2\epsilon)d \geq 2d/3$, then all singleton values occurring more than $(1-2 \epsilon)d$ times in a given column are still extracted. Therefore, from the perspective of extracting singleton values and partial supports, the only implication of using $1/6$ instead of $\epsilon$ is that some singleton values and their associated partial supports may be missed. 

In regard to clustering however, as $(1-4 \epsilon)d \geq d/3$,  then using $1/6$ instead of $\epsilon$ loses the guarantees afforded by Lemma \ref{corollary_sort_ps}. To be clear, consider two partial supports $\vw_1$ and $\vw_2$ satisfying $|\supp(\vw_1)|>2d/3$ and $|\supp(\vw_2)|>2d/3$: if $\vw_1^T \vw_2 \geq d/3 \geq 2 \epsilon d$ then we can conclude that these two partial supports originate from the same column, however $\vw_1^T\vw_2 < d/3$ does not guarantee that these partial supports originate from different columns. If D-EBF is deployed as described in Algorithm \ref{alg:DEBF_detail} this may result in duplicate reconstructions of the same column of $\mA$ in $\hat{\mA}$. This issue can be overcome by adding in an additional column merge subroutine at the end of each iteration. An example merge subroutine is provided in Algorithm \ref{alg:merge}. This algorithm merges duplicate columns in $\hat{\mA}$, as well as the corresponding rows of $\hat{\mX}$, by checking the inner product between columns of $\hat{\mA}$. Assuming that $\epsilon \leq 1/6$, then, if any pair of columns has an inner product larger than $d/3$ and given that $d/3 \geq 2 \epsilon d$, we may conclude that these two columns and their respective rows in $\hat{\mX}$ should be combined.

\begin{algorithm}
	\caption{MERGE$(d, \hat{\mA},\hat{\mX})$} \label{alg:merge}
	\begin{algorithmic}[1]
	\STATE $\mG \leftarrow \hat{\mA}^T \hat{\mA}$
	\FOR{$i=1:n$}
	\IF{$|\supp(\hat{\va}_i)| \geq d/3$}
	\FOR{$j=(i+1):n$}
	\IF{$g_{i,j} \geq d/3$}
	\STATE $\hat{\va}_i \leftarrow \sigma(\hat{\va}_{i} + \hat{\va}_{j})$
	\STATE $\hat{\va}_j \leftarrow \textbf{0}_m$
	\STATE $\tilde{\hat{\vx}}_i \leftarrow \tilde{\hat{\vx}}_i + \tilde{\hat{\vx}}_j$
	\STATE $\tilde{\hat{\vx}}_j \leftarrow \textbf{0}_N^T$
	\ENDIF
	\ENDFOR
	\ENDIF
	\ENDFOR
	\STATE Return $\hat{\mA},\hat{\mX}$
	\end{algorithmic}
\end{algorithm}
In summary, as long as $\mA \in \cE_{k, \epsilon, d}^{m \times n}$ with $\epsilon \leq 1/6$, then in practice even without access to the expansion parameter $\epsilon$ D-EBF can still be deployed by using the upper bound 1/6 instead. We demonstrate this empirically in Section \ref{subsec:DEBF_numerics}. Note also that the same reasoning applies to any bound on $\epsilon$ made available to the decoder in advance.

As defined in Algorithm \ref{alg:DEBF_detail}, D-EBF requires knowledge of the column sparsity $d$ and number of columns $n$ of $\mA$. We speculate that it may be possible to estimate or use upper and lower bounds on $d$ and therefore still deploy D-EBF without access to $d$ in advance. We also hypothesise that similar results and algorithms may be derivable when $A$ is not the adjacency matrix of a fixed degree expander, with instead the degree of each node being bounded in some interval with high probability. We leave the study of such questions to future work and proceed under the assumption that the decoder has access to $d$ in advance. With regard to requiring access to $n$, we claim this is an artefact of our problem setup rather than a necessity in practice. Indeed, instead of initialising $\hat{\mA}$ as an $m \times n$ array of zeros, the reconstruction of $\mA$ could instead be kept dynamic, with partial or complete columns of $\mA$ being added to $\hat{\mA}$ as and when they are recovered. Furthermore, D-EBF can be run in advance of receiving all of the columns of $\mY$. Consider the situation in which the columns of $\mY$ are received in sequence: as each new column arrives D-EBF can be ran only on this new column, in a so called turnstile model, or in combination with the residuals of previous columns, using the partial reconstruction of $\mA$ already acquired as an initialisation point. In either case, after the decoder has seen sufficiently many columns to enable it to recover the encoder matrix up to permutation, then future columns can be decoded efficiently as and when they arrive, using a decoding algorithm from the CCS literature.

To summarise, in addition to $\mY$, then in practice D-EBF requires access only to the column sparsity $d$ of the encoder in advance. Furthermore, D-EBF can easily be adapted to process subsets of columns of $\mY$, allowing it to be deployed for example in data streaming contexts.

\subsection{Performance of D-EBF on mid-sized problems} \label{subsec:DEBF_numerics}
As a result of Lemma \ref{lemma:whp_expander} and Theorem \ref{theorem:main}, then for sufficiently large, sparse problems the assumptions upon which D-EBF is based hold with high probability. However, understanding the regimes in which D-EBF will succeed or fail at a practical level is not clear. Indeed, the probability bounds derived in \cite{DBLP:journals/corr/abs-1804-09212} are loose, making it hard to specify exactly how large $n$ needs to be in order for the encoder $A$, sampled from the PSB model, to satisfy $A \in \cE_{k, \epsilon, d}^{m \times n}$ with $\epsilon \leq 1/6$ with high probability. Additionally, even if this assumption holds the probability bound provided in Theorem \ref{theorem:main} is asymptotic and also likely to be loose. The purpose of this section is to provide a preliminary investigation into the empirical performance of D-EBF, testing the parameter settings for which the algorithm is likely to be successful. In particular, as already highlighted our theory suggests that D-EBF will likely be successful on large, sparse problems and unsuccessful on small, dense ones. We therefore focus our attention on mid-sized problems, varying $N$ and the ratio $k/n$, in order to better understand the transition point between success and failure.

The parameter settings for our experiments are as follows: $n=1000$, $m=800$ and $d=10$ are kept fixed, with $k$ varied between $1\%$ and $10\%$ of $n$ and $N \in \{100,200,300\}$. For each of these parameter settings 10 trials where computed: in each case an $A$ and an $X$ were generated as per the PSB model, Definition \ref{def_data_model}, with the coefficients of $X$ being sampled from the uniform distribution over an interval bounded away from 0 (in this case $[0.1,10.1]$). For each trial the D-EBF algorithm was then deployed to recover $A$ and $X$ up to permutation from $Y=AX$ with the $\epsilon$ input parameter set to 1/6: we note here that in our experiments D-EBF, Algorithm \ref{alg:DEBF_detail}, was not amended as per the discussion in Section \ref{subsec:deploying_DEBF} to include an additional merge subroutine. With $\hat{A}$ and $\hat{X}$ denoting the reconstructions returned by D-EBF, then for each parameter setting the Frobenius norm of the residual $||Y - \hat{A} \hat{X}||_F$ was computed as a percentage of $||Y||_F$ and averaged over the 10 trials. In addition, for each parameter setting the number of iterations of the while loop of D-EBF, lines 6-25 of Algorithm \ref{alg:DEBF_detail}, was counted and likewise averaged over the 10 trials. The outcomes of these experiments are shown in Figure \ref{fig:DEBF_numerics}.

\begin{figure}[!htb]
	\centering
	\includegraphics[width=0.9\textwidth]{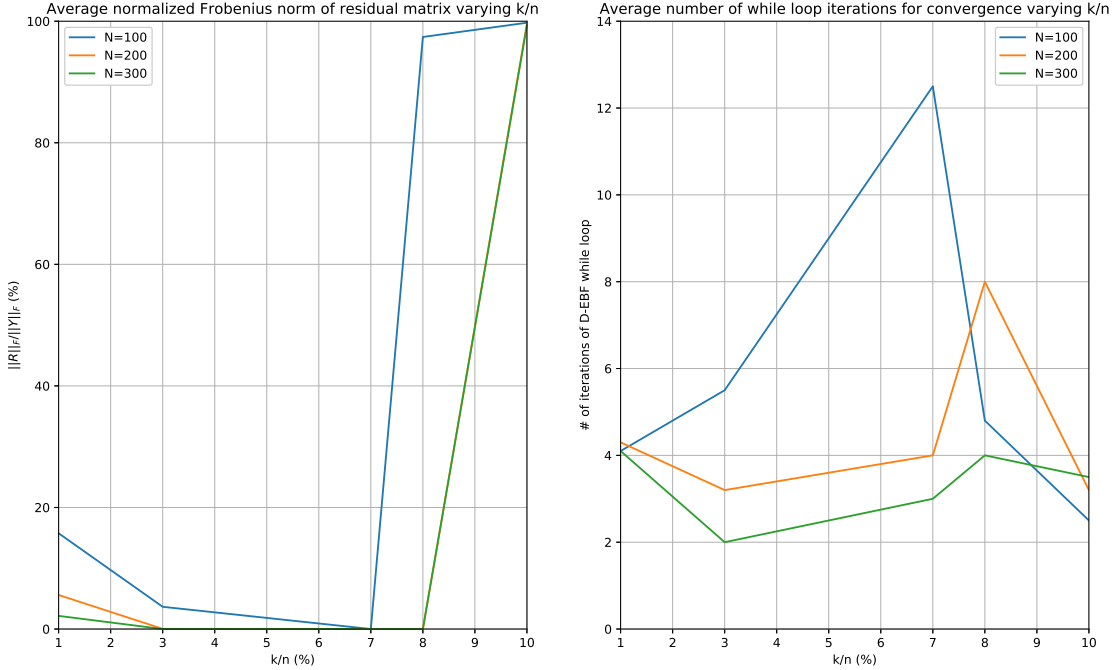}
	\caption{The left-hand plot shows the percentage $||\mY - \hat{\mA} \hat{\mX}||_F/||\mY||_F$, averaged over 10 trials with each trial generated as per the PSB model, Definition \ref{def_data_model}. Here $\hat{\mA}$ and $\hat{\mX}$ denote the reconstructions returned by D-EBF. The right-hand plot shows the average number of iterations of the while loop of D-EBF, lines 6-25, required for convergence over the same trials. The parameters $k$ and $N$ are varied while $n = 1000$, $m = 800$, $d = 10$ are fixed. In all trials the expansion parameter $\epsilon$ of $A$ is not known or computed in advance, D-EBF is deployed using $1/6$ in its stead.}
	\label{fig:DEBF_numerics}
\end{figure} 

In the left-hand plot of Figure \ref{fig:DEBF_numerics} we observe that $Y$ is most likely to be successfully factorized when $N$ is large and $k/n$ is neither too small or too large. Indeed, a larger $N$ will typically imply more partial supports extracted at each iteration of the while loop and therefore a greater chance of recovering $A$ up to permutation. This in turn allows for the recovery of $X$ up to permutation. Similar to the situation in which $N$ is small, when $k/n$ is small fewer partial supports are extracted per column of $Y$ at each iteration of the while loop, making the recovery of $A$ harder. As can be observed in the left-hand plot of Figure \ref{fig:DEBF_numerics}, this results in the incomplete factorisation of $Y$ for $k/n < 3\%$. Note also for $k/n < 3\%$, that as $N$ increases $Y$ on average becomes closer to being fully factorized: this is because the partial supports extracted from the additional columns compensate for the reduced number of partial supports extracted per column. We would expect even for small $k/n$ that as $N$ increases further then $Y$ will converge towards being fully factorized. When $k/n$ is large a similar issue occurs in terms of few or even no partial supports being extracted per column of $Y$. However, this is due to the fact that with $m$ and $n$ fixed then as $k$ increases the probability that $A \in \cE_{k, \epsilon, d}^{m \times n}$ with $\epsilon \leq 1/6$ converges towards $0$. When this assumption is not satisfied it is likely no longer possible to identify singleton values using frequency of occurrence, let alone cluster their associated partial supports. This can be observed in the failure to make any progress in factorising  $Y$ when $k/n = 10 \%$. Furthermore, unlike the former situation, increasing $N$ in this context is unlikely to yield any improvement. The middle ground, $3\% \leq k/n \leq 8\% $, appears ripe for factorisation, striking the right balance between $k$ being large enough so that each column of $A$ contributes to many of the columns of $Y$, while being sufficiently small so that singleton values can still be identified via the frequency with which they appear. 

Turning our attention to the right-hand plot of Figure \ref{fig:DEBF_numerics}, we observe that as $N$ increases then the number of iterations of the while loop of D-EBF generally decreases. This is to be expected as a larger $N$ means more partial supports extracted at each iteration, and therefore we would expect $A$ to be recovered in fewer iterations. Observe for $N =100$ and $k/n = 7\%$ that the iterative approach of D-EBF is invaluable in being able to factorise $Y$. To summarise, these experiments demonstrate that it is possible to practically decode a set of linear measurements without access to the full encoder matrix. While Theorem \ref{theorem:main} provides asymptotic guarantees, these experiments illustrate that D-EBF is successful even on relatively small to medium sized problems. We expect for large problems the performance of D-EBF to improve, allowing for problems with larger $k/n$ ratios to be factorized successfully. We leave a comprehensive empirical study of the D-EBF algorithm to future work.

\section{Concluding remarks and potential extensions}\label{sec:conclusion}
In this paper we studied the feasibility of performing multimeasurement combinatorial compressed sensing in the setting where the decoder does not have access to the encoder matrix. To this end we presented a particular random encoder matrix, see Algorithm \ref{alg:generate_A}, and the Decoder-Expander Based Factorisation algorithm, summarised in Algorithm \ref{alg:DEBF_detail}. We proved that the combination of these two allows for the successful factorisation of a measurement matrix $Y$ drawn from PSB model, thereby recovering both the encoder and sparse coding matrices up to permutation, with high probability and with near optimal sample complexity. In particular, the decoder only needs knowledge of the column sparsity $d$ of the encoder matrix in advance. We further illustrated how this combination also enables recovery of the location information by applying an additional column ordering procedure on the encoder matrix, agreed on in advance, both prior to encoding and after decoding. We then demonstrated that this algorithm performs well in practice even on mid-sized problems. In addition, the per while loop iteration computational complexity of D-EBF is $\cO(k^2(n+N)N)$: this, combined with experimental evidence indicating only a modest number of iterations are required for convergence, as well as ample opportunities for parallelization, suggest that the computational cost of D-EBF at scale is very manageable. Some potential extensions of this work are as follows.

\begin{itemize}
	\item \textbf{Relaxing the $\epsilon \leq 1/6$ condition:} Lemma \ref{corollary_sort_ps} guarantees that partial supports can be trivially clustered if the binary factor matrix $A$ is a $(k,\epsilon,d)$-expander graph with expansion parameter $\epsilon \leq 1/6$.  This expansion parameter is guaranteed over all $\binom{n}{k}$ sets of $k$ columns in $A$.  As $Y$ contains at most $N$ such sets of $k$ columns, and as $N$ would typically be exponentially smaller than this binomial coefficient, it is expected that similar bounds could be derived under a model for which it holds only with high probability that a subset of $k$ columns of $A$ satisfies the expansion bound.
	
	\item \textbf{Relaxing the dissociated condition condition:} similar to the the first point, it may suffice if most if not all possible subsets of the nonzeros of a sparse code are dissociated. However, adversarial choices of nonzero values in $X$ may require combinatorial searches on certain entries of $Y$ and could result in $Y$ no longer having a unique factorisation.
	
	\item \textbf{Stability to noise and projections of arbitrary sparse matrices onto the set of PSB matrices:}  more developed matrix factorizations, such as PCA, subspace clustering, and dictionary learning, are all known to be effective for matrices which only approximately achieve the modelling assumptions in context. An extension of the PSB model would be to show that the factors are stable to additive noise and remain efficient and reliable to compute. Stability to additive noise one might expect to be achieved via related work in combinatorial compressed sensing, such as \cite{8292953}. This could be further augmented by having repeated measurement vectors contained in $Y$. In order for the PSB model to be more generally applicable, approximation rates are required. In particular, a derivation of the distance between an arbitrary sparse matrix $\mY$ and its projection onto the set of PSB matrices, which would presumably show better approximation as the number of columns $n$ of the encoder matrix $\mA$ is increased. To be clear, this would entail bounds of the form $\|\mY-P_{PSB(n)}(\mY)\|\leq f(n)$, where $P_{PSB(n)}(\cdot)$ projects to the nearest PSB matrix with parameters $k,m,n,d,N$ under some matrix norm, and $f(n)$ is a rapidly decreasing function of $n$ and or potentially $d$ and $k$. One would also expect expressivity benefits from relaxing the fixed degree condition, instead allowing the number of nonzeros per column of $A$ to vary between an upper and lower bound. However, this would likely increase the computational complexity of the factorisation task as the number of nonzeros per column would also need to be inferred from $Y$.  
\end{itemize}

% 	\item \textbf{Broadening the applicability of PSB model:} two assumptions that restrict the expressivity of the PSB model are that there is a fixed number $d$ nonzeros per column of $A$ and that the columns of $X$ are dissociated. Relaxing these conditions would clearly aid in the applicability of the model. One would expect it to be possible to derive similar results for the case in which the number of nonzeros per column of $A$ is bounded from above and below. However, this would likely increase the computational complexity of the factorisation task as the number of nonzeros per column would also need to be inferred from $Y$. Robustness to the entries in $X$ being dissociated, at least with high probability, can also be expected. However, adversarial choices of nonzero values in $X$ may require combinatorial searches on certain entries of $Y$ and could result in $Y$ no longer having a unique factorisation.
	
% 	\item \textbf{Applications:} in Section \ref{subsec:summary_contributions} we sketched a few scenarios in which decoding a set of sparse codes without access to the encoder matrix could be useful. Exploration and investigation of applications where this capability might be useful, and the efficacy of D-EBF in these contexts, would be interesting.

\section*{Acknowledgements}
This work is supported by the Alan Turing Institute under the EPSRC grant EP/N510129/1 and the Ana Leaf Foundation. The authors would also like to thank Robert Calderbank for his helpful comments and advice.

\clearpage
\bibliographystyle{unsrt}  
\bibliography{references}
\clearpage
\appendix
\section{Extraction of singleton values and partial supports}\label{appendix:singletons_partials}
In this appendix we provide proofs of the results given in Section \ref{sec:alg} concerning the extraction of singleton values and the clustering of partial support. To quickly recap some key pieces of notation we recall the following function definitions.
\begin{itemize}
    \item Let the function $f: \reals \times \reals^m \rightarrow [m]$ count the number of times a real number $\alpha \in \reals$ appears in some vector $\vr \in \reals^m$, $f(\alpha, \vr) \defeq | \{j \in [m]: r_j = \alpha \}|$.
    \item Let the function $g: \reals \times \reals^m \rightarrow \{0,1 \}^m$ return a binary vector whose nonzeros correspond to the locations in which $\alpha \in \reals$ appears in a vector $\vr \in \reals^m$, i.e., with $\vw = g(\alpha, \vr)$ then $w_j = 1$ iff $r_j = \alpha$ and is $0$ otherwise.
\end{itemize}

We first prove Lemma \ref{lemma:adjmat}, the three properties of expander graph adjacency matrices identified in this lemma will provide the basis for much of what follows.

\textbf{Lemma \ref{lemma:adjmat} (Properties of the adjacency matrix of a $(k, \epsilon, d)$-expander graph)} \textit{ If $\mA \in \cE^{m \times n}_{k, \epsilon, d}$ then any submatrix $\mA_{\cS}$ of $\mA$, where $\cS \in [n]^{(\leq k)}$, satisfies the following.
	\begin{enumerate}
		\item There are more than $(1-\epsilon)d |\cS|$ rows in $\mA_{\cS}$ that have \textbf{at least one} non-zero.
		\item There are more than $(1-2\epsilon)d |\cS|$ rows in $\mA_{\cS}$ that have \textbf{only one} non-zero.
		\item The overlap in support of the columns of $\mA_{\cS}$ is upper bounded as follows,
		\[
		|\bigcap_{l \in \cS} \supp(\va_l)| < 2 \epsilon d.
		\]
	\end{enumerate}}

\begin{proof}
Property 1 follows directly from Definition \ref{def_exp} and property 2 from Theorem \ref{theorem_unp}. To prove property 3, consider the pairwise overlap in support between any two distinct columns $\va_l$ and $\va_h$ of $\mA \in \cE_{k, \epsilon, d}^{m \times n}$. From Definition \ref{def_exp} it follows that
	\[
	|supp(\va_l)\cup supp(\va_h)| > (1-\epsilon)2d.
	\]
	Using the inclusion-exclusion principle then
	\[
	|supp(\va_l)| + |supp(\va_h)| - | supp(\va_l) \cap supp(\va_h)| > (1- \epsilon)2d.
	\]
	Given that $|supp(\va_l)| = |supp(\va_h)| = d $, then a simple rearrangement shows that
	\[
	| supp(\va_l)\cap supp(\va_h) | < 2d - (1- \epsilon)2d = 2 \epsilon d.
	\]
	Therefore, for any $\cS \subseteq [n]$ with $|\cS|\geq 2$ and for any $l,h \in \cS$ such that $l \neq h$, then
	\[
	|\bigcap_{l \in \cS} supp(\va_l)|\leq| supp(\va_l)\cap supp(\va_h) | < 2 \epsilon d.
	\]
\end{proof}

Using Lemma \ref{lemma:adjmat} it is possible to derive the following sufficient condition for identifying singleton values, based on the frequency with which they appear.

\textbf{Corollary \ref{corollary:suff_cond_singletons} (Sufficient condition for identifying singleton values (i))} \textit{Consider a vector $\vr = \mA \vz$ where $\mA \in \cE^{m \times n}_{k, \epsilon, d}$ and $\vz\in \cX^n_k$. With $\alpha \in \reals \backslash \{0\}$ if $f(\alpha, \vr)\geq 2 \epsilon d$ then there exists an $l \in \supp(\vx)$ such that $\alpha = x_l$.}

\begin{proof}
    By construction if $\alpha \in \reals \backslash \{0\}$ and $f(\alpha, \vr)>0$ then there exists a $\cS \subseteq \supp(\vz)$ such that $\alpha = \sum_{l \in \cS} x_l$. Suppose that $|\cS|\geq 2$. Then
    \[
    \begin{aligned}
    f(\alpha, \vr) &= |\{j \in [m]: \cS = \supp(\tilde{\va}_j)\cap \supp(\vz) \} |\\
    & \leq |\{j \in [m]: \cS \subseteq \supp(\tilde{\va}_j)\cap \supp(\vz) \} |\\
    & \leq |\{j \in [m]: \cS \subseteq \supp(\tilde{\va}_j)\} |\\
    & =| \bigcap_{l \in \cS} \supp(a_l)|\\
    & < 2\epsilon d,
    \end{aligned}
    \]
    where the final inequality follows from Lemma \ref{lemma:adjmat}. Therefore if $\alpha \in \reals \backslash \{0\}$ satisfies $f(\alpha, \vr) \geq 2 \epsilon d$ it must follow that $\alpha$ is a singleton value, i.e., $|\cS|=1$, and therefore there exists an $l \in \supp(\vx)$ such that $\alpha = x_l$.
\end{proof}

Corollary \ref{corollary:suff_cond_singletons} provides a sufficient condition for identifying singleton values, however it does not provide any guarantees that there will be any singleton values appearing at least $2 \epsilon d$ in any such $\vr$. This existence of high frequency singleton values is resolved by the following lemma.

\textbf{Lemma \ref{lemma_existence_singletons} (Existence of frequently occurring singleton values, adapted from \cite[Theorem 4.6]{mendoza-smith2017a})} \textit{ Consider a vector $\vr = \mA \vz$ where $\mA \in \cE^{m \times n}_{k, \epsilon, d}$ and $\vz \in \cX^n_k$. For $l \in \supp(\vz)$, let $\Omega_l \defeq \{j \in [m]: r_j= z_l\}$ be the set of row indices $j \in [m]$ of $\vr$ such that $r_j= z_l$. Defining $\cT \defeq \{l \in \supp(\vz): |\Omega_l| > ( 1-2 \epsilon) d \}$ as the set of singleton values which each appear more than $(1-2\epsilon)d$ times in $\vr$, then
	\[
	| \cT | \geq \frac{|\supp(\vz)|}{(1+2\epsilon)d}.
	\]}
	
\begin{proof}
	For typographical ease let $\gamma \defeq \ceil{(1-2\epsilon)d}$ and define $U \defeq \supp(\vz)  \backslash \cT$ as the set of singleton values which appear at most $(1-2\epsilon)d$ times.  Additionally, for this proof we adopt the following graph inspired notation.
	\begin{itemize}
		\item $\cN_1(\vr) \defeq \bigcup_{l \in \supp(\vz)} \Omega_l$ is the set of row indices of $\vr$ corresponding to singleton values, by Theorem \ref{theorem_unp} it holds that $|\cN_1(\vr)| \geq \gamma |\supp(\vz)|$.
		\item $\cN_1^{\cT}(\vr) \defeq \bigcup_{l \in \cT} \Omega_l$ is the set of row indices of $\vr$ corresponding to singleton values that appear more than $(1-2 \epsilon )d$ times. Therefore $ \gamma |\cT| \leq |\cN_1^{\cT}(\vr)| \leq d |\cT|$.
		\item $\cN_1^{\cU}(\vr) \defeq \bigcup_{l \in \cU} \Omega_l$ is the set of row indices of $\vr$ corresponding to singleton values that appear at most $(1-2 \epsilon )d$ times. Note that $ |\cN_1^{\cU}(\vr)| \leq (\gamma - 1) |\cU|$.
	\end{itemize}
	\noindent Clearly $|\cN_1(\vr)| = |\cN_1^{\cT}(\vr)| + |\cN_1^{\cU}(\vr)|$, it therefore follows that
	\[
	\begin{aligned}
	|\cN_1(\vr)| & \geq \gamma |\supp(\vz)|\\
	& =  \gamma(|\cT| + |\cU|)\\
	& =  \gamma|\cT| + |\cU| + (\gamma - 1) |\cU|\\
	& \geq  \gamma |\cT| + |\cU| + |\cN_1^{\cU}(\vr)|.
	\end{aligned}
	\]
	Substituting $|\cN_1(\vr)| = |\cN_1^{\cT}(\vr)| + |\cN_1^{\cU}(\vr)|$ then
	\[
	\begin{aligned}
	|\cN_1^{\cT}(\vr)| &\geq \gamma |\cT| + |\cU| \\
	& = |\cT| + |\cU| + (\gamma -1)|\cT|\\
	& = |\supp(\vz)|+ (\gamma -1)|\cT|.
	\end{aligned}
	\]
	As $|\cN_1^{\cT}(\vr)| \leq d |\cT|$ then $d |\cT| \geq |\supp(\vz)| + (\gamma -1)|\cT|$, rearranging gives
	\[
	| \cT | \geq \frac{|\supp(\vz)|}{d - \gamma + 1}.
	\]
	Finally, as $\gamma \geq d - 2 \epsilon d$ then
	\[
	| \cT | \geq  \frac{|\supp(\vz)|}{1+ 2 \epsilon d}
	\]
	as claimed. Note in addition that so long as $\vr$ is nonzero, meaning  $|\supp(\vz)| \geq 1$, then it is always possible to extract at least one partial support with cardinality greater than $(1-2\epsilon)d$.
\end{proof}

In addition, the sufficient condition concerning the extraction of singleton values and partial supports, Corollary \ref{corollary:suff_cond_singletons}, can be extended to the following case.

\textbf{Lemma \ref{lemma:freq_non_singletons_residual} [Sufficient condition for identifying singleton values (ii)]}\textit{ Consider a vector $\vy = \mA \vx$ where $\mA \in \cE^{m \times n}_{k, \epsilon, d}$ and $\vx \in \cX^n_k$. Let $\hat{\mA} \in \{0,1\}^{m \times n}$ and $\hat{\vx}\in\cX^n_k$ be such that a column $\hat{\va}_l$ of $\hat{\mA}$ is nonzero iff $\hat{x}_l\neq 0$, and there exists a permutation matrix $\mP \in \cP^{n \times n}$ such that $\supp(\hat{\mA}\mP^T) \subseteq \supp(\mA)$, $\supp(\mP\hat{\vx}) \subseteq \supp(\vx)$ and $\hat{x}_{P(l)} = x_{P(l)}$ for all $l \in supp(\hat{\vx})$, where $P:[n] \rightarrow [n]$ denotes the row permutation caused by pre-multiplication with $\mP$. Consider the residual $\vr = \vy - \hat{\mA} \hat{\vx}$, if for some $\alpha \in \reals \backslash \{ 0\} $ it holds that $f(\alpha, \vr) \geq 2 \epsilon d$ is satisfied, then there exists an $l \in [n]$ such that $\alpha = x_l$. Furthermore, with $\vw = g(\alpha, \vr)$ then $\supp(\vw) \subseteq \supp(\va_l)$ is a partial support of $\va_l$.}

\begin{proof}
As in Corollary \ref{corollary:suff_cond_singletons}, then by construction if $\alpha \in \reals \backslash \{0\}$ and $f(\alpha, \vr)>0$ then there exists a $\cS \subseteq \supp(\vz)$ such that $\alpha = \sum_{l \in \cS} x_l$. In order to simplify our notation for what follows, let $\mV \defeq \hat{\mA}\mP^T$ and $\vu \defeq \mP\hat{\vx}$. By construction $\supp(\mV) \subseteq \supp(\mA)$ and $\supp(\vu)\subseteq \supp(\vx)$. In addition, with $\Omega_j \defeq \supp(\tilde{\va}_j) \cap \supp(\vx)$ and $\Gamma_j \defeq \supp(\tilde{\vv}_j) \cap \supp(\vu)$, then $\Gamma_j \subseteq \Omega_j$ for all $j \in [m]$. As $u_l = x_l$ for all $l \in \supp(\vu)$ and $\vr = \vy - \mV \vu$ then
\[
\begin{aligned}
r_j = \sum_{l \in \Omega_j} x_l -  \sum_{h \in \Gamma_j} x_h = \sum_{l \in (\Omega_j \backslash \Gamma_j)} x_l.
\end{aligned}
\]
It therefore follows that each entry $r_j$ is a sum over a subset of the nonzero values of $\vx$ used in the sum of $y_j$. As a result, we may write $\vr \defeq \mB \vx$ where $\mB \in \{0,1\}^{m \times n}$ and $\supp(\mB)\subseteq \supp(\mA)$. Therefore, for any $\cS \subseteq \supp(\vx)$ such that $|\cS| \geq 2$, then from Lemma \ref{lemma:adjmat} it follows that
\[
|\bigcap_{l \in \cS} \supp(\vb_l)| \leq |\bigcap_{l \in \cS} \supp(\va_l)| < 2\epsilon d.
\]
Therefore, if $\alpha = \sum_{l \in \cS}x_l$ with $|\cS| \geq 2$ then $f(\alpha, \vr)< 2\epsilon d$. As a result, for some $\alpha \in \reals$ if $f(\alpha, \vr)\geq 2 \epsilon d$ then $\alpha$ is a singleton and there exists an $l \in [n]$ such that $\alpha=x_l$. With $\vw = g(\alpha, \vr)$, then as $\vx$ is dissociated it follows that $\supp(\vw) \subseteq \supp(\vb_l) \subseteq \supp(\va_l)$ and so $\vw$ is a partial support of $\va_l$.
\end{proof}

The next result provides guarantees with regards to clustering partial supports.

\textbf{Lemma \ref{corollary_sort_ps} (Clustering partial supports)} \textit{ Consider a pair of partial supports $\vw_1$ and $\vw_2$, extracted from $\vr_1 = \mA \vz_1$ and $\vr_2 = \mA \vz_2$ respectively, where $\mA \in \cE^{m \times n}_{k, \epsilon, d}$ and $\vz_1, \vz_2 \in \cX^n_k$. If $\epsilon \leq 1/6$, $|\supp(\vw_1)| >(1 - 2 \epsilon)d$ and $|\supp(\vw_2)| >(1 - 2 \epsilon)d$, then $\vw_1$ and $\vw_2$ originate from the same column of $\mA$ iff $\vw_1^T \vw_2 \geq 2 \epsilon d$.}

\begin{proof}
    Suppose that $\vw_1$ and $\vw_2$ originate from $i$th and $j$th column of $\mA$ respectively. From Lemma \ref{lemma:adjmat}, and observing that $|supp(\va_i) \cap supp(\va_j)| = \va_i^T\va_j$ for binary vectors $\va_i$ and $\va_j$, then $i \neq j$ iff $\va_i^T \va_j < 2 \epsilon d$. Observe also that as $\text{supp}(\vw_1) \subseteq \text{supp}(\va_i)$ and $\text{supp}(\vw_2) \subseteq \text{supp}(\va_j)$ then $\vw_1^T \vw_2 \leq \va_i^T \va_j$. Suppose that $\vw_1^T \vw_2 \geq 2 \epsilon d$ and assume that $i \neq j$. This is a contradiction however as it must hold that $\va_i^T \va_j < 2 \epsilon d$, but
    \[
    \va_i^T \va_j \geq \vw_1^T \vw_2 \geq 2 \epsilon d.
    \]
    Therefore $\vw_1^T \vw_2 \geq 2 \epsilon d$ implies that $i=j$, i.e., $\vw_1$ and $\vw_2$ originate from the same column of $\mA$. Now assume as per the statement of the lemma that $|\text{supp}(\vw_1)|>(1-2\epsilon)d$, $|\text{supp}(\vw_2)|>(1-2\epsilon)d$ and $\epsilon \leq 1/6$, which implies that $(1-2\epsilon)d \geq 2 \epsilon d$. Suppose $\vw_1^T \vw_2 < 2 \epsilon d$, if $i=j$ then $\vw_1$ and $\vw_2$ must differ by less than $4 \epsilon d$ nonzeros. This implies that
    \[
    \vw_1^T \vw_2 > (1 - 4\epsilon)d \geq 2 \epsilon d
    \]
    which is a contradiction. Therefore, under the assumptions given it follows that $\vw_1^T \vw_2 < 2 \epsilon d$ implies $i \neq j$. Combining these results then $\vw_1$ and $\vw_2$ originate from the same column of $\mA$ iff $\vw_1^T \vw_2 \geq 2 \epsilon d$ as claimed.
\end{proof}

\section{Accuracy of D-EBF and ND-EBF}\label{appendix:accuracy}
In this appendix we prove results concerning the operation of D-EBF and ND-EBF. Before proceeding we emphasise a point regarding notation: many of the variables discussed are updated iteratively throughout the run-time of an algorithm. Therefore, in order to avoid a proliferation of distinct variables representing the same program variable at different stages of run-time, unless otherwise stated we will treat each variable as a program variables. We start by proving accuracy guarantees for the subroutines common to both D-EBF and ND-EBF, starting with EXTRACT\&MATCH, detailed in Algorithm \ref{alg:extract}.

\begin{lemma}[\textbf{Accuracy of Algorithm \ref{alg:extract}}] \label{lemma:extract_accuracy}
     Consider a vector $\vy = \mA \vx$, where $\mA \in \cE^{m \times n}_{k, \epsilon, d}$ with $\epsilon \leq 1/6$ and $\vx \in \cX^n_k$. Let $\hat{\mA}^{(1)} \in \{0,1\}^{m \times n}$ and $\hat{\vx}^{(1)}\in\cX^n_k$ be such that a column $\hat{\va}^{(1)}_l$ of $\hat{\mA}^{(1)}$ is nonzero iff $\hat{x}_l^{{(1)}}\neq 0$, and if it is nonzero then  $|\supp(\hat{\va}^{(1)}_l)|>(1-2\epsilon)d$. In addition, assume that there exists a permutation matrix $\mP \in \cP^{n \times n}$ such that $\supp(\hat{\mA}^{(1)}\mP^T) \subseteq \supp(\mA)$, $\supp(\mP\hat{\vx}^{(1)}) \subseteq \supp(\vx)$ and $\hat{x}_{P(l)}^{(1)} = x_{P(l)}$ for all $l \in supp(\hat{\vx}^{(1)})$, where $P:[n] \rightarrow [n]$ denotes the row permutation caused by pre-multiplication with $\mP$. Define the residual $\vr\defeq \vy - \hat{\mA}^{(1)} \hat{\vx}^{(1)}$ and consider the variables returned by Algorithm \ref{alg:extract},
     \[
     \vq, \mW,p, \hat{\mA}^{(2)}, \hat{\vx}^{(2)} \leftarrow \mathrm{EXTRACT\&MATCH}(\vr, m, d, \epsilon, \hat{\mA}^{(1)}, \hat{\vx}^{(1)}).
     \]
     Then the following hold.
     \begin{enumerate}
        \item If $p \geq 1$ then for each $i \in [p]$ there exists a unique $l \in \supp(\vx)$ such that $q_i=x_l$ and $\supp(\vw_i) \subseteq \supp(\va_l)$ with $|\supp(\vw_i)|>(1-2\epsilon)d$. In addition, $\va_l^T \hat{\va}_h^{(1)}< 2 \epsilon d$ for all $h \in [n]$.
        \item The updated reconstructions satisfy the following: a column $\hat{\va}^{(2)}_l$ of $\hat{\mA}^{(2)}$ is nonzero iff $\hat{x}_l^{{(2)}}\neq 0$, $\supp(\hat{\mA}^{(1)}\mP^T) \subseteq \supp(\hat{\mA}^{(2)}\mP^T) \subseteq \supp(\mA)$, $\supp(\mP\hat{\vx}^{(1)}) \subseteq \supp(\mP\hat{\vx}^{(2)}) \subseteq \supp(\vx)$ and $\hat{x}_{P(l)}^{(2)} = x_{P(l)}$ for all $l \in \supp(\hat{\vx}^{(2)})$.
     \end{enumerate}
\end{lemma}

\begin{proof}
We prove both statements by analysing the updates to the reconstructions, which take place on lines 21, 22, 25, 26 and 32. To this end, we first prove that the variables $r_i$ and $\vw$ in Algorithm \ref{alg:extract} are a singleton value and partial support respectively if $c>(1 -2 \epsilon)d$. We then show that when $\vw$ and a nonzero column of $\hat{\mA}^{(1)}$ are partial supports originating from the same column in $\mA$, then $\vw$ is correctly used to update this column of the reconstruction and $r_i$ the corresponding row of $\hat{\vx}^{(1)}$. In the case where $\vw$ originates from a different column of $\mA$ to any of the nonzero columns of $\hat{\mA}^{(1)}$, we show that $\vw$ and $r_i$ are correctly added to $\mW$ and $\vq$ respectively.

By the assumptions of the lemma the residual $\mR$ satisfies the conditions of Lemma \ref{lemma:freq_non_singletons_residual}. Therefore, and as $\epsilon \leq 1/6$, if on line 17 $c=f(r_i, \vr)>(1-2 \epsilon)d \geq 2 \epsilon d$ then there exists an $l \in \supp(\vx)$ such that $r_i= x_l$. Furthermore, with $\vw = g(r_i, \vr)$ then $\supp(\vw) \subseteq \supp(\va_l)$ and  $|\supp(\vw)| = f(r_i, \vr)>(1-2 \epsilon)d$. We conclude then that if $c>(1-2\epsilon)d$ then $r_i$ and $\vw$ are a singleton and partial support respectively. All that remains to be shown is that the partial support $\vw$ is either correctly matched to an existing nonzero column of $\hat{\mA}^{(1)}$, or it is correctly identified as originating from a column of $\mA$ for which there is not currently a partial or full reconstruction. We now assume that $\vw$ is a partial support of $\va_l$ and examine each of these two cases in turn.

Suppose there exists a nonzero column $\hat{\va}_h^{(1)}$ of $\hat{\mA}^{(1)}$ which is a partial support originating from $\va_l$, i.e., $\supp(\hat{\va}_h^{(1)}) \subseteq \supp(\va_l)$. As $|\supp(\hat{\va}_h^{(1)})| > (1 -2 \epsilon)d \geq 2 \epsilon d$ then by Lemma \ref{lemma:adjmat} it follows that $\supp(\hat{\va}_h^{(1)}) \nsubseteq \supp(\va_k)$ for all $k\neq h$. From Lemma \ref{corollary_sort_ps} then $\vw^T \hat{\va}_h^{(1)} \geq 2 \epsilon d$ and $\vw^T \hat{\va}_k^{(1)}< 2 \epsilon d$ for all $k\neq l$. As a result, on line 19 it must follow that $\kappa = h$ and on line 20 that $t_{\kappa} \geq 2 \epsilon d$. Therefore, the partial support $\vw$ is identified as matching with the reconstruction of $\va_l$ in $\hat{\mA}^{(1)}$, and $\hat{x}_l$ and $\hat{\va}_h$ are then correctly updated on lines 21 and 22 accordingly. Consider now the other case, in which there does not exist a nonzero column $\hat{\va}_h^{(1)}$ of $\hat{\mA}^{(1)}$ which is a partial support originating from $\va_l$. Therefore, by Lemma \ref{corollary_sort_ps}, for any $h \in [n]$ then $\vw^T\hat{\va}_h^{(1)} < 2 \epsilon d$ and so the condition on line 20 is not satisfied. As a result, the singleton value $r_i$ and partial support $\vw$ are correctly used to update $\vq$ and $\mW$ on lines 25 and 26, and as a consequence statement 1 of the lemma immediately follows. For Statement 2, in addition observe that line 32 introduces no erroneous nonzeros as $\vx$ is dissociated. Note also that as updates only occur to the existing nonzero columns of $\hat{\mA}^{(1)}$, the sets of column indices corresponding to the nonzero columns of $\hat{\mA}^{(1)}$ and $\hat{\mA}^{(2)}$ are equal. Given these observations Statement 2 must also be true.
\end{proof}

Under the similar assumptions we are also able to provide the following accuracy guarantees for the DECODE subroutine, detailed in Algorithm \ref{alg:decode}.

\begin{lemma}[\textbf{Accuracy of Algorithm \ref{alg:decode}}]\label{lemma:accuracy_decode}
    Consider a vector $\vy = \mA \vx$, where $\mA \in \cE^{m \times n}_{k, \epsilon, d}$ with $\epsilon \leq 1/6$ and $\vx \in \cX^n_k$. Let $\hat{\mA} \in \{0,1\}^{m \times n}$ and $\hat{\vx}^{(1)} \in \cX^n_k$ be such that a column $\hat{\va}_l$ of $\hat{\mA}$ is nonzero iff $\hat{x}_l^{{(1)}}\neq 0$, and if it is nonzero then  $|\supp(\hat{\va}_l)|>(1-2\epsilon)d$. In addition, assume there exists a permutation matrix $\mP \in \cP^{n \times n}$ such that $\supp(\hat{\mA}\mP^T) \subseteq \supp(\mA)$, $\supp(\mP\hat{\vx}^{(1)}) \subseteq \supp(\vx)$ and $\hat{x}_{P(l)}^{(1)} = x_{P(l)}$ for all $l \in supp(\hat{\vx}^{(1)})$, where $P:[n] \rightarrow [n]$ denotes the row permutation caused by pre-multiplication with $\mP$. Consider the variables computed during and returned by DECODE, Algorithm \ref{alg:decode},
    \[
    \hat{\vx}^{(2)}, \mathrm{UPDATED} \leftarrow \mathrm{DECODE}(\vy,m,\epsilon, d,\hat{\mA},\hat{\vx}^{(1)}).
    \]
    Let $\cQ$ denote the event that $\supp(\mP\hat{\vx}) \subseteq \supp(\vx)$ and $\hat{x}_{P(l)} = x_{P(l)}$ for all $l \in \supp(\hat{\vx})$. Then the following hold.
    \begin{enumerate}
        \item At any point during the run-time of DECODE $\cQ$ is true.
        \item The sparse code reconstruction returned by DECODE satisfies $\supp(\mP\hat{\vx}^{(1)}) \subseteq \supp(\mP\hat{\vx}^{(2)})$. Furthermore, recalling that $\cS = \{l \in [n]: \hat{\va}_l = d\}$ and defining $P(\cS) \defeq \{h \in [n]: \exists l \in \cS \; s.t. \; P(l) = h \}$, if $\supp(\vx) \subseteq P(\cS)$ then $\mP\hat{\vx}^{(2)} = \vx$.
        \item The while loop, lines 5-31, exits after at most $k$ iterations.
    \end{enumerate}
\end{lemma}

\begin{proof}
First we prove Statement 1 by induction. Clearly, by the assumptions on $\hat{\vx}^{(1)}$, then $\cQ$ is true at the start of iteration 1. In addition, if $\cQ$ is true at the end of an iteration it must be true at the beginning of the next iteration. To prove Statement 1 it therefore suffices to prove that if $\cQ$ is true at the start, i.e., line 4, of iteration $t$, then it is also true at the end, i.e., line 31, of iteration $t$. Suppose then that $\cQ$ is true at the start of iteration $t$, then the residual at the start of iteration $t$ satisfies the conditions required for Lemma \ref{lemma:freq_non_singletons_residual}. As a result, if on line 20 it holds that $c=f(r_i, \vr) \geq 2 \epsilon d$, then there exists an $l \in \supp(\vx)$ such that $r_i= x_l$, meaning $r_i$ is a singleton value. Furthermore, with $\vw = g(r_i, \vr)$ then $|\supp(\vw)|\geq 2 \epsilon d$ and there exists a unique $l \in [n]$ such that $\supp(\vw) \subseteq \supp(\va_l)$. Similar to Lemma \ref{lemma:extract_accuracy}, we now explore the two possible cases: either there exists a unique $h \in \cS$ such that $\hat{\va}_h=\va_l$ or there does not. 

Suppose then that there exists a column $\hat{\va}_h$ such that $\hat{\va}_h = \va_l$. Therefore $\vw^T \hat{\va}_h \geq 2 \epsilon d$ and from Lemma \ref{corollary_sort_ps} it follows that $\vw^T \hat{\va}_k < 2 \epsilon d$ for all $k \neq n$. Therefore on line 22 it must be true that $\kappa = h$. As a result the condition on line 23 is satisfied and so the partial support $\vw$ is identified as matching with the complete reconstruction of $\va_l$ in $\hat{\mA}$. Subsequently the corresponding entry in $\hat{\vx}$ is correctly updated as per line 24. Suppose now for all $h \in \cS$ it holds that $\hat{\va}_h \neq \va_l$. From Lemma \ref{lemma:adjmat} it therefore follows that $t_h = \vw^T\hat{\va}_h \leq \va_l^T \va_{P(h)} < 2 \epsilon d$ for all $h \in \cS$. Therefore $t_{\kappa}<2 \epsilon d$ and hence the condition on line 23 is not satisfied. As a result, under this assumption $\vw$ correctly does not match with any complete column of the reconstruction and so no update occurs to $\hat{\vx}$. As $\hat{\mA}$ is unchanged and only the row entries of $\hat{\vx}$ with index in $\cS$ are updated, then this implies that $\cQ$ is true at the end of iteration $t$. This concludes the proof by induction of Statement 1. 

The first part of Statement 2 follows immediately from Statement 1. For the second part of Statement 2, let $\mV \defeq \hat{\mA}\mP^T$ and $\vu \defeq \mP\hat{\vx}^{(1)}$. By assumption $\supp(\mV) \subseteq \supp(\mA)$ and $\supp(\vu)\subseteq \supp(\vx)$, therefore, at any iteration of the while loop, the residual calculated on line 6 can be expressed as
\[
\vr = \vy - \hat{\mA}\hat{\vx} = \mA \vx - \mV \vu.
\]
For typographical ease let $\cC \defeq \supp(\vx)$ and suppose that $\cC \subseteq P(\cS)$, we proceed by contradiction and assume that Statement 2 is false. If Statement 2 is false then this implies that the while loop of DECODE exits at some iteration with $\vr \neq \textbf{0}_m$. Therefore, and as Statement 1 is true, it must follow that $f(x_l, \vr)< 2 \epsilon d$ for all $l \in \cC$. Due to our assumptions and the proof of Statement 1 then $\mV_{P(\cS)}= \mA_{P(\cS)}$, $\supp(\vu) \subseteq \cC$ and $u_l = x_l$ for all $l \in \supp(\vu)$. Therefore
\[
\vr = \mA_{\cC} \vx_{\cC} - \mV_{\cC} \vu_{\cC} = \mA_{\cC}\vz
\]
where $\mA_{\cC} \in \cE^{m \times k}_{k, \epsilon, d}$ with $\epsilon \leq 1/6$ and $\vz \defeq \vx_{\cC} - \vu_{\cC} \in \cX^k_k$. As $\mA_{\cC} \in \cE^{m \times k}_{k, \epsilon, d}$ with $\epsilon \leq 1/6$ and $\vz \in \cX^k_k$, $\vz \neq \textbf{0}_k$, then Lemma \ref{lemma_existence_singletons} ensures the existence an $l \in \supp(\vx) \backslash \supp(\vu)$ such that $f(x_l, \vr) > (1-2 \epsilon)d \geq 2 \epsilon d$, which is a contradiction. Therefore Statement 2 must also be true.

Finally, to prove Statement 3, observe that at each iteration prior to the exit iteration the contribution of at least one column of $\hat{\mA}_{\cS}$ is removed from the residual $\vr$. As $\vx$ is $k$ sparse then the maximum number of possible iterates is $k$.
\end{proof}

With Lemmas \ref{lemma:extract_accuracy} and \ref{lemma:accuracy_decode} in place then we are able to prove accuracy guarantees for Algorithm \ref{alg:DEBF_detail}, D-EBF.

\textbf{Lemma \ref{lemma:accuracy_DEBF} (Accuracy of D-EBF)} \textit{ Let $\mY = \mA \mX$, where $\mA \in \cE^{m \times n}_{k, \epsilon, d}$ with $\epsilon \leq 1/6$ and $\mX \in \cX^{n\times N}_k$. Consider Algorithm \ref{alg:DEBF_detail}, D-EBF: in regard to the reconstructions $\hat{\mA}$ and $\hat{\mX}$ computed at any point during the run-time of D-EBF, we will say $\cQ$ is true iff the following hold.
\begin{enumerate}[label=(\alph*)]
    \item A column $\hat{\va}_l$ of $\hat{\mA}$ is nonzero iff the $l$th row of $\hat{\mX}$ is nonzero.
    \item For all $l \in [n]$ if $\hat{\va}_l$ is nonzero then $\supp(\hat{\va}_l) > (1-2 \epsilon)d$.
    \item There exists a permutation matrix $\mP \in \cP^{n \times n}$ such that $\supp(\hat{\mA}\mP^T) \subseteq \supp(\mA)$, $\supp(\mP\hat{\mX}) \subseteq \supp(\mX)$ and $\hat{x}_{P(l),h} = x_{P(l), h}$ for all $h \in [N]$ and $l \in supp(\hat{\vx}_l)$, where $P:[n] \rightarrow [n]$ denotes the row permutation caused by pre-multiplication with $\mP$.
\end{enumerate}
Suppose that D-EBF exits after the completion of iteration $t_f\in \naturals$ of the while loop, then the following statements are true.
\begin{enumerate}
    \item At any point during the run time of D-EBF $\cQ$ is true. Therefore the reconstructions $\hat{\mA}$ and $\hat{\mX}$ returned by D-EBF also also satisfy the conditions for $\cQ$ to be true.
    \item For any iteration of the while loop $t <t_f$, let $\hat{\mA}^{(1)}$ and $\hat{\mX}^{(1)}$ denote the reconstructions at the start of the while loop, line 6, and $\hat{\mA}^{(2)}$ and $\hat{\mX}^{(2)}$ denote the reconstructions at the end of the while loop, line 26. Then either $\supp(\hat{\mA}^{(1)})\subset \supp(\hat{\mA}^{(2)})$ and $\supp(\hat{\mX}^{(1)})\subseteq \supp(\hat{\mX}^{(2)})$ or $\supp(\hat{\mA}^{(1)})\subseteq \supp(\hat{\mA}^{(2)})$ and $\supp(\hat{\mX}^{(1)})\subset \supp(\hat{\mX}^{(2)})$.
    \item Consider the reconstructions $\hat{\mA}$ and $\hat{\mX}$ returned by D-EBF and the associated residual $\mR = \mY - \hat{\mA} \hat{\mX}:$ there exists a permutation matrix $\mP \in \cP^{n \times n}$ such that $\hat{\mA}\mP^T = \mA$ and $\mP \hat{\mX} = \mX$ iff $\mR =\textbf{0}_{m \times N}$. As a result, $\mR = \textbf{0}_{m \times N}$ is a sufficient condition to ensure that $\mY$ has a unique, up to permutation, factorisation of the form $\mY = \mA \mX$ where $\mA\in \cE^{m \times n}_{k, \epsilon, d}$ with $\epsilon \leq 1/6$ and $\mX \in \cX^{n\times N}_k$.
\end{enumerate}}

\begin{proof}
We first prove statement 1 by induction, showing that if at the start of an iteration of the while loop $\cQ$ is true, then $\cQ$ will also be true throughout and up to the end of this iteration. This implies that $\cQ$ will also be true at the start of the next iteration, and therefore by induction we will be able to conclude that Statement 1 is true. At the start of the first iterate $\cQ$ is trivially true as this corresponds to the initialisation $\hat{\mA} = \textbf{0}_{m \times n}$, $\hat{\mX} = \textbf{0}_{n \times N}$. Assume then at any iterate $t \in [t_f]$ that $\cQ$ is true on line 6, this implies $\cQ$ is also true at the start of EXTRACT\&MATCH. This ensures that the conditions on $\hat{\mA}$ and $\hat{\mX}$ required by Lemma \ref{lemma:extract_accuracy} are satisfied and therefore the reconstructions returned by EXTRACT\&MATCH, which are passed to CLUSTER\&ADD on line 20, are also such that $\cQ$ is true. In addition, by Statement 1 of Lemma \ref{lemma:extract_accuracy}, then for all $i \in [p]$ it follows that $q_i$ and $\vw_i$ are a singleton value and partial support respectively, with $|\supp(\vw_i)|>(1-2 \epsilon)d$. Furthermore, these partial supports originate from columns of $\mA$ for which there is not yet a corresponding nonzero partial or complete column in $\hat{\mA}$. As a result, the only manner in which the reconstructions returned by CLUSTER\&ADD on line 20 can result in $\cQ$ being false is if there is a clustering error, i.e., there exists a pair of partial supports whose inner product is larger than $2 \epsilon d$ despite originating from different columns. This however would contradict Lemma \ref{corollary_sort_ps}, therefore the reconstructions returned by CLUSTER\&ADD, which are inputted into DECODE on line 23, ensure that $\cQ$ is true. This in turn implies that the conditions of Lemma \ref{lemma:accuracy_decode} are satisfied, and therefore, from Statement 1 of Lemma \ref{lemma:accuracy_decode}, it follows that the reconstructions returned after the completion of the for loop, running on lines 22-24, are also such that $\cQ$ is true. Therefore, if $\cQ$ is true at the beginning of an iteration it is also true throughout said iteration and at its end. As this implies that $\cQ$ is true at the beginning of the next iteration, and given that $\cQ$ is true at the start of the first iteration, then by induction Statement 1 must be true.

Given that Statement 1 is true, and as by inspection at no point during the run-time of D-EBF are nonzeros removed from the reconstructions, then clearly $\supp(\hat{\mA}^{(1)})\subseteq \supp(\hat{\mA}^{(2)})$ and $\supp(\hat{\mX}^{(1)})\subseteq \supp(\hat{\mX}^{(2)})$. In addition, as it is assumed that $t<t_f$, then by the construction of D-EBF an update must occur to either $\hat{\mA}$ or $\hat{\mX}$ during the iteration, therefore Statement 2 must also be true.

Finally, to prove Statement 3, clearly if there exists a $\mP \in \cP^{n \times n}$ such that $\mA = \hat{\mA} \mP^T$ and $\mX = \mP \hat{\mX}$ then $\mY = \mA \mX = \hat{\mA} \mP^T \mP\hat{\mX}$ and so $\mR= \textbf{0}_{m \times N}$. For the converse, as Statement 1 is true, then there exists a $\mP \in \cP^{n \times n}$ such that $\mV \defeq \hat{\mA} \mP^T$ satisfies $\supp(\mV)\subseteq \supp(\mA)$, $\mU \defeq \mP \hat{\mX}$ satisfies $\supp(\mU)\subseteq \supp(\mX)$ and $u_{i,j} = x_{i,j}$ for all $(i,j) \in supp(\mU)$. If $\mR= \textbf{0}_{m \times N}$ then $\mY = \hat{\mA}\hat{\mX} = \hat{\mA} \mP^T \mP \hat{\mX} = \mV \mU$. As $\mY = \mA \mX = \mV \mU$ and $\supp(\mV)\subseteq \supp(\mA)$, $\supp(\mU)\subseteq \supp(\mX)$ and $u_{i,j} = x_{i,j}$ for all $(i,j) \in supp(\mU)$, then $\mA = \mV$ and $\mX = \mU$.
\end{proof}

Analogous accuracy guarantees can be proved for ND-EBF.

\textbf{Lemma \ref{lemma:accuracy_ND-EBF} (Accuracy of ND-EBF)}\textit{ Let $\mY = \mA \mX$, where $\mA \in \cE^{m \times n}_{k, \epsilon, d}$ with $\epsilon \leq 1/6$ and $\mX \in \cX^{n\times N}_k$. Consider Algorithm \ref{alg:NDEBF}, ND-EBF: in regard to the reconstructions $\hat{\mA}$ and $\hat{\mX}$ computed at any point during the run-time of ND-EBF, we will say, for $\eta \in [n]$, that $\cQ(\eta)$ is true iff the following hold.
\begin{enumerate}[label=(\alph*)]
    \item The column vector $\hat{\va}_l$ and row vector $\tilde{\hat{\vx}}_l$ are nonzero iff $l < \eta$.
    \item If $l < \eta$ then $|\supp(\hat{\va}_l)| = d$.
    \item There exists a permutation matrix $\mP \in \cP^{n \times n}$ such that $\supp(\hat{\mA}\mP^T) \subseteq \supp(\mA)$, $\supp(\mP\hat{\mX}) \subseteq \supp(\mX)$ and $\hat{x}_{P(l),i} = x_{P(l), i}$ for all $i \in [N]$ and $l \in supp(\hat{\vx}_l)$, where $P:[n] \rightarrow [n]$ denotes the row permutation caused by pre-multiplication with $\mP$.
\end{enumerate}
 Suppose that ND-EBF exits after the completion of iteration $\eta_f \in [n]$ of the while loop, then the following statements are true.
\begin{enumerate}
    \item For all $\eta \in [\eta_f]$, then at the start of the $\eta$th iteration of the while loop $\cQ(\eta)$ is true.
    \item Consider the reconstructions $\hat{\mA}$ and $\hat{\mX}$ returned by ND-EBF: $\hat{\mA}$ and $\hat{\mX}$ always satisfy (c) and there exists a permutation matrix $\mP \in \cP^{n \times n}$ such that $\hat{\mA}\mP^T = \mA$ and $\mP \hat{\mX} = \mX$ iff $\eta_f = n$ and $\cQ(n+1)$ is true. 
    \item Consider the reconstructions $\hat{\mA}$ and $\hat{\mX}$ returned by ND-EBF and the associated residual $\mR = \mY - \hat{\mA} \hat{\mX}:$ there exists a permutation matrix $\mP \in \cP^{n \times n}$ such that $\hat{\mA}\mP^T = \mA$ and $\mP \hat{\mX} = \mX$ iff $\mR =\textbf{0}_{m \times N}$. As a result, $\mR = \textbf{0}_{m \times N}$ is a sufficient condition to ensure that $\mY$ has a unique, up to permutation, factorisation of the form $\mY = \mA \mX$, where $\mA\in \cE^{m \times n}_{k, \epsilon, d}$ with $\epsilon \leq 1/6$ and $\mX \in \cX^{n\times N}_k$.
\end{enumerate}}

\begin{proof}
We will prove Statement 1 by induction: the base case $\cQ(1)$ is trivially true at the start of the first iteration as this corresponds to the initialisation $\hat{\mA}^{(0)} = \textbf{0}_{m \times n}$ and $\hat{\mX}^{(0)} = \textbf{0}_{n \times N}$. It therefore suffices to prove for any $\eta\in[n_f-1]$, that if $\cQ(\eta)$ is true at the start of iteration $\eta$, then $\cQ(\eta+1)$ is true at the end of iteration $\eta$, as this in turn implies $\cQ(\eta+1)$ is true at the start of iteration $\eta+1$. Assume then for arbitrary $\eta \in [\eta_f-1]$ that $\cQ(\eta)$ is true at the start of iteration $\eta$, it follows that the reconstructions $\hat{\mA}$ and $\hat{\mX}$ inputted into EXTRACT\&MATCH satisfy the conditions required for Lemma \ref{lemma:extract_accuracy}. Therefore, by Lemma \ref{lemma:extract_accuracy}, for all $i \in [p]$ it follows that $q_i$ and $\vw_i$ are a singleton value and partial support respectively, $|\supp(\vw_i)|>(1-2 \epsilon)d$ and $\vw_i$ originates from a column $\va_l \neq \hat{\va}_h$ for all $h\in[n]$. Note also that as $\supp(\hat{\va}_l) = d$ for any $l < \eta$ and $\hat{\va}_l = \textbf{0}_m$ otherwise, updates to $\hat{\mX}$ during EXTRACT\&MATCH occur only on rows with a row index less than $\eta$. Furthermore, as per line 12 ND-EBF ignores updates to $\hat{\mA}$ performed by EXTRACT\&MATCH. Therefore $\cQ(\eta)$ must still be true after the completion of EXTRACT\&MATCH. Consider now the updates to the reconstructions performed by the MAXCLUSTER\&ADD subroutine on line 19 of ND-EBF. This subroutine updates only the $\eta$th column and row of $\hat{\mA}$ and $\hat{\mX}$ respectively. As $\eta<\eta_f$ then $p>0$ and $|\supp(\hat{\va}_{\eta})|=d$. Furthermore, the corresponding row of $\hat{\mX}$ is also updated by MAXCLUSTER\&ADD with $\hat{x}_{\eta,i} = q_i$ for all $i \in \cC^*$. Therefore conditions (a) and (b) of $\cQ(\eta+1)$ must be true at the completion of MAXCLUSTER\&ADD. Assume now that condition (c) is not satisfied by the reconstructions returned by MAXCLUSTER\&ADD: then $\hat{\va}_{\eta} \neq \va_h$ for all $h \in [n]$ and as a result there must have occurred a clustering error. This implies in turn that there exists a pair of columns of $\mW$ whose inner product is at least $2 \epsilon d$ despite both columns being partial supports originating from different columns of $\mA$. This however would contradict Lemma \ref{corollary_sort_ps}, therefore by contradiction condition (c) is satisfied and as a result we may conclude that at the exit of MAXCLUSTER\&ADD $\cQ(\eta+1)$ is true. As a result, the reconstructions inputted to DECODE on line 22 are such that the conditions required for Lemma \ref{lemma:accuracy_decode} are satisfied, hence by Lemma \ref{lemma:accuracy_decode} it follows that $\supp(\mP \mX) \subseteq \supp(\mX)$ and $\hat{x}_{P(l),i} = x_{P(l),i}$ for all $(l,i)\in \supp(\hat{\mX})$. As DECODE only updates the rows of $\hat{\mX}$ corresponding to the complete columns of $\hat{\mA}$, then at the end of iteration $\eta$ of the while loop of ND-EBF it follows that $\cQ(\eta+1)$ is true. This concludes the proof by induction of Statement 1.

Turning our attention to Statement 2 then to prove the reconstructions returned by ND-EBF always satisfy condition (c), observe from Statement 1 that $\cQ(\eta_f)$ is true at the start of iteration $\eta_f$. Therefore, by the same arguments used in the proof of Statement 1, the reconstructions returned by MAXCLUSTER\&ADD at the $\eta_f$th iteration satisfy condition (c). As $\eta_f$ is the final iteration of the while loop of ND-EBF, then by construction there are two possibilities: (i) $|\supp(\va_{\eta})| <d$ or (ii) $|\supp(\va_{\eta})| =d$ and $\eta_f = n$. In the case of (i), ND-EBF exits with no further updates to the reconstructions, therefore the reconstructions returned by ND-EBF must satisfy (c). In the case of (ii), as the reconstructions satisfy condition (c) and $|\supp(\hat{\va_h})| = d$ for all $h \in [n]$, then there must also exist a permutation $\mP \in \cP^{n \times n}$ such $\hat{\mA}\mP^T = \mA$. It follows that the inputs to DECODE during the $n$th iteration of the while loop satisfy the conditions required for Lemma \ref{lemma:accuracy_decode}. As $\cS = \{l \in [N]: |\supp(\va_l)|=d \} =[n]$, then $\supp(\vx_i)\subseteq P(\cS)$ for all $i \in [N]$ and, due to Statement 2 of Lemma \ref{lemma:accuracy_decode}, it therefore follows that $\mP \hat{\mX} = \mX$. In summary, in either case the reconstructions returned by ND-EBF satisfy condition (c). 

Concerning the second part of Statement 2, observe that if $\cQ(n+1)$ is true then $\eta_f = n$. Otherwise, if $\eta_f< n$ then ND-EBF exits with $|\supp(\hat{\va}_h)| < d$ for all $\eta_f < h \leq n$, and so $\cQ(n+1)$ cannot be true. Suppose there exists a permutation matrix $\mP \in \cP^{n \times n}$ such $\hat{\mA}\mP^T = \mA$ and $\mP \hat{\mX} = \mX$. By inspection it is clear that $\cQ(n+1)$ must be true which in turn implies $\eta_f=n$. The converse follows by observing that $\eta_f=n$ and $\cQ(n+1)$ being true imply that case (ii) must occur, which we have already shown implies the existence of a permutation $\mP \in \cP^{n \times n}$ such $\hat{\mA}\mP^T = \mA$ and $\mP \hat{\mX} = \mX$.

The proof of Statement 3 follows in the same manner as the proof of Statement 3 in Lemma \ref{lemma:accuracy_DEBF}. Clearly, if there exists a $\mP \in \cP^{n \times n}$ such that $\mA = \hat{\mA} \mP^T$ and $\mX = \mP \hat{\mX}$ then $\mY = \hat{\mA} \mP^T \mP\hat{\mX} = \hat{\mA}\hat{\mX}$ and therefore $\mR= \mY - \hat{\mA} \hat{\mX}=\textbf{0}_{m \times N}$. For the converse, as Statement 2 is true then there exists a $\mP \in \cP^{n \times n}$ such that $\mV \defeq \hat{\mA} \mP^T$ satisfies $\supp(\mV)\subseteq \supp(\mA)$, $\mU \defeq \mP \hat{\mX}$ satisfies $\supp(\mU)\subseteq \supp(\mX)$ and $u_{i,j} = x_{i,j}$ for all $(i,j) \in supp(\mU)$. As $ \mA \mX = \hat{\mA} \mP^T \mP \hat{\mX} =\mV \mU$, $\supp(\mV)\subseteq \supp(\mA)$, $\supp(\mU)\subseteq \supp(\mX)$ and $u_{i,j} = x_{i,j}$ for all $(i,j) \in supp(\mU)$, then $\mA = \mV$ and $\mX = \mU$.
\end{proof}

\section{Supporting lemmas for Theorem \ref{theorem:main}} \label{appendix:supporting_lemmas}

The following corollary highlights that it suffices to recover the encoder up to permutation.

\textbf{Corollary \ref{corollary:recovery_A_sufficient} (Recovery of $\mA$ up to permutation is sufficient to recover both factors up to permutation)} \textit{Let $\mY = \mA \mX$, where $\mA \in \cE^{m \times n}_{k, \epsilon, d}$ with $\epsilon \leq 1/6$ and $\mX \in \cX^{n\times N}_k$. Suppose that $\hat{\mA}$ and $\hat{\mX}$ are the reconstructions of $\mA$ and $\mX$ returned by either ND-EBF or D-EBF. If there there exists a $\mP \in \cP^{n \times n}$ such that $\hat{\mA}\mP^T = \mA$ then $\mP \hat{\mX} = \mX$.}

\begin{proof}
In the case of ND-EBF this result follows directly from Statement 2 of Lemma \ref{lemma:accuracy_ND-EBF}. For D-EBF, observe that the last update to $\hat{\mA}$ occurs on line 19 during the final iteration of the while loop prior to the algorithm terminating. Therefore, if there exists a $\mP \in \cP^{n \times n}$ such that the returned reconstruction of $\mA$ satisfies $\hat{\mA}\mP^T = \mA$, then from line 20 and onward it must likewise hold that $\hat{\mA}\mP^T = \mA$. Therefore, at the input to DECODE on line 23 it follows that $\cS = \{l \in [n]: \hat{\va}_l=d \} = [n]$ and as a result $\supp(\vx_i)\subset P(\cS)$ for all $i \in [N]$. In addition, due to Statement 1 of Lemma \ref{lemma:accuracy_DEBF}, it follows that the input reconstructions to DECODE on line 23 satisfy the conditions of Lemma \ref{lemma:accuracy_decode}. As a result, by Statement 2 of Lemma \ref{lemma:accuracy_decode}, the sparse code $\hat{\mX}$ returned by DECODE satisfies $\mP \hat{\mX} = \mX$. This concludes the proof of the corollary.
\end{proof}

We now proceed to the derivation of Lemma \ref{lemma:NDEBF_lower_bounds_DEBF},
to prove this result however we require the following result.

\begin{lemma} \label{lemma:decode_sub_support_result}
Let $\vy = \mA \vx$ where $\mA \in \cE^{m \times n}_{k, \epsilon, d}$ with $\epsilon \leq 1/6$ and $\vx \in \cX^{n}_k$. For $b \in \{1,2\}$, let $\hat{\mA}^{(b)}$ and $\hat{\vx}^{(b)}$ be such that there exists a $\mP_b \in \cP^{n \times n}$  $\supp(\hat{\mA}^{(b)}\mP_b^T) \subseteq \supp(\mA)$ and $\hat{x}_{P_b(l)}^{(b)} = x_{P_b(l)}$ for all $l \in \supp(\hat{\vx}^{(b)})$, where $P_b:[n] \rightarrow [n]$ is the row permutation corresponding to pre-multiplication by $\mP_b$. In addition, assume that $\supp(\hat{\mA}^{(1)}\mP_1^T) \subseteq \supp(\hat{\mA}^{(2)}\mP_2^T) $ and $\supp(\mP_1 \hat{\vx}^{(1)}) \subseteq \supp(\mP_2 \hat{\vx}^{(2)})$. Considering the sparse codes returned by DECODE,
\[
    \hat{\vx}^{(b)},\_ \leftarrow \mathrm{DECODE}(\vy,m,\hat{\mA}^{(b)},\hat{\vx}^{(b)}, \epsilon, d)
\]
then $\supp(\hat{\vx}^{(1)}) \subseteq \supp(\hat{\vx}^{(2)})$.
\end{lemma}

\begin{proof}
For typographical ease let $\mV^{(1)} = \hat{\mA}^{(1)}\mP_1^T$, $\vu^{(1)} =\mP_1 \hat{\vx}^{(1)}$, $\mV^{(2)} = \hat{\mA}^{(2)}\mP_2^T$ and $\vu^{(2)} =\mP_2 \hat{\vx}^{(2)}$. Observe for $b \in \{1,2\}$ that if
\begin{equation} \label{eq:decode_permuted}
    \vu^{(b)},\_ \leftarrow \mathrm{DECODE}(\vy,m,\mV^{(b)},\vu^{(b)}, \epsilon, d)
\end{equation}
then $\vu^{(b)} = \mP_b \hat{\vx}^{(b)}$. Given equation \ref{eq:decode_permuted} it therefore suffices to prove $\supp(\vu^{(1)}) \subseteq \supp(\vu^{(2)})$. Let $\cQ(t)$ be true iff at the end of iteration $t$ of the while loop, starting on line 4 of DECODE, it holds that $\supp(\vu^{(1)}) \subseteq \supp(\vu^{(2)})$. We proceed by induction: observe that $\cQ(0)$, which we may interpret as the event that at the start of the first iteration $\supp(\vu^{(1)}) \subseteq \supp(\vu^{(2)})$, is trivially true by the assumptions of the lemma. Therefore it suffices to prove that if $\cQ(t-1)$ is true then $\cQ(t)$ is true.

Assume $\cQ(t-1)$ is true, analysing the residual computed on line 6 of DECODE it follows that
\[
\begin{aligned}
\vr^{(1)} &= \mA \vx - \mV^{(1)}\vu^{(1)},\\
\vr^{(2)} &= \mA \vx - \mV^{(2)}\vu^{(2)}.
\end{aligned}
\]
Define, as in the proof of Lemma \ref{lemma:adjmat}, for $j \in [m]$ the sets $\Omega_j \defeq \supp(\tilde{\va}_j)\cap \supp(\vx)$, $\Gamma_j^{(1)} \defeq \supp(\tilde{\vv}_j^{(1)})\cap \supp(\vu^{(1)})$ and $\Gamma_j^{(2)} \defeq \supp(\tilde{\vv}_j^{(2)})\cap \supp(\vu^{(2)})$. It follows for any entry $j \in [m]$ that
\[
\begin{aligned}
r_{j}^{(1)} &= \sum_{l \in \Omega_j}x_l - \sum_{h \in \Gamma_j^{(1)}}x_h = \sum_{l \in \Omega_j \backslash \Gamma_j^{(1)}}x_l \\
r_{j}^{(2)} &= \sum_{l \in \Omega_j}x_l - \sum_{h \in \Gamma_j^{(2)}}x_h = \sum_{l \in \Omega_j \backslash \Gamma_j^{(2)}}x_l 
\end{aligned}
\]
As $\Gamma_j^{(1)} \subseteq \Gamma_j^{(2)}$, then $\Omega_j \backslash \Gamma_j^{(2)} \subseteq \Omega_j \backslash \Gamma_j^{(1)}$. Therefore any entry $r_{j}^{(2)}$ is a sum over a subset of the nonzeros in $\vx_i$ used in the sum of $r_{j,i}^{(1)}$. Consider any row index $\kappa\in [n]$ such that $u_{\kappa}^{(1)} = 0$ at the start of iteration $t$ and $u_{\kappa}^{(1)} = x_{\kappa}$ at the end of iteration $t$.  This implies that $f(x_{\kappa}, \vr^{(1)})\geq 2 \epsilon d$ and therefore $\vv_{\kappa}^{(1)} = \vv_\kappa^{(2)} = \va_{\kappa}$. If $u_{\kappa}^{(2)} = x_{\kappa}$ at the start of iteration $t$ then then clearly $u_{\kappa}^{(2)} = x_{\kappa}$ at the end of iteration $t$. If $u_{\kappa}^{(2)} =0$ at the start of iteration $t$, then $\kappa \in \Omega_j \backslash \Gamma^{(2)}$ for all $j \in \supp(\va_{\kappa})$. In addition, as $\Omega_j \backslash \Gamma^{(2)} \subseteq \Omega_j \backslash \Gamma^{(1)} $ then $f(x_{\kappa}, \vr^{(2)})\geq f(x_{\kappa}, \vr^{(1)})\geq 2 \epsilon d$, and therefore line 24 of DECODE ensures $u_{\kappa}^{(2)} = x_{\kappa}$ at the end of iteration $t$. Therefore we have proved, assuming $\cQ(t-1)$ is true, that if $u_{\kappa}^{(1)} = 0$ at the start of iteration $t$ and $u_{\kappa}^{(1)} = x_{\kappa}$ at the end of iteration $t$, then $u_{\kappa}^{(2)} = x_{\kappa}$ at the end of iteration $t$. Therefore $\supp(\vu^{(1)}) \subseteq \supp(\vu^{(2)})$ at the end of iteration $t$ and so $\cQ(t)$ is true. Given that we have proved both the base case and the inductive step the proof is complete.
\end{proof}

With this result in place we may now proceed to the proof of Lemma \ref{lemma:NDEBF_lower_bounds_DEBF}.

\textbf{Lemma \ref{lemma:NDEBF_lower_bounds_DEBF} (If ND-EBF successfully computes the matrix factors up to permutation then so will D-EBF)}\textit{ Let $\mY = \mA \mX$ where $\mA \in \cE^{m \times n}_{k, \epsilon, d}$ with $\epsilon \leq 1/6$ and $\mX \in \cX^{n\times N}_k$. Let $\hat{\mA}^{(1)}$ and $\hat{\mX}^{(1)}$ denote the reconstructions returned by ND-EBF and $\hat{\mA}^{(2)}$ and $\hat{\mX}^{(2)}$ the reconstructions returned by D-EBF. If there exists a $\mP_1 \in \cP^{n \times n}$ such that $\hat{\mA}^{(1)}\mP_1^T = \mA$ and $\mP_1 \hat{\mX}^{(1)} = \mX$, then there exists a $\mP_2 \in \cP^{n \times n}$ such that $\hat{\mA}^{(2)}\mP_2^T = \mA$ and $\mP_2 \hat{\mX}^{(2)} = \mX$.}

\begin{proof}
Before proceeding we emphasise certain points concerning the notation used in this proof. In particular, in what follows we use $\hat{\mA}^{(1)}$ and $\hat{\mX}^{(1)}$ to denote the reconstructions computed any point during the run-time of ND-EBF. Likewise, we use $\hat{\mA}^{(2)}$ and $\hat{\mX}^{(2)}$ to denote the reconstructions computed at any point during the run-time of D-EBF. In addition, let $\ell(\eta)$, where $\ell: [n] \rightarrow [n]$, denote the column of $\mA$ recovered by ND-EBF during iteration $\eta$, i.e., $\hat{\va}_{\eta}^{(1)} = \va_{\ell(\eta)}$. Furthermore, let $\tau(\eta)$, where $\tau: [n] \rightarrow [n]$, denote the index of the column reconstruction of $\va_{\ell(\eta)}$ in $\hat{\mA}^{(2)}$. Note that as a result of Lemma \ref{lemma:accuracy_DEBF}, at any point during the run-time of D-EBF either $\hat{\va}_{\tau(\eta)}^{(2)} = \textbf{0}_m$ or $\va_{\ell(\eta)}^T\hat{\va}_{\tau(\eta)}^{(2)}>(1-2\epsilon)d$. Finally, let $\cQ(\eta)$ be true iff at the end of iteration $\eta$ of the while loop of D-EBF $\hat{\va}_{\tau(l)}^{(2)} = \hat{\va}_l^{(1)}$ for all $l \in [\eta]$, and for any  $l \in [\eta]$,  $\kappa \in [N]$ such that $\hat{x}_{l,\kappa}^{(1)} \neq 0$, then $\hat{x}_{\tau(l),\kappa}^{(2)} = \hat{x}_{l,\kappa}^{(1)}$. 

It suffices to prove that if ND-EBF recovers $\va_{\ell(\eta)}$ by the end of $\eta$th iteration of its while loop, then D-EBF also recovers $\va_{\ell(\eta)}$ by the end of $\eta$th iteration of its while loop. Indeed, if this is true then if ND-EBF recovers $\mA$ up to permutation then so to will D-EBF. The result claimed then follows from Corollary \ref{corollary:recovery_A_sufficient}.
Our task therefore is to prove for all $\eta \in [n]$ that $\cQ(\eta)$ is true. We proceed by induction. Observe that $\cQ(0)$ is trivially true as both algorithms are initialised with $\hat{\mA}^{(1)} = \hat{\mA}^{(2)} = \textbf{0}_{m \times n}$ and $\hat{\mX}^{(1)} = \hat{\mX}^{(2)} = \textbf{0}_{n \times N}$. It suffices then to show for $\eta \in [n]$ that $\cQ(\eta-1)$ true implies $\cQ(\eta)$ true. Under the assumption that $\cQ(\eta-1)$ is true, we now analyse the residuals computed at the beginning of iteration $\eta$ of the respective while loops of ND-EBF and D-EBF. From Statement 1 of Lemma \ref{lemma:accuracy_ND-EBF}, then there exists a $\mP_1 \in \cP^{n \times n}$ such that $\mV^{(1)} \defeq \hat{\mA}^{(1)} \mP_1^T$ satisfies $\supp(\mV^{(1)})\subseteq \supp(\mA)$, $\mU^{(1)} \defeq \mP_1 \hat{\mX}^{(1)}$ satisfies $\supp(\mU^{(1)})\subseteq \supp(\mX)$, and $u_{i,j}^{(1)} = x_{i,j}$ for all $(i,j) \in \supp(\mU^{(1)})$. Likewise, from Statement 1 of Lemma \ref{lemma:accuracy_DEBF}, there exists a $\mP_2 \in \cP^{n \times n}$ such that $\mV^{(2)} \defeq \hat{\mA}^{(2)} \mP_2^T$ satisfies $\supp(\mV^{(2)})\subseteq \supp(\mA)$, $\mU^{(2)} \defeq \mP_2 \hat{\mX}^{(2)}$ satisfies $\supp(\mU^{(2)})\subseteq \supp(\mX)$, and $u_{i,j}^{(2)} = x_{i,j}$ for all $(i,j) \in \supp(\mU^{(2)})$. From Lemma \ref{lemma:accuracy_decode}, for any $l\geq \eta$ then $\vv^{1}_{P_1(l)}$ and $\tilde{\vu}^{1}_{P_1(l)}$ are zero vectors, therefore as $\cQ(\eta-1)$ is true $\supp(\mV^{(1)}) \subseteq \supp(\mV^{(2)})$ and $\supp(\mU^{(1)}) \subseteq \supp(\mU^{(2)})$. The respective residuals inputted to EXTRACT\&MATCH for each algorithm can be written as
\[
\begin{aligned}
\mR^{(1)} &= \mA \mX - \mV^{(1)}\mU^{(1)},\\
\mR^{(2)} &= \mA \mX - \mV^{(2)}\mU^{(2)}.
\end{aligned}
\]
Considering an arbitrary column index $i \in [N]$ of the residuals, define, as in the proof of Lemma \ref{lemma:adjmat}, $\Omega_j \defeq \supp(\tilde{\va}_j)\cap \supp(\vx_i)$, $\Gamma_j^{(1)} \defeq \supp(\tilde{\vv}_j^{(1)})\cap \supp(\vu_i^{(1)})$ and $\Gamma_j^{(2)} \defeq \supp(\tilde{\vv}_j^{(2)})\cap \supp(\vu_i^{(2)})$, observing that $\Gamma_j^{(1)} \subseteq \Gamma_j^{(2)}$. The entry $(j,i) \in [m] \times [N]$ in each of these residuals can therefore be written as
\[
\begin{aligned}
r_{j,i}^{(1)} &= \sum_{l \in \Omega_j}x_l - \sum_{h \in \Gamma_j^{(1)}}x_h = \sum_{l \in \Omega_j \backslash \Gamma_j^{(1)}}x_l, \\
r_{j,i}^{(2)} &= \sum_{l \in \Omega_j}x_l - \sum_{h \in \Gamma_j^{(2)}}x_h = \sum_{l \in \Omega_j \backslash \Gamma_j^{(2)}}x_l 
\end{aligned}
\]
As $\Gamma_j^{(1)} \subseteq \Gamma_j^{(2)}$, then $\Omega_j \backslash \Gamma_j^{(2)} \subseteq \Omega_j \backslash \Gamma_j^{(1)}$. Therefore, any entry $r_{j,i}^{(2)}$ is a sum over a subset of the nonzeros in $\vx_i$ used in the sum of $r_{j,i}^{(1)}$. 

Suppose that at the end of iteration $\eta$ of the while loop of D-EBF it holds that $\hat{\va}_{\tau(\eta)}^{(2)} = \va_{\ell(\eta)}$. Letting $\cS_1 \defeq \{l\in [n]: \supp(\hat{\va}_l^{(1)}) =d \}$ and $\cS_2 \defeq \{l\in [n]: \supp(\hat{\va}_l^{(2)}) =d \}$, then, as we are also assuming that $\cQ(\eta-1)$ is true, it must follow that $\cS_1 \subseteq \cS_2$. As a result, by Lemma \ref{lemma:decode_sub_support_result}, if for some column index $\kappa \in [N]$ it holds that $\hat{x}_{\eta, \kappa}^{(1)} = x_{\ell(\eta), \kappa}$, then $\hat{x}_{\tau(\eta), \kappa}^{(2)} = x_{\ell(\eta), \kappa}$ and as a result $\cQ(\eta)$ must be true. Therefore, to prove the inductive step it suffices to show that $\hat{\va}_{\tau(\eta)}^{(2)} = \va_{\ell(\eta)}$ at the end of iteration $\eta$.

To this end, consider then an arbitrary $\gamma \in \supp(\va_{\ell(\eta)})$. By the construction of ND-EBF and the assumption that ND-EBF recovers $\mA$ up to permutation, then $\hat{a}_{\gamma, \eta}^{(1)} = 0$ at the start of iteration $\eta$ and $\hat{a}_{\gamma, \eta}^{(1)} = 1$ at the end of iteration $\eta$. Therefore, considering the residual at the start of iteration $\eta$, there must exist a $\kappa \in [N]$ such that $f(x_{\ell(\eta),\kappa}, \vr_{\kappa}^{(1)})>(1-2\epsilon)d$ and $r_{\gamma, \kappa}^{(1)} = x_{\ell(\eta),\kappa}$. Recall now that $r_{j, \kappa}^{(2)}$ is a sum over a subset of the nonzeros in $\vx_{\kappa}$ used in the sum of $r_{j,\kappa}^{(1)}$, therefore either $r_{j,\kappa}^{(2)} = 0$ or $r_{j,\kappa}^{(2)} = x_{\ell(\eta),\kappa}$. If $r_{j,\kappa}^{(2)} = 0$ then it must already be true at the beginning of iteration $\eta$ that $\hat{\va}_{\gamma, \tau(\eta)}^{(2)} = 1$. If at the beginning of iteration $\eta$ it holds that $\hat{x}_{\tau(\eta), \kappa}^{(2)} = x_{\ell(\eta),\kappa}$ and $r_{j,\kappa}^{(2)} = x_{\ell(\eta),\kappa}$, then the for loop on lines 30-35 of EXTRACT\&MATCH ensures by the end of iteration $\eta$ that $\hat{\va}_{\gamma, \tau(\eta)}^{(2)} =1$. If at the beginning of iteration $\eta$ it holds instead that $\hat{x}_{\tau(\eta), \kappa}^{(2)} = 0$ and $r_{j,\kappa}^{(2)} = x_{\ell(\eta),\kappa}$, then $\ell(\eta) \in \Omega_p \backslash \Gamma_p$ for all $p \in \supp(\va_{\ell(\eta)})$. Therefore, as $f(x_{\ell(\eta),\kappa}, \vr_{\kappa}^{(1)})>(1-2\epsilon)d$ then $f(x_{\ell(\eta),\kappa}, \vr_{\kappa}^{(2)})>(1-2\epsilon)d$. Therefore during the subroutine EXTRACT\&MATCH on line 13 of D-EBF, the partial support $\vw = g(x_{\ell(\eta),\kappa}, \vr_{\kappa}^{(2)})$, which satisfies $\gamma \in \supp(\vw)$ and $|\supp(\vw)|>(1-2\epsilon)d$, is extracted. During the CLUSTER\&ADD subroutine on line 20 then by Lemma \ref{lemma:accuracy_DEBF} it follows that $\hat{\va}_{\tau(\eta)}^{(2)} \leftarrow \sigma (\hat{\va}_{\tau(\eta)}^{(2)} + \vw)$. Therefore by the end of iteration $\eta$ $\hat{a}_{\gamma,\tau(\eta)}^{(2)}=1$. As $\gamma \in \supp(\va_{\ell(\kappa)})$ was arbitrary, then it follows that $\hat{\va}_{\tau(\eta)}^{(2)}= \va_{\ell(\kappa)}$ at the end of iteration $\eta$. Therefore, if $\cQ(\eta-1)$ is true then $\cQ(\eta)$ is true. This concludes the proof of the lemma.
\end{proof}

The third lemma we prove in this appendix provides a lower bound on the probability that the random matrix $X$, as defined in Definition \ref{def_data_model}, has at least $\beta \in \ints_{\geq 0}$ nonzeros per row. This result plays a key role in the proof of Theorem \ref{theorem:main}.

\textbf{Lemma \ref{lemma:N_lb} (Nonzeros per row in $X$)} \label{lemma_N_lb}
\textit{For some $\beta \in \naturals$ and $\mu>1$, if $N \geq \beta \left(\mu  \frac{n}{k} \ln(n)+1 \right)$ then the probability that the random matrix $X$, as defined in the PSB model Definition \ref{def_data_model}, has at least $\beta$ nonzeros per row is more than $\left( 1 - n^{-(\mu -1)} \right)^{\beta}$.}

\begin{proof}
Let $X(k, \beta, r^c)$ denote a random binary matrix constructed by concatenating i.i.d. random, binary column vectors whose supports are drawn uniformly at random across all possible supports with a fixed cardinality $k$, until there are at least $\beta$ non-zeros per row. Here $r^c$ indicates that the elements of the supports of each column are chosen without replacement. Let then $X(k, \beta, r)$ denote a random binary matrix constructed by concatenating i.i.d. random, binary column vectors with replacement, i.e., the support of each column is generated by taking $k$ i.i.d samples uniform from $[n]$ with replacement, until there are at least $\beta$ non-zeros per row. Let $\eta [X(k, \beta, r^c)]$ be the random variable which counts the number of columns of $X(k, \beta, r^c)$, note here that we adopt square brackets purely for typographical clarity. Also note that the supports of the columns of the random matrix $X$ as defined in Definition \ref{def_data_model} are drawn in the same manner as those of the random matrix $X(k, \beta, r^c)$. The  difference between these random matrices is that the number of columns $N$ of $X$ is fixed in advance instead of, as is the case for $X(k, \beta, r^c)$, drawing enough columns until there are at least $\beta$ nonzeros per row. Therefore
	\[
	\prob \left( \text{``X has at least $\beta$ non-zeros per row"}\right) = \prob \left(\eta [X(k, \beta, r^c)] \leq N\right).
	\]
Or goal then is to lower bounding $\prob \left(\eta [X(k, \beta, r^c)] \leq N\right)$. To this end we first note that this problem can be interpreted as a generalisation of the coupon collector and dixie cup problems \cite{article,2014arXiv1412.3626D,coupon_collec_group}, which correspond to analysing the expectations of $\eta[X(1, 1, r^c)]$ and $\eta[X(\beta, 1, r^c)]$ respectively. For context, in the classic coupon collector problem, objects are drawn one at a time from a set of size $n$, typically uniformly at random with replacement until each object has been seen once. We again emphasise that in the notation adopted here $r^c$ refers to sampling the entries in the support of a given column without replacement, not the supports of the columns themselves which are mutually independent by construction. The generalisation studied here is equivalent to analysing the number of draws of $k$ objects from $[n]$, uniform at random with replacement \footnote{Recall that $r^c$ refers to sampling without replacement the entries in a given column, the supports of different columns are sampled with replacement.}, until each object has been picked at least $\beta$ times with high probability. As the distribution of $\eta [X(k, \beta, r^c)]$ is challenging to analyse directly, we instead bound the quantity of interest using the classic coupon collector distribution, for which Chernoff style bounds can more easily be derived.
	
Using the notation described above, let $\left( X^{(i)}(k, 1, r^c) \right)_{i=1}^\beta$ be mutually independent and identically distributed random matrices and define
	\[
		\tilde{X}(k, \beta, r^c) \defeq [ X^{(1)}(k, 1, r^c) , \;  X^{(2)}(k, 1, r^c) ... \;  X^{(\beta)}(k, 1, r^c)   ].
	\]
As each of the submatrices $X^{(i)}(k, 1, r^c)$ may contain rows with more than 1 nonzero in them then clearly it holds that $\eta [X(k, \beta, r^c)] \leq  \eta [\tilde{X}(k, \beta, r^c)] $, and therefore
	\[
		\prob(\eta [X(k, \beta, r^c)] \leq N) \geq \prob(\eta [\tilde{X}(k, \beta, r^c)] \leq N).
	\]
	 Observe that
	\[
	\begin{aligned}
	\bigcap_{i=1}^{\beta} \{\eta [X^{(i)}(k, 1, r^c)] \leq \frac{N}{\beta} \} &\implies \{\eta [\tilde{X}(k, \beta, r^c)]\leq N\},\\
	\{\eta [\tilde{X}(k, \beta, r^c)]\leq N\} & \centernot \implies  \bigcap_{i=1}^{\beta} \{\eta [X^{(i)}(k, 1, r^c)]  \leq \frac{N}{\beta} \},
	\end{aligned}
	\]
	therefore $\bigcap_{i=1}^{\beta} \{\eta [X^{(i)}(k, 1, r^c)] \leq \frac{N}{\beta} \} \subseteq \{\eta [\tilde{X}(k, \beta, r^c)]\leq N\}$. As a result
	\[
	\begin{aligned}
	\prob(\eta [X(k, \beta, r^c)] \leq N) &\geq \prob(\eta [\tilde{X}(k, \beta, r^c)] \leq N)\\
	& \geq \prob \left( \bigcap_{i=1}^{\beta} \{\eta [X^{(i)}(k, 1, r^c)] \leq \frac{N}{\beta} \} \right)\\
	& = \prod_{i=1}^{\beta} \prob\left( \eta [X^{(i)}(k, 1, r^c)] \leq \frac{N}{\beta} \right)\\
	& = \left[ \prob\left( \eta [X(k, 1, r^c)] \leq \frac{N}{\beta} \right) \right]^{\beta}\\
	\end{aligned}
	\]
where the equalities on the third and fourth lines follow from the assumptions of mutual independence and identical distribution respectively. If the support of a column is sampled with replacement then this will potentially decrease its cardinality relative to if it were sampled without replacement. It follows that
	\[
	\prob\left( \eta [X(k, 1, r^c)] \leq \frac{N}{\beta} \right) \geq \prob\left( \eta [X(k, 1, r)] \leq \frac{N}{\beta} \right).
	\]
Observe that sampling $k$ elements of the support of a column with replacement is equivalent to unioning the supports of $k$ mutually independent, identically distributed column vectors, each with only one element in their support, drawn uniformly at random across all $n$ possible locations. To this end, consider the random matrix $X(1,1)$, where the $r$ argument has been dropped as sampling with or without replacement when there is only one nonzero per column are equivalent. The columns of this matrix are i.i.d. random vectors of dimension $n$ with a single nonzero, whose location is drawn uniformly at random. Furthermore, sufficiently many of these columns are drawn and concatenated so as to ensure that $X(1,1)$ has at least one nonzero per row. Denote the $l$th column of $X(1,1)$ as $Z_l$. Let $X'$ be a random binary matrix which is a deterministic function of $X(1,1)$ with $\supp(X_l') \defeq \bigcup_{i = (l-1)k+1}^{lk} \supp(Z_i)$ for all $l \in [\ceil{\eta[X(1,1)]/k}]$ \footnote{Note that if $\eta [X(1,1)]$ is not a multiple of $k$, then additional columns can be drawn so that the support of the last column of $X'$ is still the union of the supports of $k$ columns.}. Therefore, given that for any $x \in \reals$ it holds that $\ceil{x} \leq x + 1$, then
	\[
	\begin{aligned}
		\prob\left( \eta [X(k, 1, r^c)] \leq \frac{N}{\beta} \right) &\geq \prob\left( \eta [X(k, 1, r)] \leq \frac{N}{\beta} \right)\\
		& = \prob\left( \eta [X'] \leq \frac{N}{\beta} \right)\\
		& = \prob\left( \ceil{\eta [X(1,1)]/k} \leq \frac{N}{\beta} \right)\\
		& \geq \prob\left( \eta [X(1,1)] \leq k\left(\frac{N}{\beta} - 1 \right)\right).
	\end{aligned}
	\]
By inspection $\eta [X(1,1)]$ is equivalent to the number of draws required to see each coupon at least once in the classic coupon collector problem, for which the following bound with $t>0$ is well known (see e.g., \cite{Doerr2020}), 
		\[
			\prob ( \eta [X(1,1)] > n \ln(n) + tn ) < e^{-t}.
		\]
\noindent Letting $k\left(\frac{N}{\beta} - 1 \right) = n \ln(n) + tn $, then rearranging gives $t = \frac{k(N-\beta) - \beta n \ln(n)}{\beta n}$. Assuming $n >1$, then for some $\mu > 1$ if $N = \beta \left(\mu  \frac{n}{k} \ln(n)+1 \right)$, then $t = (\mu -1) \ln(n)$ and therefore
			\[
			\begin{aligned}
			\prob\left( \eta [X(1,1)] \leq k\left(\frac{N}{\beta} - 1 \right)\right) & = 1 - \prob\left( \eta [X(1,1)] > k\left(\frac{N}{\beta} - 1 \right)\right)\\
			&> 1 - e^{-(\mu -1)  \ln(n)}\\
			& = 1 - n^{-(\mu -1)}.
			\end{aligned}
			\]
Therefore, if $X$ has $N \geq \beta \left(\mu  \frac{n}{k} \ln(n)+1 \right)$ columns whose supports are sampled as described in Definition \ref{def_data_model}, then
		\[
		\begin{aligned}
		\prob \left( \text{``X has at least $\beta$ non-zeros per row"}\right)
		& > \left( 1 - n^{-(\mu -1)} \right)^{\beta}
		\end{aligned}
		\]
		as claimed.
\end{proof}

The final lemma we provide in this appendix lower bounds the probability that at each iteration of its while loop ND-EBF is able to recover a column of $\mA$. This lemma also plays a key role in the proof of Theorem \ref{theorem:main}. 

\textbf{Lemma \ref{lemma:recon_from_L} (Probability that ND-EBF recovers a column at iteration $\eta$)}\textit{ Let $Y = AX$ as per the PSB model, detailed in Definition \ref{def_data_model}. In addition, assume that $A \in \cE_{k, \epsilon, d}^{m \times n}$ with $\epsilon \leq 1/6$, and that each row of $X$ has at least $\beta(n) = (1+2\epsilon d) L(n)$ nonzeros, where $L:\naturals \rightarrow \naturals$. Suppose ND-EBF, Algorithm \ref{alg:NDEBF}, is deployed to try and compute the factors of $Y$ and that the algorithm reaches iteration $\eta \in [n]$ of the while loop. Then there is a unique column $A_{\ell(\eta)}$ of $A$, satisfying $A_{\ell(\eta)} \neq \hat{A}_{l}$ for all $l \in [\eta-1]$, such that
\[
\begin{aligned}
\prob(\hat{A}_{\eta} = A_{\ell(\eta)}) > 1 - d e^{- \tau(n) L(n)},
\end{aligned}
\]
where $\tau(n) \defeq - \ln \left( 1 - \left(1 - \frac{b}{\alpha_1 n} \right)^{\alpha_1 n}\right)$ is  $\cO(1)$.
} 

\begin{proof}
Consider the start of iteration $\eta$ of the while loop of ND-EBF, then from Lemma \ref{lemma:accuracy_ND-EBF} the following must hold: $\hat{A}_l$ and $\tilde{\hat{X}}_l$ are nonzero iff $l < \eta$, if $l < \eta$ then $|\supp(\hat{A}_l)| = d$, there exists a random permutation matrix $P \in \cP^{n \times n}$ such that $\supp(\hat{A} P^T) \subseteq \supp(A)$, $\supp(P\hat{X}) \subseteq \supp(X)$ and $\hat{X}_{P(l),i} = X_{P(l), i}$ for all $i \in [N]$ and $l \in supp(\hat{X}_l)$, where $P:[n] \rightarrow [n]$ denotes the row permutation caused by pre-multiplication with the random matrix $P$. With $V \defeq \hat{A} P^T$ and $U \defeq P \hat{X}$ then by the construction of ND-EBF the residual can be expressed as
	\[
	\begin{aligned}
	R = Y - \hat{A}\hat{X}
	= AX - VU
	= A(X-U).
	\end{aligned}
	\]
Defining $Z \defeq X-U$, as $X \in \cX^{n \times N}_k$ and $U_{l,i} = X_{l,i}$ for all $(l,i) \in \supp(U)$, then $Z \in \cX^{n \times N}_k$, $\supp(Z)\subseteq \supp(X)$ and $Z_{l,i} = X_{l,i}$ for all $(l,i) \in \supp(Z)$. Consider now a column $i \in [N]$ of the residual, inputted to EXTRACT\&MATCH on line 12 of ND-EBF. By the construction of DECODE as well as Lemma \ref{lemma:accuracy_decode}, any partial support extracted from $R_i$ must originate from a column $A_h$ such that $A_h \neq \hat{A}_l$ for all $l \in [\eta-1]$, otherwise this partial support would have been extracted by DECODE during iteration $\eta-1$. Let
	\[
	\cT_i \defeq \{\alpha \in \reals: f(\alpha, R_i)> (1-2 \epsilon) d \} 
	\]
	denote the set of singleton values appearing more than $(1-2 \epsilon)d$ extracted from $R_i$ during EXTRACT\&MATCH. Then using Lemma \ref{lemma_existence_singletons} we can lower bound the number of partial supports extracted across all columns of the residual at iteration $\eta$ as follows.
	\[
	\begin{aligned}
	\sum_{i\in [N]} |\cT_i| &> \frac{1}{(1+2 \epsilon)d} \sum_{i = 1}^N |supp(Z_i)|\\
	& =\frac{1}{(1+2 \epsilon)d} \sum_{ j =1}^n |supp(\tilde{Z}_j)|\\
	& \geq \frac{1}{(1+2 \epsilon)d} \sum_{ j = \eta}^n |supp(\tilde{X}_j)|\\
	& \geq \frac{(n - \eta + 1)\beta(n)}{(1+2 \epsilon)d}\\
	& = (n - \eta + 1) L(n).
	\end{aligned}
	\] 
Th inequality on the first line follows directly from Lemma \ref{lemma_existence_singletons}. The equality on the the second line is clearly true as $\sum_{i = 1}^N |supp(Z_i)|= |\supp(Z)| = \sum_{ j =1}^n |supp(\tilde{Z}_j)|$. The inequality on the third line follows from Lemma \ref{lemma:accuracy_ND-EBF}. The fourth inequality and final equality on line 5 arise as a result of the fact that we are conditioning on there being at least $\beta(n) = (1+2\epsilon)d L(n)$ nonzeros per row of $X$. As there are more than $(n - \eta + 1) L(n)$ partial supports, all of which belong to one of the $n-\eta+1$ columns of $A$ not yet recovered, then by the pigeon hole principle there exists at least one cluster of partial supports with cardinality at least $L(n)$. Therefore, on line 17 of MAXCLUSTER\&ADD it must follow that $|\cC^*| \geq L(n)$. 

For clarity, let the index of each partial support in $C^*$ correspond to the column of $Y$ from which it originates, for example, $W_i$ is extracted from $R_i$ and the corresponding column of $X$ is therefore $X_i$. Note that as $\epsilon \leq 1/6$ then each column of $Y$ can only provide a maximum of one partial support with cardinality at least $2 \epsilon d$ per column of $A$. Therefore this choice of indexing does not result in any ambiguity. Line 18 of MAXCLUSTER\&ADD can be written equivalently as
	\[
	\hat{A}_{\eta} = \sigma \left(\sum_{i \in C^*}W_i \right).
	\]
By Lemmas \ref{corollary_sort_ps} and \ref{lemma:accuracy_ND-EBF} there exists a unique column $A_{\ell(\eta)}$ such that $\supp(\hat{A}_{\eta}) \subseteq \supp(A_{\ell(\eta)})$. We now proceed to upper bound the probability that $\hat{A}_{\eta} \neq A_{\ell(\eta)}$. Given that $|\cC^*| \geq L(n)$ then
	\[
	\begin{aligned}
	\prob(\hat{A}_{\eta} \neq A_{\ell(\eta)}) &=  \prob \left(  \left( \bigcup_{i =1}^{|\cC^*|} \supp(W_i)  \right)\neq  \supp(A_{\ell(\eta)}) \right)\\
	& \leq  \prob \left(  \left( \bigcup_{i =1}^{L(n)} \supp(W_i)  \right)\neq  \supp(A_{\ell(\eta)}) \right)\\
	& = \prob \left( \bigcup_{j \in \supp(A_{\ell(\eta)}) } \left(  j \notin \left(\bigcup_{i =1}^{L(n)} \supp(W_i) \right)\right) \right)
	\end{aligned}
	\]
	Given that $|\supp(A_l)| = d$, then applying the union bound it follows that
	\[
	\begin{aligned}
	\prob(\hat{A}_{\eta} \neq A_{\ell(\eta)})  & \leq d\prob \left( j \notin \left( \bigcup_{i =1}^{L(n)} \supp(W_i) \right) \; | \; j \in \supp(A_{\ell(\eta)}) \right)\\
	& =  d \prob \left( \bigcap_{i =1}^{L(n)} \left(j \notin \supp(W_i)\right) \; | \; j \in \supp(A_{\ell(\eta)}) \right).
	\end{aligned}
	\]
	If $j \notin \supp(W_i)$ then there exists a $h \in \supp(Z_i) \backslash \{\eta\}$ such that $j \in \supp(A_h)$. To be clear, if $j \in \supp(A_{\ell(\eta)})$ but $j \notin \supp(W_i)$, then there must exist a column $A_h$ in the submatrix $A_{\supp(Z_i)\backslash\{\ell(\eta)\}}$ which has a 1 in its $j$th entry or row, thereby resulting in $R_{j,i}$ not being a singleton value. Letting $\zeta \defeq \ceil{n/\floor{m/d}}$ then
	\[
	\begin{aligned}
	\prob \left( \bigcap_{i =1}^{L(n)} \left(j \notin \supp(W_i)\right) \; | \; j \in \supp(A_{\ell(\eta)}) \right)
% 	&= \prob \left( \bigcap_{i =1}^{L(n)} \left(j  \in \left(\bigcup_{h \in \supp(Z_i) \backslash\{ \ell(\eta)\}} \supp(A_h) \right) \right)\; | \; j \in \supp(A_{\ell(\eta)}) \right) \\
	& \leq \prob \left( \bigcap_{i =1}^{L(n)} \left(j \in \left( \bigcup_{h \in \supp(Z_i) \backslash\{ \ell(\eta)\}} \supp(A_h) \right) \right)\right)\\
	& \leq \prob \left( \bigcap_{i =1}^{L(n)} \left( j \in \left( \bigcup_{h \in \supp(Z_i) \backslash\{\ell(\eta)\}} \supp(A_h) \right)\right) \; \big| \; | \supp(\tilde{A}_j)| = \zeta\right) \\
	& \leq \prob \left( \bigcap_{i =1}^{L(n)} \left( j \in \left( \bigcup_{h \in \supp(X_i) \backslash\{ \ell(\eta)\}} \supp(A_h) \right)\right) \; \big| \; | \supp(\tilde{A}_j)| = \zeta\right) \\
	& = \prod_{i=1}^{L(n)} \prob \left(   j \in \left(  \bigcup_{h \in \supp(X_i) \backslash \{\ell(\eta)\} } \supp(A_h) \right)  \; \big| \; | \supp(\tilde{A}_j)| = \zeta \right).
	\end{aligned}
	\]
The inequality on the second line follows from  the number of nonzeros per row of $A$ being bounded by construction, therefore the events $j \in \supp(A_{\ell(\eta)})$ and $j \in \supp(A_h)$ for $h \neq \ell(\eta)$ are negatively correlated. The inequality on the third line of the above follows by considering the $j$th row to achieve the upper bound $\zeta$ of nonzeros per row. The inequality on the fourth line follows from the fact that $\supp(Z_i) \subseteq \supp(X_i)$ for all $i \in [N]$. Lastly, the equality on the fifth line follows from the mutual independence of the supports of the columns of $X$. Observe that
\[
j \in \left(\bigcup_{h \in supp(X_i) \backslash \{\ell(\eta)\}} \supp(A_h) \right) \iff (\supp(X_i) \cap \supp(\tilde{A}_j)) \backslash \{\ell(\eta)\} \neq \emptyset.
\]
Furthermore, note that by the construction of $X$ in the PSB model, Definition \ref{def_data_model}, that the probability that none of the supports of the columns of $A$ with index in $\supp(X_i) \backslash \{\ell(\eta)\}$ contain $j$ is simply a count of the number of draws of $\supp(X_i) \backslash \{\ell(\eta)\}$ that contain none of the column indices in $\supp(\tilde{A}_j) \backslash\{\ell(\eta)\}$, divided by the total number of possible draws of $\supp(X_i)\backslash \{\ell(\eta)\}$. As a result
\[
\begin{aligned}
\prob \left(j \in  \left( \bigcup_{r \in \supp(X_i) \backslash \{\ell(\eta)\}} \supp(A_r)\right) \; \big| \; | \supp(\tilde{A}_j)| = \zeta\right) = 1 - \binom{n - \zeta}{k-1} \binom{n-1}{k-1}^{-1},
\end{aligned}
\]
from which it follows that
\begin{equation} \label{eq:prob_bound_complicated}
\begin{aligned}
\prob(\hat{A}_l \neq A_l) & \leq  d \left( 1 - \binom{n - \zeta}{k-1} \binom{n-1}{k-1}^{-1} \right)^{L(n)}.
\end{aligned} 
\end{equation}
To simplify equation \ref{eq:prob_bound_complicated}, we bound the binomial coefficients as follows:
\[
	\begin{aligned}
	 \binom{n - \zeta}{k-1} \binom{n-1}{k-1}^{-1}
	& = \frac{(n - \zeta)! (n-k)!}{(n-\zeta-k+1)! (n-1)! }\\
	& = \frac{(n-\zeta)(n - \zeta -1)...(n - \zeta - k + 2)}{(n-1)(n-2)...(n-k+1)}\\
	& \geq  \left(\frac{n-\zeta-k+2}{n-k+1}\right)^{k-1}.
	\end{aligned}
\]
Therefore
\begin{equation} \label{eq:prob_bound_less_complicated}
\begin{aligned}
		\prob(\hat{A}_l \neq A_l) & \leq  d \left( 1 - \left(\frac{n-\zeta-k+2}{n-k+1}\right)^{k-1} \right)^{L(n)}.
\end{aligned} 
\end{equation}
Recalling from Definition \ref{def_data_model} that $k = \alpha_1 n$, $m = \alpha_2 n$ and $\zeta = \ceil{n/ \floor{m/d}} \leq \frac{nd}{m+d}+1$, then
\[
\begin{aligned}
 \left(\frac{n-\zeta-k+2}{n-k+1}\right)^{k-1} & \geq \left( 1 - \frac{nd}{(m+d)(n-k+1)} \right)^{\alpha_1 n-1}\\
 &> \left( 1 - \frac{nd}{m(n-k)} \right)^{\alpha_1 n}\\
 &=\left( 1 - \frac{d}{\alpha_2(1 - \alpha_1)n } \right)^{\alpha_1 n}\\
 & =  \left( 1 - \frac{b}{\alpha_1n}\right)^{\alpha_1n}
\end{aligned}
\]
where $b \defeq \frac{d \alpha_1}{\alpha_2(1-\alpha_1)}$. As a result
\[
\begin{aligned}
\prob(\hat{A}_{\eta} \neq A_{\ell(\eta)}) &\leq  d \left( 1 - \left(\frac{n-\zeta-k+2}{n-k+1}\right)^{k-1} \right)^{L(n)}\\
& \leq d \left( 1 - \left( 1 - \frac{b}{\alpha_1n}\right)^{\alpha_1n}\right)^{L(n)}\\
&=d e^{- \tau(n) L(n)}
\end{aligned}
\] 
where the final equality is derived by taking the exponential of the logarithm, and $\tau(n) \defeq - \ln \left( 1 - \left(1 - \frac{b}{\alpha_1 n} \right)^{\alpha_1 n}\right)$. Note that $\tau(x)$ for $x \in \reals_{>0}$ is continuous and $\lim_{n \rightarrow \infty} \tau(n)= - \ln(1 - e^{-b})$, therefore $\tau(n) = \cO(1)$ as claimed.
\end{proof}

% Recalling that $k = \alpha_1n +1$, $m = \alpha_2n$ and $\zeta = \ceil{nd/m} \leq nd/m +1$, then
% \[
% \begin{aligned}
%  \left(\frac{n-\zeta-k+2}{n-k+1}\right)^{k-1} &\geq \left( 1 - \frac{d}{\alpha_2(1 - \alpha_1)n } \right)^{\alpha_1 n}\\
%  & =  \left( 1 - \frac{b}{\alpha_1n}\right)^{\alpha_1n}
% \end{aligned}
% \]
% where $b \defeq \frac{d}{\alpha_2(1-\alpha_1)}$. Taking the exponential of the logarithm of the upper bound provided in equation \ref{eq:prob_bound_less_complicated}, then
% \[
% \begin{aligned}
% \prob(\hat{A}_{\eta} \neq A_{\ell(\eta)}) & \leq d \left( 1 - \left( 1 - \frac{b}{\alpha_1n}\right)^{\alpha_1n}\right)^{L(n)}\\
% &=d e^{- \tau(n) L(n)}
% \end{aligned}
% \] 
% where $\tau(n) \defeq - \ln \left( 1 - \left(1 - \frac{b}{\alpha_1 n} \right)^{\alpha_1 n}\right)$. Observe that as $\alpha_1 \leq 1 - \frac{d}{m}$ then $n > \frac{d}{\alpha_2 (1 - \alpha_1)}$. By inspection this ensures that $\tau(n)$ is a monotonically increasing function of $n$. Furthermore, as $\lim_{n \rightarrow \infty} \tau(n)= - \ln(1 - e^{-b})$ then $\tau(n) = \cO(1)$ as claimed.
% \end{proof}

\end{document}